\documentclass[smallextended,referee,envcountsect,
]{svjour3}
\smartqed
\usepackage{graphicx}

\usepackage{amsmath}
\usepackage{amssymb}
\usepackage{graphicx}
\usepackage{color}
\usepackage{ifpdf}
\usepackage{url}
\usepackage{hyperref}
\usepackage{algorithm}
\usepackage[usenames,dvipsnames]{xcolor}
\usepackage{paralist}
\usepackage[sort, numbers]{natbib}

\usepackage{algorithm}
\usepackage{algpseudocode}
\usepackage{algorithmicx}
\usepackage[ruled,vlined, linesnumbered, algo2e]{algorithm2e}
\usepackage[ruled,vlined]{algorithm2e}

\usepackage[T1]{fontenc} 
\usepackage[letterpaper, margin=1.2in]{geometry}
\usepackage{multicol} 
\usepackage[hang, small,labelfont=bf,up,textfont=it,up]{caption} 
\usepackage{booktabs} 
\usepackage{float} 
\newtheorem{assumption}{Assumption}[section]
\usepackage{todonotes}

\newcommand{\Id}{\mathbb{I}}

\newcommand{\R}{\mathbb{R}}

\newcommand{\set}[1]{\left\{#1\right\}}
\newcommand{\sets}[1]{\{#1\}}

\newcommand{\norms}[1]{\Vert#1\Vert}

\newcommand{\Eproof}{\hfill $\square$}

\newcommand{\dom}[1]{\mathrm{dom}(#1)}

\newcommand{\zero}[1]{{\boldsymbol{0}}}
\newcommand{\Exps}[2]{\mathbb{E}_{#1}\left[#2\right]}
\newcommand{\Expsn}[2]{\mathbb{E}_{#1}[#2]}
\newcommand{\Expsk}[2]{\mathbb{E}_{#1}\big[#2\big]}
\newcommand{\Expk}[1]{\mathbb{E}\big[#1\big]}

\newcommand{\zer}[1]{\mathrm{zer}(#1)}
\newcommand{\zerm}[2]{\mathrm{zer}_{#1}(#2)}

\newcommand{\gra}[1]{\mathrm{gra}(#1)}

\newcommand{\mcal}[1]{\mathcal{#1}}

\newcommand{\Ac}{\mathcal{A}}

\newcommand{\Bc}{\mathcal{B}}

\newcommand{\Xc}{\mathcal{X}}
\newcommand{\Yc}{\mathcal{Y}}
\newcommand{\Zc}{\mathcal{Z}}
\newcommand{\Sc}{\mathcal{S}}

\newcommand{\Dc}{\mathcal{D}}
\newcommand{\Lc}{\mathcal{L}}
\newcommand{\Qc}{\mathcal{Q}}

\newcommand{\Mc}{\mathcal{M}}
\newcommand{\Tc}{\mathcal{T}}

\newcommand{\Fc}{\mathcal{F}}

\newcommand{\Pc}{\mathcal{P}}

\newcommand{\Rc}{\mathcal{R}}

\newcommand{\iprods}[1]{\langle #1\rangle}
\newcommand{\Exp}[1]{\mathbb{E}\left[#1\right]}

\newcommand{\BigO}[1]{\mathcal{O}\left(#1\right)}
\newcommand{\BigOs}[1]{\mathcal{O}\big(#1\big)}

\newcommand{\mytbi}[1]{\textbf{\textit{#1}}}

\newcommand{\mbf}[1]{\mathbf{#1}}
\newcommand{\mbb}[1]{\mathbb{#1}}

\newcommand{\beforesubsec}{\vspace{-3ex}}
\newcommand{\aftersubsec}{\vspace{-1.5ex}}
\newcommand{\beforesec}{\vspace{-3ex}}
\newcommand{\aftersec}{\vspace{-2ex}}

\newcommand{\ncmt}[1]{{\color{black}#1}}
\newcommand{\revise}[1]{{\color{black}{#1}}}

\begin{document}

\titlerunning{Unbiased and Biased Variance-Reduced Forward-Reflected-Backward Splitting Methods}
\authorrunning{Q. Tran-Dinh \textit{and} N. Nguyen-Trung}

\title{Unbiased and Biased Variance-Reduced Forward-Reflected-Backward Splitting Methods for Stochastic Composite Inclusions}

\author{Quoc Tran-Dinh$^{*}$ \and Nghia Nguyen-Trung$^{*}$}

\institute{
$^{*}$Department of Statistics and Operations Research\\
The University of North Carolina at Chapel Hill\\
318 Hanes Hall, UNC-Chapel Hill, NC 27599-3260.\\
\textit{Email:} {quoctd@email.unc.edu, nghiant@unc.edu}.
}

\date{Received: date / Accepted: date}

\maketitle

\begin{abstract}
This paper develops new variance-reduction techniques for the forward-reflected-backward splitting (FRBS) method for solving possibly nonmonotone stochastic composite inclusions. 
While unbiased estimators such as mini-batching are well understood, incorporating biased estimators presents fundamental technical challenges and has not previously been studied for inclusion or fixed-point problems. 
We address this gap by developing a unified FRBS framework for both unbiased and biased estimators. Our main idea is to construct variance-reduced estimators directly for the forward-reflected direction.
First, we introduce a class of unbiased estimators that includes increasing mini-batch SGD, loopless SVRG, and SAGA. 
We establish an $\mathcal{O}(1/k)$ best-iterate convergence rate for the expected squared residual norm and almost-sure convergence of the iterates. 
The resulting oracle complexities are $\mathcal{O}(n^{2/3}\epsilon^{-2})$ for the $n$-finite-sum setting and $\mathcal{O}(\epsilon^{-10/3})$ for the expectation setting. 
Second, we introduce a class of biased estimators that includes SARAH, Hybrid SGD, and Hybrid SVRG. 
The same convergence rate remains valid, with oracle complexities of $\mathcal{O}(n^{3/4}\epsilon^{-2})$ and $\mathcal{O}(\epsilon^{-5})$ for the two settings, respectively. 
Finally, numerical experiments on AUC optimization and policy evaluation support our theoretical findings and demonstrate competitive performance.
\end{abstract}

\beforesec
\section{Introduction}\label{sec:intro}
\aftersec
\noindent\textbf{$\mathrm{(a)}$~Problem statement and assumptions.}
The forward-reflected-backward splitting (FRBS) method introduced in \cite{malitsky2020forward} provides a general framework for finding a root of the following \textbf{\textit{composite inclusion}} (also known as a \textit{generalized equation}):
\begin{equation}\label{eq:GE}
\textrm{Find $x^{\star} \in \R^p$ such that:}~ 0 \in \Phi x^{\star} := Gx^{\star} + Tx^{\star}, 
\tag{CI}
\end{equation}
where $G : \R^p\to\R^p$ is single-valued  and $T : \R^p\rightrightarrows \R^p$ is a possibly multivalued mapping.
While the original FRBS method was developed for deterministic monotone inclusions of the form \eqref{eq:GE} in \cite{malitsky2020forward}, in this paper, we extend it to handle the following two stochastic settings:
\begin{compactitem}
\item[\textsf{(E)}] We cannot directly access the mapping $G$, but it is equipped with an unbiased stochastic oracle $\mbf{G} : \R^p\times\Omega\to\R^p$ defined on a given probability space $(\Omega,\mathbb{P},\Sigma)$ such that 
\begin{equation}\label{eq:expectation_form}
Gx = \Exps{\xi}{\mbf{G}(x, \xi)} \equiv \Exps{\xi\sim\mbb{P}}{\mbf{G}(x,\xi)}.
\end{equation}
\item[\textsf{(F)}] The mapping $G$ is given by the following finite-sum of $n$ summands:
\begin{equation}\label{eq:finite_sum_form}
Gx := \frac{1}{n} \sum_{i=1}^nG_ix,
\end{equation}
where $G_i : \R^p \to\R^p$ are given single-valued mappings for $i \in [n] := \sets{1, 2, \cdots, n}$.
\end{compactitem}
Note that the finite-sum setting \textsf{(F)} can be viewed as a special case of the expectation setting \textsf{(E)}.
However, we treat them separately since the corresponding methods have different oracle complexities.
Let us denote by $\zer{\Phi} = \zer{G+T} := \sets{x^{\star} \in \R^p : 0 \in Gx^{\star} + Tx^{\star}}$ the solution set of \eqref{eq:GE}.
Throughout this paper, we make the following assumptions.
\begin{assumption}\label{as:A1}
Both $G$ and $T$ in \eqref{eq:GE} satisfy the following conditions:
\begin{compactitem}
\item[$\mathrm{(a)}$]  The solution set $\zer{\Phi} = \zer{G+T} := \sets{x^{\star} \in \R^p : 0 \in Gx^{\star} + Tx^{\star}}$ is nonempty. 

\item[$\mathrm{(b)}$]$($\textbf{Bounded variance}$)$ 
For the expectation setting \emph{(\textsf{E})}, there exists $\sigma^2 \in [0, +\infty)$ such that 
\begin{equation*}
\Exps{\xi}{ \norms{\mbf{G}(x, \xi) - Gx }^2 } \leq \sigma^2, \quad\textrm{for all}~x\in\dom{G}.
\end{equation*}

\item[$\mathrm{(c)}$]$($\textbf{Lipschitz continuity}$)$ 
For the expectation setting \emph{(\textsf{E})}, $G$ is $L$-Lipschitz continuous in expectation, i.e., there exists $L \in [0, +\infty)$ such that
\begin{equation}\label{eq:Lipschitz_cond}
\Exps{\xi}{ \norms{\mbf{G}(x, \xi) - \mbf{G}(y, \xi) }^2 } \leq L^2\norms{x - y}^2, \quad \forall x, y \in \dom{G}.
\end{equation}
Alternatively, for the finite-sum setting \emph{(\textsf{F})}, $G$ is $L$-average Lipschitz continuous, i.e., there exists $L \in [0, +\infty)$ such that
\begin{equation}\label{eq:Lipschitz_cond2}
\frac{1}{n}\sum_{i=1}^n\norms{G_ix - G_iy}^2 \leq L^2\norms{x - y}^2, \quad \forall x, y \in \dom{G}.
\end{equation}

\item[$\mathrm{(d)}$]$($\textbf{Weak Minty solution}$)$
Problem \eqref{eq:GE} admits a \revise{$\rho$-weak-Minty} solution $x^{\star} \in \zer{\Phi}$ if there exists $\rho \geq 0$ such that 
\begin{equation}\label{eq:weak_Minty_cond}
\iprods{Gx + v, x - x^{\star}} \geq -\rho\norms{Gx + v}^2, \quad \forall (x, v)\in\gra{T}.
\end{equation}
\end{compactitem}
\end{assumption}
\revise{
Denote by $\zerm{M}{\Phi}$ the set of $\rho$-weak-Minty solutions of \eqref{eq:GE}. 
Clearly, we have $\zerm{M}{\Phi}\subseteq\zer{\Phi}$.
If $\Phi$ is $\rho$-co-hypomonotone, i.e., $\iprods{u - v, x - y} \geq -\rho\norms{u - v}^2$ for all $(x, u), (y, v) \in \gra{\Phi}$, then every $x^{\star} \in \zer{\Phi}$ is a $\rho$-weak-Minty solution, and hence $\zer{\Phi} = \zerm{M}{\Phi}$.
If $\rho = 0$, then a $\rho$-weak-Minty solution $x^{\star}$ reduces to a weak-Minty solution, i.e., $\iprods{Gx + v, x - x^{\star}} \geq 0$ for all $(x, v) \in \gra{T}$.
In particular, if $\Phi$ is monotone, then every $x^{\star} \in \zer{\Phi}$ is also a weak-Minty solution.
}

\revise{Assumption~\ref{as:A1}(c) is also called \textit{mean-square Lipschitz continuity}, which has been widely used in the literature.}
By Jensen's inequality, either \eqref{eq:Lipschitz_cond} or \eqref{eq:Lipschitz_cond2} implies that  $G$ is $L$-Lipschitz continuous, i.e., $\norms{Gx - Gy} \leq L\norms{x-y}$ for all $x, y \in \dom{G}$.

Assumption~\ref{as:A1}(a) requires the existence of solutions to \eqref{eq:GE}, which is standard. 
The bounded-variance condition in Assumption~\ref{as:A1}(b) is also standard in stochastic approximation, particularly in stochastic optimization \cite{lan2020first,Nemirovski2009}. 
This assumption can be extended or generalized in various ways.
However, in this paper, we adopt its classical form. 
The Lipschitz continuity in Assumption~\ref{as:A1}(c) is not new. 
It is often required for developing variance-reduction methods via control-variate techniques, and it has been widely used in stochastic optimization and, more recently, in inclusions and variational inequalities, see, e.g., \cite{alacaoglu2021stochastic,TranDinh2024}. 
Finally, the existence of a weak-Minty solution in Assumption~\ref{as:A1}(d) extends the classical one from $\rho = 0$ to $\rho \geq 0$. 
This condition appears to have been first used in \cite{diakonikolas2021efficient} and then in subsequent works \cite{pethick2023solving,bohm2022solving,tran2022accelerated,TranDinh2024}.
Note that \eqref{eq:weak_Minty_cond} does not imply the monotonicity of $\Phi$. 
Hence, it allows \eqref{eq:GE} to cover a class of nonmonotone problems, see \cite{diakonikolas2021efficient,pethick2023solving,TranDinh2024} for concrete examples.

\vspace{0.75ex}
\noindent$\textrm{(b)}$~\textbf{\textit{Motivation and goals.}}
Our motivation is manifold.
First, problem \eqref{eq:GE} provides a versatile template that includes fixed-point formulations, variational inequality problems (VIPs), and minimax optimization as special cases. These formulations, particularly in \mytbi{stochastic settings}, are fundamental for modeling a wide range of applications, from classical problems such as two-player games, traffic equilibrium, Nash equilibrium, and robust optimization \cite{Bauschke2011,Facchinei2003,phelps2009convex,Rockafellar2004,Rockafellar1976b,ryu2016primer,dafermos1980traffic} to modern ones including generative adversarial networks (GANs), fair machine learning, adversarial training, reinforcement learning, and distributionally robust optimization \cite{arjovsky2017wasserstein,goodfellow2014generative,du2021fairness,madry2018towards,azar2017minimax,zhang2021multi,Ben-Tal2009,Bertsimas2011,levy2020large}. 
These modern applications also motivate us to focus on certain tractable classes of nonmonotone inclusions. 

Second, while classical models covered by \eqref{eq:GE} typically assume \mytbi{monotone operators} \cite{Bauschke2011,Facchinei2003}, this assumption is often violated in modern machine learning settings, where \mytbi{nonmonotone structure} is inherent. 
Although general nonmonotone inclusions are often intractable, we further examine the structure of \eqref{eq:GE} to identify tractable subclasses. 
In particular, Assumption~\ref{as:A1}(d) covers a tractable class of nonmonotone problems for which efficient solution methods can be developed, and our work also builds on this assumption.

Third, modern applications of \eqref{eq:GE} are increasingly \mytbi{large-scale} and \mytbi{high-dimensional}, rendering traditional deterministic methods computationally expensive and often impractical because they require full-batch operator evaluations. 
To mitigate this bottleneck, stochastic methods are widely used. 
However, standard stochastic estimators such as SGD and its mini-batch variants suffer from non-vanishing variance, which can lead to unstable behavior and persistent oscillations around a solution rather than exact convergence. 
To stabilize the iterates, these methods typically require \mytbi{diminishing stepsizes}, resulting in \mytbi{slow progress} in later iterations. 
This limitation motivates us to develop variance-reduction methods with constant stepsizes to improve practical performance.
Our first goal is to incorporate novel variance-reduction techniques into  the FRBS method in \cite{malitsky2020forward} that cover a wide range of estimators, including unbiased and biased instances. 
In addition, we treat both the finite-sum setting \textsf{(F)} and the expectation setting \textsf{(E)} of \eqref{eq:GE} in a unified framework.

Fourth, most existing variance-reduction techniques developed for \eqref{eq:GE} and its special cases are \mytbi{unbiased}. 
In contrast, the theoretical understanding and algorithmic development of stochastic methods with \mytbi{biased estimators} remain largely open. 
In stochastic optimization, however, biased estimators have been extensively studied and shown to be highly effective. 
Methods such as SARAH \cite{nguyen2017sarah}, SPIDER \cite{fang2018spider}, STORM \cite{Cutkosky2019}, and Hybrid SGD \cite{Tran-Dinh2019a} deliberately introduce bias to recursively control variance, achieving optimal oracle complexities (e.g., $\mathcal{O}(\epsilon^{-3})$ in the expectation setting) for nonconvex optimization. 
Despite their success, their application to \eqref{eq:GE} and its special cases such as VIPs, particularly under nonmonotone assumptions, has not been explored. 
The main theoretical challenge is that, unlike optimization, \eqref{eq:GE} lacks an objective function for constructing a suitable Lyapunov function.
Together with nonmonotonicity, this makes controlling the accumulated bias error notoriously difficult.

Finally, studying \mytbi{biased estimators} is not merely of theoretical interest.
It is highly relevant in practice. 
Many common techniques for \eqref{eq:GE} in the literature inherently produce biased estimators but lack rigorous justification.
For example, data shuffling such as random reshuffling \cite{mishchenko2020random}, distribution shifts, and delayed or asynchronous updates in parallel and distributed systems naturally break the unbiasedness of stochastic oracles. 
Since existing analyses of stochastic variance-reduced FRBS algorithms for \eqref{eq:GE} cannot accommodate such bias \cite{TranDinh2024}, a critical gap remains between practical implementations and theoretical convergence guarantees. 
Our second goal is to bridge this gap by proposing a unified stochastic FRBS framework that systematically handles both unbiased and, for the first time, biased variance-reduced estimators for nonmonotone inclusions.

\vspace{0.75ex}
\noindent$\textrm{(c)}$~\textbf{\textit{Challenges.}}
Since we study both settings  \textsf{(F)}  and  \textsf{(E)} of \eqref{eq:GE} under Assumption~\ref{as:A1}, we face the following theoretical and technical challenges that differentiate our work from existing literature.
\begin{compactitem}[$\bullet$]
\item 
First, our setting lacks strong assumptions, e.g., co-coercivity or monotonicity. 
In many standard setups, the single-valued operator $G$ is assumed to be co-coercive or monotone, 
which allows the use of the classical forward method (or gradient method), the extragradient algorithm  \cite{korpelevich1976extragradient}, or the forward-backward splitting (FBS) scheme. 
However, in this paper, we only assume the operator is ``average Lipschitz continuous'' and admits a weak-Minty solution, which covers certain nonmonotone operators.
This absence of co-coercivity and monotonicity strictly prohibits the use of standard FBS, mirror descent, and dual averaging schemes, necessitating the adoption of the forward-reflected-backward splitting (FRBS) method to accommodate a mild assumption, Assumption~\ref{as:A1}. 

\item Second, most existing variance-reduced estimators are designed to directly approximate the operator value $Gx^k$, following the construction of gradient estimators in optimization.
Departing from this approach, our FRBS scheme uses estimators constructed directly for the \textit{forward-reflected direction} defined by $S^k := 2Gx^k - Gx^{k-1}$. 
Constructing a variance-reduced, and particularly a biased, estimator for this operator $S^k$ is significantly more involved, as it requires keeping track of and handling the noise across two or three consecutive iterates simultaneously.

\item Third, controlling the accumulated error terms for biased estimators is highly non-trivial.
The Lyapunov function for analyzing convergence involves the error terms $\iprods{e^k, x^k - x^{\star}}$ and $\iprods{e^k, x^k - x^{k-1}}$, where $x^k$ is the current iterate of the algorithm, $x^{\star} \in \zer{\Phi}$, and $e^k := \widetilde{S}^k - S^k$ is the error between the estimator $\widetilde{S}^k$ and its true value $S^k$.
If $\widetilde{S}^k$ is unbiased, these terms vanish after taking the conditional expectation.  
However, this is not the case when $\widetilde{S}^k$ is biased.
In stochastic optimization, these error terms can be handled through the objective function.
Since \eqref{eq:GE} lacks an objective function, it requires us to develop auxiliary algebraic techniques to carefully process these intertwining error terms within our Lyapunov analysis, ensuring that the bias errors dynamically vanish without destroying the convergence rate of the algorithm.
\end{compactitem}
As noted earlier, bias arises not only in stochastic methods but also naturally in settings such as random reshuffling, the use of stale or delayed oracles, or numerical and computational errors. 
Developing tools to handle such bias terms enables new and efficient algorithms for solving \eqref{eq:GE} with rigorous convergence guarantees, extending beyond the purely stochastic methods considered in this paper.

\vspace{0.75ex}
\noindent$\textrm{(d)}$~\textbf{\textit{Our contributions.}}
Our contributions in this paper can be summarized as follows.
\begin{compactitem}
\item[$\mathrm{(i)}$] 
First, in Section \ref{sec:iFRBS_method}, we propose an inexact forward-reflected-backward splitting (VrFRBS) framework where a stochastic estimator $\widetilde{S}^k$ of the search direction $S^k$ is used, and provide a one-iteration analysis of the dynamics of the algorithm (Lemma~\ref{le:key_estimate1}), depending on the error $e^k := \widetilde{S}^k - S^k$.

\item[$\mathrm{(ii)}$]
Next, in Section \ref{sec:VrFRBS}, we construct a class of unbiased variance-reduced estimators for $S^k$ (Definition \ref{de:Vr_estimator}), and show that increasing mini-batch SGD, the loopless SVRG, and SAGA fall within our definition.
Using the estimators covered by Definition \ref{de:Vr_estimator}, we leverage Lemma~\ref{le:key_estimate1} to obtain a $\BigOs{1/k}$ best-iterate convergence rate for $\Expk{\norms{Gx^k + \xi^k}^2}$ with $\xi^k \in Tx^k$ as well as the almost-sure asymptotic convergence of $\norms{Gx^k + \xi^k}^2$ and the almost-sure convergence of the iterate sequence $\sets{x^k}$ to a solution of \eqref{eq:GE}.
We can also prove an oracle complexity of $\BigOs{\epsilon^{-4}\ln(\epsilon^{-1})}$ if the increasing mini-batch SGD estimator is used, or of $\BigOs{n^{2/3}\epsilon^{-2}}$ and $\BigOs{\epsilon^{-10/3}}$ for the finite-sum setting \textsf{(F)} and the expectation setting \textsf{(E)}, respectively, if either the SVRG or the SAGA estimator is used.

\item[$\mathrm{(iii)}$]
 Finally, in Section \ref{sec:VrFRBS_method2}, we introduce a new class of biased variance-reduced estimators for $S^k$ (Definition~\ref{de:Vr_estimator2}). 
Within this class, we develop three concrete estimators: SARAH, Hybrid SGD, and a new Hybrid SVRG for $S^k$.
While the convergence rate of our method remains similar to that of the unbiased case, we obtain an oracle complexity of $\BigOs{\epsilon^{-5}}$ for the expectation setting  \textsf{(E)} if the first two estimators are used, and an oracle complexity of $\BigOs{n^{3/4}\epsilon^{-2}}$ for the finite-sum setting \textsf{(F)} if SARAH and Hybrid SVRG estimators are used.
\end{compactitem}
\noindent\textbf{Comparison.} 
The FRBS scheme was developed in \cite{malitsky2020forward} for monotone inclusions, while we apply it to handle \eqref{eq:GE} under the existence of a weak-Minty solution. 
Our method in Section~\ref{sec:VrFRBS} is different from existing works such as \cite{alacaoglu2021stochastic,alacaoglu2021forward,chavdarova2019reducing,gorbunov2022stochastic,iusem2017extragradient} essentially due to the construction of stochastic estimators.
For the finite-sum case \textsf{(F)}, our results for the class of unbiased estimators are similar to those in \cite{TranDinh2024}.
However, we cover also the expectation setting \textsf{(E)} and obtain almost-sure convergence results.
To the best of our knowledge, our method in Section \ref{sec:VrFRBS_method2} is novel and has not appeared in the literature. 

\vspace{0.75ex}
\noindent$\textrm{(e)}$~\textbf{\textit{Related work.}}
Problem \eqref{eq:GE} is classical.
Theory and methods for solving it have been widely developed for many decades, especially in the monotone setting, see, e.g., \cite{Bauschke2011,Facchinei2003,Rockafellar1997}. 
In this section, let us briefly review the literature most closely related to our work.

$\textrm{(i)}$~\textbf{Splitting schemes and extragradient-type methods.}
FRBS is an operator splitting method and can also be viewed as a variant of the extragradient method \cite{korpelevich1976extragradient}. 
Splitting algorithms solve \eqref{eq:GE} by working with the operators of $G$ and $T$ separately, rather than evaluating their sum directly. 
Both the classical forward-backward splitting (FBS) and backward-forward splitting (BFS) methods require one operator in \eqref{eq:GE} to be co-coercive, a condition stronger than monotonicity that does not cover Assumption~\ref{as:A1}(d). 
To avoid co-coercivity, the extragradient method \cite{korpelevich1976extragradient} and Tseng’s forward-backward-forward splitting (FBFS) \cite{tseng2000modified} use two evaluations of $G$ and perform two sequential steps. 
These sequential updates create a fundamental obstacle to developing variance-reduction methods, due to the nested dependence of the iterates.
Several works have aimed to reduce the number of evaluations of $G$ from two to one. 
Notably, Popov’s past-extragradient scheme \cite{popov1980modification} is an early example. 
Malitsky and collaborators further developed a series of single-call methods, including the projected reflected gradient method \cite{malitsky2015projected}, the forward-reflected-backward splitting (FRBS) algorithm \cite{malitsky2020forward}, and the golden-ratio scheme \cite{malitsky2019golden}. 
These methods were designed for solving monotone inclusions of the form \eqref{eq:GE}. 
Building on this line of work, many related schemes, such as optimistic gradient methods and accelerated variants, have been studied extensively in the last decade, see, e.g., \cite{daskalakis2018limit,gorbunov2022last,mertikopoulos2019optimistic,tran2022accelerated}.

$\textrm{(ii)}$~\textbf{Beyond monotonicity.}
Classical methods for \eqref{eq:GE} have relied heavily on monotonicity assumptions of $G$ and $T$. 
Note, however, that modern machine-learning applications such as GANs and adversarial training often do not satisfy these monotone structures. 
Many works have weakened monotonicity to broader conditions for \eqref{eq:GE}, including pseudo-monotonicity, quasi-monotonicity, star-monotonicity, and the two-sided Polyak-{\L}ojasiewicz (PL) condition \cite{Konnov2001,noor2003extragradient,noor1999wiener,vuong2018weak,yang2020global}. 
More recently, another line of work has focused on weak monotonicity (or, equivalently, hypomonotonicity) and co-hypomonotonicity \cite{Rockafellar1997,bauschke2020generalized}. 
In this work, we study \eqref{eq:GE} under the existence of a weak-Minty solution. This condition was arguably first used in \cite{diakonikolas2021efficient} and subsequently adopted in works such as \cite{bohm2022solving,gorbunov2022extragradient,luo2022last,pethick2023solving}. 
It can be viewed as a weaker assumption than co-hypomonotonicity or the classical Minty solution condition.

$\textrm{(iii)}$~\textbf{Stochastic approximation methods for \eqref{eq:GE}.}
Most classical stochastic methods for \eqref{eq:GE} and its VIP special case build on the Robbins-Monro stochastic approximation (SA) framework \cite{robbins1951stochastic}. 
They typically combine deterministic techniques, such as mirror-prox with averaging, or projection and extragradient schemes, with classical SA ideas, see, e.g., \cite{cui2021analysis,iusem2017extragradient,kannan2019optimal,mishchenko2020revisiting,juditsky2011solving,pethick2023solving,yousefian2018stochastic}. 
While these approaches reduce per-iteration cost by replacing exact evaluations with unbiased stochastic oracles, they also introduce non-vanishing variance. 
As a result, they usually require diminishing stepsizes, which leads to slow sublinear convergence rates.
To address this variance bottleneck and accelerate convergence, recent work has increasingly leveraged additional structure and more refined algorithmic designs. 
Examples include methods tailored to co-coercive operators \cite{beznosikov2023stochastic} and bilinear game models \cite{li2022convergence}. 
Moreover, to stabilize stochastic iterates and enable constant (or larger) stepsizes, several works incorporate variance-reduction mechanisms into extragradient or  Halpern-type iterations, yielding both accelerated and non-accelerated convergence rates for solving \eqref{eq:GE}, see, e.g.,  \cite{cai2022stochastic,cai2023variance}.

$\textrm{(iv)}$~\textbf{Variance reduction methods for \eqref{eq:GE}.}
Another approach to address the non-vanishing variance in classical SA methods and to avoid the slow convergence caused by diminishing stepsizes is to utilize recent  variance-reduction (VR) techniques.  
Several VR ideas from optimization have been adapted to \eqref{eq:GE}, including increasing mini-batch, SVRG \cite{johnson2013accelerating}, and SAGA \cite{Defazio2014}, see, e.g., \cite{alacaoglu2021stochastic,alacaoglu2021forward,carmon2019variance,chavdarova2019reducing,huang2022accelerated,bot2019forward,palaniappan2016stochastic,TranDinh2024,yu2022fast}. 
However, most of these works focus on the monotone setting of \eqref{eq:GE}, and many obtain oracle complexities worse than those known in optimization. 
Moreover, they largely restrict attention to unbiased estimators.
In contrast, biased VR techniques such as SARAH \cite{nguyen2017sarah}, SPIDER \cite{fang2018spider}, Hybrid-SGD \cite{Tran-Dinh2019a}, and STORM \cite{Cutkosky2019} have been developed and shown effective in optimization, but their use for \eqref{eq:GE} remains largely open due to some technical challenges mentioned earlier. 
To the best of our knowledge, only \cite{cai2023variance,TranDinh2025a} exploit biased estimators for a class of \eqref{eq:GE} under co-coercivity. 
Our work broadens the scope of biased stochastic methods to a wider class of \eqref{eq:GE} under Assumption~\ref{as:A1}, including nonmonotone operators, by accommodating a broad family of biased estimators.

\vspace{0.5ex}
\noindent$\textrm{(f)}$~\textbf{\textit{Paper organization.}}
The rest of this paper is organized as follows. 
Section~\ref{sec:iFRBS_method} presents an inexact variant of the FRBS framework originally developed in \cite{malitsky2020forward} and establishes a fundamental one-step estimate for this variant. 
Section~\ref{sec:VrFRBS} proposes a new class of unbiased variance-reduced estimators for the search direction $S^k := 2Gx^k - Gx^{k-1}$ and utilizes the results in Section~\ref{sec:iFRBS_method} to establish rigorous convergence guarantees and oracle complexities for the framework using these unbiased estimators. 
In Section~\ref{sec:VrFRBS_method2}, we introduce a novel class of biased estimators for $S^k$ and provide for the first time the theoretical convergence analysis and oracle complexity bounds for the framework when operating with these biased estimators. 
Finally, Section~\ref{sec:numerical_experiments} validates the proposed methods on two machine learning applications, benchmarking them against state-of-the-art algorithms.

\vspace{0.5ex}
\noindent$\textrm{(g)}$~\textbf{\textit{Notation and basic concepts.}}
We work with a finite-dimensional Euclidean space $\R^p$ equipped with the standard inner product $\iprods{\cdot, \cdot}$ and Euclidean norm $\norms{\cdot}$.
For a multivalued mapping $T : \R^p \rightrightarrows \R^p$, $\dom{T} = \set{x \in\R^p : Tx \not= \emptyset}$ denotes its domain and $\gra{T} = \set{(x, y) \in \R^p\times \R^p : y \in Tx}$ denotes its graph.
The inverse mapping of $T$ is defined as $T^{-1}y := \sets{x \in \R^p : y \in Tx}$.
We say that $T$ is \emph{closed} if $\gra{T}$ is closed.
For given functions $g(t)$ and $h(t)$ from $\R_{+}$ to $\R_{+}$, we say that $g(t) = \BigOs{h(t)}$ if there exists $M > 0$ and $t_0 \geq 0$ such that $g(t) \leq M h(t)$ for $t \geq t_0$.  

For a mapping $T : \R^p \rightrightarrows \R^p$ and $\rho \geq 0$, we say that $T$ is $\rho$-co-hypomonotone \cite{bauschke2020generalized,combettes2004proximal} if $\iprods{u - v, x - y} \geq -\rho\norms{u  - v}^2$ for all  $(x, u), (y, v)  \in \gra{T}$.
If $\rho = 0$, then we say that $T$ is monotone.
We say that $T$ is maximally $\rho$-co-hypomonotone if $\gra{T}$ is not properly contained in the graph of any other $\rho$-co-hypomonotone operator.
If $\rho = 0$, then $T$ is called maximally monotone.
Note that a co-hypomonotone operator can also be nonmonotone.
For a single-valued mapping $F$, we say that $F$ is $L$-Lipschitz continuous with a Lipschitz constant $L \geq 0$ if $\norms{Fx - Fy} \leq L\norms{x - y}$ for all $x, y\in\dom{F}$.
Given an operator $T$, $J_Tx := \set{y \in \R^p : x \in y + Ty}$ is called the resolvent of $T$, denoted by $J_Tx = (\Id + T)^{-1}x$, where $\Id$ is the identity mapping.
\revise{Note that $J_T$ is generally multivalued; however, if $T$ is $\rho$-co-hypomonotone and $0 \leq \rho < 1$, then $J_T$ is single-valued and nonexpansive \cite{bauschke2020generalized}.}

\revise{
Throughout this paper, $\Fc_k$ denotes the $\sigma$-algebra containing all the information generated by the algorithm before the fresh random variables at iteration $k$ are sampled. 
More precisely, if $\mathcal{U}_k$ denotes all the fresh randomness sampled at iteration $k$, then $\mathcal{F}_{k+1} := \mathcal{F}_k\vee\sigma(\mathcal{U}_k)$, where $x^{k+1}$ is measurable with respect to $\mathcal{F}_{k+1}$.
In particular, $x^k$, $x^{k-1}$, $\xi^k$, and $y^k$ are $\Fc_k$-measurable. After conditioning on $\Fc_k$, the algorithm samples the random variables used to construct $\widetilde{S}^k$, and then computes $x^{k+1}$. 
Consequently, $\widetilde{S}^k$, $e^k:=\widetilde{S}^k-S^k$, $\Delta_k$, and $x^{k+1}$ are $\mathcal{F}_{k+1}$-measurable. 
In what follows, we write $\mathbb{E}_k[\cdot] := \mathbb{E}[\cdot\mid\mathcal{F}_k]$.
We also denote by $\mbb{E}[\cdot]$ the full expectation.
For a random variable $\xi$ and a sample batch $\Bc$, we denote by $\mbb{E}_{\xi}[\cdot]$ and $\mbb{E}_{\Bc}[\cdot]$ the conditional expectations w.r.t. $\xi$ and $\Bc$, respectively.
}
 
\beforesec
\section{The Inexact Forward-Reflected-Backward Splitting Framework}\label{sec:iFRBS_method}
\aftersec

The FRBS scheme  was originally developed in \cite{malitsky2020forward} to solve \eqref{eq:GE} in the deterministic monotone setting.
This method can be described as follows:
\textit{
Starting from an initial point $x^0 \in \R^p$, we set $x^{-1} = x^0$, and at each iteration $k \geq 0$, we update:
\begin{equation}\label{eq:FRBS_scheme}
x^{k+1} := J_{\eta T}(x^k - \eta(2Gx^k - Gx^{k-1})),
\end{equation}
where $\eta > 0$ is a given stepsize and $J_{\eta T}$ is the resolvent of $\eta T$.}

In \cite{malitsky2020forward}, the authors proved that if $G$ is monotone and $L$-Lipschitz continuous and $T$ is maximally monotone, then for any $\eta \in (0, \frac{1}{L}]$, the sequence $\sets{x^k}$ generated by \eqref{eq:FRBS_scheme} converges to $x^{\star} \in \zer{\Phi}$.
The convergence of \eqref{eq:FRBS_scheme} under the existence of a weak-Minty solution was established in \cite{diakonikolas2021efficient} for the VIP special case.
Several subsequent works have studied this method in different scenarios, including optimistic gradient and accelerated methods, see, e.g., \cite{daskalakis2018training,daskalakis2018limit,cai2022tight,tran2025vfog}.

The key step in \eqref{eq:FRBS_scheme} is to replace the standard forward search direction $Gx^k$ in the classical FBS method, by the following forward-reflected search direction:
\begin{equation}\label{eq:Sk_term}
S^k := 2Gx^k - Gx^{k-1} 
\end{equation}
In this case, we can rewrite the scheme \eqref{eq:FRBS_scheme} as $x^{k+1} = J_{\eta T}(x^k - \eta S^k)$.
Our stochastic methods developed in the sequel will replace $S^k$ by a stochastic variance-reduced estimator $\widetilde{S}^k$.

\vspace{1ex}
\noindent\textbf{\textit{$\mathrm{(a)}$ Our proposed method.}}
\textit{
Starting from an initial point $x^0 \in \dom{\Phi}$, we set $x^{-1} = x^0$, and at each iteration $k \geq 0$, we construct $($an unbiased or a biased$)$ stochastic variance-reduced estimator $\widetilde{S}^k$ of $S^k := 2Gx^k - Gx^{k-1}$ defined by \eqref{eq:Sk_term} and then update
\begin{equation}\label{eq:iFRBS_scheme}
x^{k+1} := J_{\eta T}(x^k - \eta\widetilde{S}^k),
\tag{VrFRBS}
\end{equation}
where $\eta > 0$ is a given constant stepsize, determined later, and $J_{\eta T}$ is the resolvent of $\eta T$.}

\revise{
Now, suppose that we run \eqref{eq:iFRBS_scheme} through iteration $k$. 
Let $\ell_k$ be a random
index uniformly distributed over $\sets{0, 1, \cdots, k}$ and independent of all the algorithmic randomness. 
We define
\begin{equation*}
    \bar{x}^k := x^{\ell_k}  \qquad\text{and}\qquad  \bar{\xi}^k := \xi^{\ell_k},
\end{equation*}
where $\xi^{\ell} \in Tx^{\ell}$ for all $\ell = 0,1,\cdots,k$. 
Consequently, we get $\bar{\xi}^k\in T\bar{x}^k$, and
\begin{equation}\label{eq:iFRBS_output}
    \mathbb{E}\big[ \norms{ G\bar{x}^k + \bar{\xi}^k}^2 \big]  = \frac{1}{k+1}\sum_{\ell = 0}^k \mathbb{E}\big[ \norms{ Gx^{\ell} + \xi^{\ell}}^2 \big].
\end{equation}
}
For a given tolerance $\epsilon > 0$, we say that $\bar{x}^k$ is an $\epsilon$-solution of \eqref{eq:GE} if $\Exp{\norms{G\bar{x}^k + \bar{\xi}^k }^2 } \leq \epsilon^2$.
Our analysis in the sequel will estimate the oracle complexity of \eqref{eq:iFRBS_scheme} required to obtain such an $\bar{x}^k$.

\vspace{0.5ex}
\noindent\textbf{\textit{$\mathrm{(b)}$~Equivalent representation and key estimate.}}
To analyze the convergence of \eqref{eq:iFRBS_scheme}, we need to rewrite it into an equivalent form.
\revise{Since $x^{-1} := x^0$, for all $k \geq 0$}, let us first define
\begin{equation}\label{eq:w_notation}
e^k := \widetilde{S}^k - S^k, \quad w^k := Gx^k + \xi^k, \quad \text{and} \quad \hat{w}^k := Gx^{k-1} + \xi^k, \quad\text{where} \quad \xi^k \in Tx^k.
\end{equation}
Then, we can rewrite \eqref{eq:iFRBS_scheme} equivalently to 
\begin{equation}\label{eq:iFRBS_scheme2}
\arraycolsep=0.2em
\begin{array}{lcl}
x^{k+1} & = & x^k - \eta(S^k + e^k + \xi^{k+1}) = x^k - \eta(2Gx^k - Gx^{k-1} + \xi^{k+1} + e^k), \vspace{1ex}\\
& = &  x^k - \eta(\hat{w}^{k+1} + w^k - \hat{w}^k + e^k).
\end{array}
\end{equation}
Next, if we introduce $y^k := x^k + \eta\hat{w}^k$, then \eqref{eq:iFRBS_scheme2} can be rewritten equivalently to
\begin{equation}\label{eq:FRBS_reform} 
\arraycolsep=0.2em
\left\{\begin{array}{lcl}
x^k & = & y^k - \eta\hat{w}^k, \vspace{1ex}\\
y^{k+1} &= & y^k - \eta(w^k + e^k).
\end{array}\right.
\end{equation}
This representation is key to the convergence analysis summarized in the following lemma.

\begin{lemma}\label{le:key_estimate1}
Let $\sets{x^k}$ be generated by \eqref{eq:iFRBS_scheme}, $e^k$, $\hat{w}^k$, and $w^k$ be defined by \eqref{eq:w_notation}, and $y^k := x^k + \eta\hat{w}^k$ as in \eqref{eq:FRBS_reform}.
Then, for any $x^{\star} \in \zer{\Phi}$ and any $\gamma > 0$, we have
\begin{equation}\label{eq:key_estimate1}
\arraycolsep=0.2em
\begin{array}{lcl}
\norms{y^{k+1} - x^{\star}}^2  &\leq & \norms{y^k - x^{\star}}^2  - 2\eta\iprods{w^k, x^k - x^{\star}} - \norms{y^k - x^k}^2 + \frac{\eta}{\gamma}\norms{Gx^k -  Gx^{k-1} }^2 \vspace{1ex}\\
&& - {~} \ (1 - \gamma\eta)\norms{y^{k+1} - x^k}^2 - 2\eta\iprods{e^k, y^k - x^{\star}} + 2\eta^2\iprods{e^k, w^k} + 2\eta^2\norms{e^k}^2.
\end{array} 
\end{equation}
In addition, we also have
\begin{equation}\label{eq:A1_th1_proof5} 
\arraycolsep=0.2em
\begin{array}{lcl}
\eta^2\norms{w^k}^2 &\leq &  (2 + 3L^2\eta^2)\big[ \norms{y^k - x^k}^2 +  \norms{y^k - x^{k-1} }^2 \big].
\end{array} 
\end{equation}
\end{lemma}

\begin{proof}
Using $y^{k+1} =  y^k - \eta(w^k + e^k)$ from \eqref{eq:FRBS_reform}, for any $x^{\star} \in \zer{\Phi}$, we can derive that
\begin{equation*}
\arraycolsep=0.1em
\begin{array}{lcl}
\norms{y^{k+1} - x^{\star}}^2 &= & \norms{y^k - x^{\star}}^2 + 2\iprods{y^{k+1} - y^k, y^{k+1} - x^{\star}} - \norms{y^{k+1} - y^k}^2 \vspace{1ex}\\
&\overset{\tiny\eqref{eq:FRBS_reform}}{=} & \norms{y^k - x^{\star}}^2 - 2\eta\iprods{w^k + e^k, y^{k+1} - x^{\star}} - \norms{y^{k+1} - y^k}^2 \vspace{1ex}\\
&= & \norms{y^k - x^{\star}}^2  - \norms{y^{k+1} - y^k}^2  - 2\eta\iprods{w^k, x^k - x^{\star}} \vspace{1ex}\\
&& - {~}   2\eta\iprods{w^k, y^{k+1} - x^k} - 2\eta\iprods{e^k, y^{k+1} - x^{\star}}.
\end{array} 
\end{equation*}
Next, utilizing the first line of \eqref{eq:FRBS_reform}, we can write $\eta w^k =  \eta \hat{w}^k + \eta(w^k - \hat{w}^k) = y^k - x^k + \eta(w^k - \hat{w}^k)$.
Hence, for any $\gamma > 0$, by  Young's inequality, we can show that
\begin{equation*}
\arraycolsep=0.2em
\begin{array}{lcl}
2\eta\iprods{w^k, y^{k+1} - x^k} &= & 2\iprods{y^k - x^k, y^{k+1} - x^k} + 2\eta\iprods{w^k - \hat{w}^k, y^{k+1} - x^k} \vspace{1ex}\\
& = & \norms{y^k - x^k}^2 + \norms{y^{k+1} - x^k}^2 - \norms{y^{k+1} - y^k}^2 + 2\eta\iprods{w^k - \hat{w}^k, y^{k+1} - x^k} \vspace{1ex}\\
&\geq & \norms{y^k - x^k}^2 + (1 - \gamma\eta)\norms{y^{k+1} - x^k}^2 - \norms{y^{k+1} - y^k}^2 - \frac{\eta}{\gamma}\norms{w^k - \hat{w}^k}^2.
\end{array} 
\end{equation*}
Using $y^{k+1} = y^k - \eta(w^k + e^k)$ from the second line of \eqref{eq:FRBS_reform}, we get
\begin{equation*}
\arraycolsep=0.2em
\begin{array}{lcl}
\iprods{e^k, y^{k+1} - x^{\star}} = \iprods{e^k, y^k - x^{\star} - \eta(w^k + e^k)} = \iprods{e^k, y^k - x^{\star}} - \eta\iprods{w^k, e^k} - \eta\norms{e^k}^2. 
\end{array} 
\end{equation*}
Combining the last three expressions, and then using $\norms{w^k - \hat{w}^k}^2 = \norms{Gx^k - Gx^{k-1}}^2$ from \eqref{eq:w_notation},  we can eventually derive that
\begin{equation*}
\arraycolsep=0.2em
\begin{array}{lcl}
\norms{y^{k+1} - x^{\star}}^2  &\leq & \norms{y^k - x^{\star}}^2  - 2\eta\iprods{w^k, x^k - x^{\star}} - \norms{y^k - x^k}^2 + \frac{\eta}{\gamma}\norms{Gx^k -  Gx^{k-1} }^2 \vspace{1ex}\\
&& - {~} \ (1 - \gamma\eta)\norms{y^{k+1} - x^k}^2 - 2\eta\iprods{e^k, y^k - x^{\star}} + 2\eta^2\iprods{w^k, e^k} + 2\eta^2\norms{e^k}^2,
\end{array} 
\end{equation*}
which proves \eqref{eq:key_estimate1}.

Finally, since $\eta w^k = y^k - x^k + \eta (w^k - \hat{w}^k)$ from \eqref{eq:FRBS_reform}, for any $a > 0$ and $b > 0$, by Young's inequality and $\norms{w^k - \hat{w}^k}^2 = \norms{Gx^k - Gx^{k-1}}^2$ above, we can show that
\begin{equation*} 
\arraycolsep=0.2em
\begin{array}{lcl}
\eta^2\norms{w^k}^2 &\leq & \eta^2(1+a)\norms{w^k - \hat{w}^k}^2 + \frac{1+a}{a}\norms{y^k - x^k}^2 \vspace{1ex}\\
& \leq & L^2\eta^2(1+a)\norms{x^k - x^{k-1} }^2 + \frac{1+a}{a}\norms{y^k - x^k}^2 \vspace{1ex}\\
& \leq & (1 + a) \big[ \frac{L^2\eta^2(1+b)}{b} + \frac{1}{a}\big] \norms{y^k - x^k}^2 + L^2\eta^2(1+a)(1+b) \norms{y^k - x^{k-1} }^2.
\end{array} 
\end{equation*}
If we choose $a = \frac{2}{3L^2\eta^2}$ and $b = 2$, then we get from the last estimate that
\begin{equation*} 
\arraycolsep=0.2em
\begin{array}{lcl}
\eta^2\norms{w^k}^2 &\leq &  (2 + 3L^2\eta^2)\big[ \norms{y^k - x^k}^2 +  \norms{y^k - x^{k-1} }^2 \big],
\end{array} 
\end{equation*}
which is exactly \eqref{eq:A1_th1_proof5}.
\Eproof
\end{proof}

The result of Lemma~\ref{le:key_estimate1} is not new in the exact case.
Indeed, if $\widetilde{S}^k = S^k$, i.e., without inexactness, then the last three terms of \eqref{eq:key_estimate1} vanish.
The resulting estimate of \eqref{eq:key_estimate1} becomes the standard estimate used to establish convergence of deterministic FRBS methods for solving \eqref{eq:GE}, see, e.g., \cite{tran2024revisiting}.
Here, we provide a very elementary and short proof of \eqref{eq:key_estimate1}.

\beforesec
\section{Unbiased Variance-Reduced FRBS Methods}\label{sec:VrFRBS}
\aftersec
This section specializes the \eqref{eq:iFRBS_scheme} framework to cover a broad class of unbiased variance-reduced estimators, including increasing mini-batch SGD, L-SVRG, and SAGA. 
Note that these estimators are constructed for $S^k$ defined in \eqref{eq:Sk_term}, rather than for $Gx^k$ itself, and thus differ from several works in the literature that relied on the estimators of $Gx^k$.

\beforesubsec
\subsection{\textbf{The class of unbiased variance-reduced estimators}}\label{subsec:unbiased_estimators}
\aftersubsec
We develop different variants of \eqref{eq:iFRBS_scheme} where $\widetilde{S}^k$ is an unbiased variance-reduced estimator of $S^k$.
This class of unbiased variance-reduced estimators satisfies the following definition.

\revise{
\begin{definition}\label{de:Vr_estimator}
Suppose that $\sets{x^k}_{k \geq 0}$ is generated by \eqref{eq:iFRBS_scheme} and that $S^k$ is defined by \eqref{eq:Sk_term}.
For each $k\geq 0$, let $\widetilde{S}^k$ be an $\mathcal{F}_{k+1}$-measurable stochastic estimator of $S^k$, constructed using the fresh randomness sampled at iteration $k$. 
Let  $e^k:=\widetilde{S}^k - S^k$ be the \textit{approximation error}.
Suppose that there exist an $\mathcal{F}_{k+1}$-measurable nonnegative random variable $\Delta_k$, three constants $\kappa\in(0,1]$,
 $\Theta \geq 0$, and $\widehat{\Theta} \geq 0$, and a deterministic nonnegative sequence $\sets{\delta_k}_{k\geq 0}$ such that, almost surely:
\begin{equation}\label{eq:Vr_estimator}
\arraycolsep=0.2em
\left\{\begin{array}{lcl}
    \mathbb{E}_k[ e^k ] & = & 0,\vspace{1ex}\\
    \mathbb{E}_k[ \norms{ e^k }^2 ] & \leq & \mathbb{E}_k[\Delta_k], \vspace{1ex}\\
    \mathbb{E}_k[ \Delta_k ] & \leq & (1 - \kappa)\Delta_{k-1} + \Theta \norms{x^k - x^{k-1} }^2 + \widehat{\Theta} \norms{ x^{k-1} - x^{k-2} }^2 +\delta_k.
\end{array}\right.
\end{equation}
Here, we choose $\Delta_{-1} := 0$ and $x^{-2}=x^{-1} := x^0$.
\end{definition}
}
From Definition~\ref{de:Vr_estimator}, we have $\mbb{E}_0[ \norms{e^0}^2 ] = \mbb{E}_0[ \norms{\widetilde{S}^0 - S^0}^2 ] \leq \delta_0$.
If $\delta_0 = 0$, then we have $\norms{\widetilde{S}^0 - S^0} = 0$, i.e., $\widetilde{S}^0 = S^0$.
Note that $\widetilde{S}^k$ satisfying Definition~\ref{de:Vr_estimator} forms a class of unbiased variance-reduced estimators for $S^k$.
Several existing estimators fit this definition.
We consider the following three examples of $\widetilde{S}^k$, but other constructions are also possible.
Definition~\ref{de:Vr_estimator} is not completely new, but it extends \cite[Definition~1]{TranDinh2024} to cover the expectation case \textsf{(E)} and an additional term $\delta_k$.

\vspace{1ex}
\noindent\textbf{$\mathrm{(a)}$~The increasing mini-batch SGD estimator.}
For the expectation setting (\textsf{E}), we construct the following increasing mini-batch SGD estimator for $S^k$:
\begin{equation}\label{eq:sgd_estimator}
\widetilde{S}^k := 2G_{\mcal{B}_k}x^k - G_{\mcal{B}_k}x^{k-1} = \frac{1}{b_k}\sum_{\xi\in\mcal{B}_k}\big[ 2\mbf{G}(x^k, \xi) - \mbf{G}(x^{k-1}, \xi) \big],
\end{equation}
where $\mcal{B}_k$ is a given i.i.d. mini-batch of size $b_k$, and $G_{\mcal{B}_k}x = \frac{1}{b_k}\sum_{\xi \in\mcal{B}_k}\mbf{G}(x,\xi)$.
Here, the mini-batch size $b_k$ is chosen to satisfy the following ``adaptive'' condition:
\begin{equation}\label{eq:b_choice}
b_k \geq \revise{9\sigma^2} \Big[ \tfrac{L^2}{2}\norms{x^k - x^{k-1}}^2 + \tfrac{L^2}{2}\norms{x^{k-1} - x^{k-2}}^2 + \delta_k \Big]^{-1},
\end{equation}
where $x^{-2} = x^{-1} = x^0$ and $\sets{\delta_k}_{k\geq 0}$ satisfies $S_{\infty} := \sum_{k=0}^{\infty}\delta_k < +\infty$.
We have the following result.

\begin{lemma}\label{le:sgd_estimator}
Let $\widetilde{S}^k$ be generated by \eqref{eq:sgd_estimator}, where the mini-batch size $b_k$ of $\mcal{B}_k$ satisfies \eqref{eq:b_choice}.
Then, $\widetilde{S}^k$ satisfies Definition~\ref{de:Vr_estimator} with $\kappa = 1$, $\Theta = \hat{\Theta} = \frac{L^2}{2}$, and $\set{\delta_k}$ in \eqref{eq:b_choice}. 
\end{lemma}

\begin{proof}
First, since $\widetilde{S}^k$ is constructed by \eqref{eq:sgd_estimator}, it follows that $\mbb{E}_{\mcal{B}_k}[ \widetilde{S}^k ] = 2Gx^k - Gx^{k-1} = S^k$, verifying the first condition of \eqref{eq:Vr_estimator} after taking the conditional expectation $\Exps{k}{\cdot}$ on both sides of this relation.
Next, let us define $\Delta_k := \frac{9\sigma^2}{b_k}$.
\revise{Then, by Assumption~\ref{as:A1}(b) and the Cauchy-Schwarz inequality, it follows from \eqref{eq:sgd_estimator} that $\mathbb{E}_{\mcal{B}_k}[ \norms{\widetilde{S}^k - S^k}^2 ] \leq \frac{9\sigma^2}{b_k} = \Delta_k$.}
Taking the conditional expectation $\mathbb{E}_k[\cdot]$ on both sides of this relation, we verify the second condition of \eqref{eq:Vr_estimator}.

Finally, from \eqref{eq:b_choice} and the definition of $\Delta_k$, we have
\begin{equation*}
\Delta_k = \frac{\revise{9\sigma^2}}{b_k} \leq \frac{L^2}{2}\norms{x^k - x^{k-1}}^2 + \frac{L^2}{2}\norms{x^{k-1} - x^{k-2}}^2 + \delta_k.
\end{equation*}
Taking the conditional expectation $\mathbb{E}_k[\cdot]$ on both sides of this relation, we verify the last condition of \eqref{eq:Vr_estimator} with $\Theta = \hat{\Theta} = \frac{L^2}{2}$.
\Eproof
\end{proof}

\noindent\textbf{$\mathrm{(b)}$~The L-SVRG estimator.}
\revise{
The SVRG estimator was introduced in \cite{SVRG}, and its loopless version (L-SVRG) was then proposed in \cite{kovalev2019don}.
In this paper, we adapt L-SVRG to construct an estimator of $S^k := 2Gx^k - Gx^{k-1}$ defined by \eqref{eq:Sk_term}.

We first initialize $\tilde{x}^0 := x^0$.
Then, at iteration $k \geq 1$, let $\zeta_k$ be a Bernoulli random variable with $\mathbb{P}(\zeta_k=1\mid\Fc_k) = \mathbf{p}$ and $\mathbb{P}(\zeta_k=0\mid\Fc_k) = 1 - \mathbf{p}$, where $\mathbf{p} \in (0, 1)$ is a given probability.
First, we update the snapshot point according to
\begin{equation}\label{eq:svrg_wk}
    \tilde{x}^k := \begin{cases}
        x^{k-1}, & \text{if } \zeta_k=1,\\
        \tilde{x}^{k-1}, & \text{if }\zeta_k=0.
    \end{cases}
\end{equation}
Next,  we construct
\begin{equation}\label{eq:svrg_estimator}
	\widetilde{S}^k := \overline{G}\tilde{x}^k - G_{\mcal{B}_k}\tilde{x}^k + 2G_{\mcal{B}_k}x^k - G_{\mcal{B}_k}x^{k-1},
	\tag{L-SVRG}
\end{equation}
where
\begin{compactitem}
\item[(1)] the mini-batch estimator $G_{\mcal{B}_k}y := \frac{1}{b}\sum_{\xi\in\mcal{B}_k}\mbf{G}(y, \xi)$ for a given mini-batch $\mcal{B}_k$.
\item[(2)] $\overline{G}\tilde{x}^k$ satisfies $\Expsn{k}{\overline{G}\tilde{x}^k} = G\tilde{x}^k$ and $\Expsn{k}{\norms{\overline{G}\tilde{x}^k - G\tilde{x}^k}^2} \leq \sigma_k^2$ for a given $\sigma_k^2 \geq 0$.
\end{compactitem}
We have at least two options to construct the unbiased estimator $\overline{G}\tilde{x}^k$.
\begin{compactitem}
\item[(i)] (\textit{Exact evaluation}). 
We exactly evaluate $\overline{G}\tilde{x}^k = G\tilde{x}^k$.
For instance, if $G$ is given by the finite-sum form \eqref{eq:finite_sum_form}, then we can compute $\overline{G}\tilde{x}^k = G\tilde{x}^k = \frac{1}{n}\sum_{i=1}^nG_i\tilde{x}^k$ using a full-batch evaluation of $G$.

\item[(ii)] (\textit{Mega-batch evaluation}). 
If $\mbf{G}$ is given by the expectation (\textsf{E}), then we set $\overline{G}\tilde{x}^k := \frac{1}{n_k}\sum_{\xi \in \mcal{M}_k}\mbf{G}(\tilde{x}^k, \xi)$ for a given mega-batch $\mcal{M}_k$ of fixed size $n_k := n$ or increasing size $n_k  \geq n_{k-1}$. 
\end{compactitem}
Conditioned on $\Fc_k$, the Bernoulli random variable $\zeta_k$, the mini-batch $\Bc_k$, and, when applicable, the mega-batch $\mcal{M}_k$ defined below are mutually independent. 
Moreover, $\Bc_k$ consists of $b$ i.i.d.\ samples drawn from $\mbb{P}$, and $\mcal{M}_k$ consists of $n_k$ i.i.d.\ samples drawn from $\mbb{P}$. 
These random objects are also independent across iterations, conditioned on the past.
}

Finally, the following lemma shows that $\widetilde{S}^k$ satisfies Definition~\ref{de:Vr_estimator}.
Its proof is in Appendix~\ref{apdx:le:svrg_estimator}.

\begin{lemma}\label{le:svrg_estimator}
Let $\widetilde{S}^k$ be an L-SVRG estimator of $S^k$ constructed by \eqref{eq:svrg_estimator}.
Then:
\begin{compactitem}
\item[$\mathrm{(a)}$] 
If $\overline{G}\tilde{x}^k = G\tilde{x}^k$, then $\widetilde{S}^k$ satisfies Definition~\ref{de:Vr_estimator} with $\kappa = \frac{\mbf{p}}{2}$, $\Theta = \frac{16L^2}{b\mbf{p}}$, $\hat{\Theta} = \frac{4L^2}{b\mbf{p}}$, and $\delta_k = 0$.
\item[$\mathrm{(b)}$] 
If $\overline{G}\tilde{x}^k = \frac{1}{n_k}\sum_{\xi\in\mcal{M}_k} \mbf{G}(\tilde{x}^k, \xi)$ is a mega-batch estimator, then $\widetilde{S}^k$ satisfies Definition~\ref{de:Vr_estimator} with $\kappa = \frac{\mbf{p}}{2} \in (0, 1)$, $\Theta = \frac{16(1+\mu)L^2}{b\mbf{p}}$, $\hat{\Theta} = \frac{4(1+\mu)L^2}{b\mbf{p}}$, and $\delta_k = \frac{(1+\mu)\mbf{p}\sigma^2}{2\mu n_k}$ for any $\mu > 0$.
\end{compactitem}
\end{lemma}

\vspace{1ex}
\noindent\textbf{$\mathrm{(c)}$~The SAGA estimator for the finite-sum case (\textsf{F}).}
The SAGA estimator was proposed in \cite{Defazio2014} for convex optimization.
It has been widely used in both convex and nonconvex optimization as well as extended to other types of problems, including \eqref{eq:GE}, see, e.g., \cite{davis2022variance}.
Here, we adopt this SAGA idea to construct an estimator for $S^k$, which is similar to the one in \cite{TranDinh2024}.

Let $G$ be given by \eqref{eq:finite_sum_form} for the finite-sum case (\textsf{F}), and  $S^k$ be defined by \eqref{eq:Sk_term}.
Then, we construct a SAGA estimator $\widetilde{S}^k$ of $S^k$ as follows:
\begin{equation}\label{eq:saga_estimator}\tag{SAGA}
\widetilde{S}^k := \widehat{G}^k - \widehat{G}_{\mcal{B}_k} + 2G_{\mcal{B}_k}x^k - G_{\mcal{B}_k}x^{k-1},
\end{equation} 
where $\mcal{B}_k \subseteq [n]$ is an i.i.d. mini-batch of size $b$ (with replacement), $\widehat{G}^k := \frac{1}{n}\sum_{i=1}^n\widehat{G}^k_i$, $\widehat{G}_{\mcal{B}_k}^k := \frac{1}{b}\sum_{i \in \mcal{B}_k}\widehat{G}^k_i$,  and $G_{\mcal{B}_k}x := \frac{1}{b}\sum_{i \in \mcal{B}_k}G_ix$.
Here, we compute $\widehat{G}^0_i = G_ix^0$ for all $i \in [n]$, and at each iteration $k \geq 1$, we update $\widehat{G}_i^k$ for all $i \in [n]$ using a mini-batch $\hat{\mcal{B}}_k\subseteq [n]$ of size $b$ (without replacement and independent of $\mcal{B}_k$) as
\begin{equation}\label{eq:saga_update}
\widehat{G}_i^k := \begin{cases}
G_ix^{k-1} & \textrm{if $i \in \hat{\mcal{B}}_k$}, \\
\widehat{G}^{k-1}_i &\textrm{otherwise}.
\end{cases}
\end{equation}
Unlike the L-SVRG estimator, the SAGA estimator requires storing $n$ values $\widehat{G}^k_i$ in a table $\mbb{T}^k := [\widehat{G}^k_1, \cdots, \widehat{G}_n^k]$ of size $n\times p$.
Clearly, when both $n$ and $p$ are large, maintaining $\mbb{T}^k$ requires significant memory.
The following lemma shows that $\widetilde{S}^k$ constructed by \eqref{eq:saga_estimator} satisfies Definition~\ref{de:Vr_estimator}, whose proof is similar to that of \cite[Lemma~2]{TranDinh2024} and is therefore omitted.

\begin{lemma}\label{le:saga_estimator}
The SAGA estimator $\widetilde{S}^k$ constructed by \eqref{eq:saga_estimator} satisfies Definition~\ref{de:Vr_estimator} with $\kappa = \frac{b}{2n} \in (0, 1)$, $\Theta = \frac{16L^2n}{b^2}$, $\hat{\Theta} = \frac{4L^2n}{b^2}$, and $\delta_k = 0$.
\end{lemma}


\beforesubsec
\subsection{\textbf{Convergence analysis of \ref{eq:iFRBS_scheme} with unbiased variance-reduced estimators}}\label{subsec:convergence_analysis1}
\aftersubsec
To prove our convergence results in this section, we need  the following technical lemma that defines a Lyapunov function and establishes its descent property.
Its proof is given in Appendix~\ref{apdx:le:key_estimate2}.

\begin{lemma}\label{le:key_estimate2}
Suppose that $\sets{(x^k, y^k)}$ is generated by \eqref{eq:iFRBS_scheme}, $x^{\star} \in \zerm{M}{\Phi}$, and $\widetilde{S}^k$ is a stochastic estimator of $S^k$ from \eqref{eq:Sk_term} that satisfies Definition~\ref{de:Vr_estimator}.
For any $\gamma > 0$, we define the following function:
\begin{equation}\label{eq:Pk_func}
\Pc_k := \norms{y^k - x^{\star}}^2 + (1-\gamma\eta)\norms{y^k - x^{k-1}}^2 +   \frac{2\eta^2(1-\kappa)}{\kappa} \Delta_{k-1} + \frac{2\eta^2 \hat{\Theta}}{\kappa}\norms{x^{k-1} - x^{k-2}}^2.
\end{equation}
Then, for any $c > 0$ and for all $k \geq 0$, we have
\begin{equation}\label{eq:key_estimate2}
\arraycolsep=0.2em
\begin{array}{lcl}
\Exps{k}{\Pc_{k+1} } & \leq & \Pc_k  - \Big[1 - \frac{4\rho}{\eta} - 2L^2\eta\big(4\rho + \frac{1}{\gamma} + c\big) \Big] \norms{y^k - x^k}^2  \vspace{1ex}\\
&& - {~} \Big[1 - \gamma\eta -  2L^2\eta\big(4\rho + \frac{1}{\gamma} + c\big) \Big] \norms{y^k - x^{k-1}}^2 \vspace{1ex}\\
&& - {~} \frac{\eta [ c\kappa L^2 - 2\eta(\Theta + \hat{\Theta})] }{\kappa} \norms{x^k - x^{k-1}}^2 + \frac{2\eta^2 \delta_k}{\kappa}.
\end{array} 
\end{equation}
\end{lemma}
Our first main result of this paper is the following theorem.

\begin{theorem}\label{th:iFRBS_convergence1}
Suppose that Assumption~\ref{as:A1} holds for \eqref{eq:GE} such that $L\rho < \frac{1}{24\sqrt{2}}$.
Let $\sets{(x^k, \bar{x}^k)}$ be generated by \eqref{eq:iFRBS_scheme} using an estimator $\widetilde{S}^k$ for $S^k$ satisfying Definition~\ref{de:Vr_estimator}.
Suppose further that the stepsize $\eta$ and the parameter triple $(\kappa, \Theta, \hat{\Theta})$ in Definition~\ref{de:Vr_estimator} satisfy
\begin{equation}\label{eq:iFRBS_cond1}
8\rho \leq \eta < \tfrac{1}{3\sqrt{2}L}  \quad \text{and} \quad \kappa L^2 \geq \Theta + \hat{\Theta}.
\end{equation}
Then, for any $K \geq 0$, \revise{any $x^{\star}\in \zerm{M}{\Phi}$}, and $\bar{x}^K$ constructed by \eqref{eq:iFRBS_output}, we have
\begin{equation}\label{eq:key_convergence_bound1}
\arraycolsep=0.2em
\begin{array}{lcl}
\Exp{\norms{G\bar{x}^K + \bar{\xi}^K}^2} &= & \frac{1}{K+1}\sum_{k=0}^K\Exp{\norms{Gx^k + \xi^k}^2} \vspace{1ex}\\
&\leq & \frac{\ncmt{13}}{\ncmt{3}C (K+1) } \Big[ \frac{2}{\eta^2} \norms{x^0 - x^{\star}}^2 + \frac{5}{2}\norms{Gx^0+\xi^0}^2 \Big]  + \frac{26}{3\kappa C(K+1)} S_K,
\end{array} 
\end{equation}
where $S_K := \sum_{k=0}^K\delta_k$, $C := 1 - 18L^2\eta^2 > 0$, $\xi^k \in Tx^k$ for all $0 \leq k\leq K$, and $\bar{\xi}^K \in T\bar{x}^K$.

If additionally $S_{\infty} := \sum_{k=0}^{\infty}\delta_k < +\infty$, then almost surely, we have
\begin{equation}\label{eq:key_convergence_bound1b}
\sum_{k=0}^{\infty} \norms{Gx^k + \xi^k}^2 < +\infty \quad \textrm{and} \quad \lim_{k\to\infty} \norms{Gx^k + \xi^k} = 0 \quad \textrm{for some}~\xi^k \in Tx^k.
\end{equation}
Moreover, if $\gra{T}$ is closed \revise{and $\zer{\Phi} = \zerm{M}{\Phi}$ $($in particular, if $\Phi$ is $\rho$-co-hypomonotone$)$}, then $\sets{x^k}$ converges almost surely to a $\zer{\Phi}$-valued random variable $x^{\star}$, a solution of \eqref{eq:GE}.
\end{theorem}

\begin{proof}
First, let us choose $c := 2\eta$ and $\gamma := \frac{1}{2\eta}$ in  \eqref{eq:key_estimate2} of Lemma~\ref{le:key_estimate2}.
Then, \eqref{eq:key_estimate2} reduces to
\begin{equation}\label{eq:key_estimate2b}
\arraycolsep=0.2em
\begin{array}{lcl}
\Exps{k}{\Pc_{k+1} } & \leq & \Pc_k  - \Big[1 - \frac{4\rho}{\eta} - 8L^2\eta\big(\rho + \eta \big) \Big] \norms{y^k - x^k}^2 - \frac{2\eta^2 [ \kappa L^2 - (\Theta + \hat{\Theta})] }{\kappa} \norms{x^k - x^{k-1}}^2 \vspace{1ex}\\
&& - {~} \frac{1}{2}\big[ 1 - 16L^2\eta(\rho + \eta ) \big] \norms{y^k - x^{k-1}}^2  + \frac{2\eta^2 \delta_k}{\kappa}.
\end{array} 
\end{equation}
Next, let us choose $\eta > 0$ such that $\frac{4\rho}{\eta} \leq \frac{1}{2}$.
Then, we have $\eta \geq 8\rho$.
In this case, we have $1 - \frac{4\rho}{\eta} - 8L^2\eta\big(\rho + \eta \big) \geq \frac{1}{2}\big(1 - 18L^2\eta^2\big) =: \frac{C}{2} > 0$ if we choose $8\rho \leq \eta < \frac{1}{3L\sqrt{2}}$.
Therefore, if $L\rho < \frac{1}{24\sqrt{2}} \approx 0.029463$, then by choosing $8\rho \leq \eta < \frac{1}{3\sqrt{2}L}$,  we can further simplify \eqref{eq:key_estimate2b} as follows:
\begin{equation}\label{eq:key_estimate2c}
\hspace{-1ex}
\arraycolsep=0.2em
\begin{array}{lcl}
\Exps{k}{\Pc_{k+1} } & \leq & \Pc_k  - \frac{C}{2} \big[ \norms{y^k - x^k}^2 +  \norms{y^k - x^{k-1}}^2 \big] - \frac{2\eta^2 [ \kappa L^2 - (\Theta + \hat{\Theta})] }{\kappa} \norms{x^k - x^{k-1}}^2  + \frac{2\eta^2 \delta_k}{\kappa}.
\end{array} 
\hspace{-1ex}
\end{equation}
Substituting the inequality \eqref{eq:A1_th1_proof5} into \eqref{eq:key_estimate2c} and assuming that $\Theta + \hat{\Theta} \leq \kappa L^2$ as in \eqref{eq:iFRBS_cond1}, we get
\begin{equation}\label{eq:key_estimate2d}
\arraycolsep=0.2em
\begin{array}{lcl}
\Exps{k}{\Pc_{k+1} } & \leq & \Pc_k  - \frac{C\eta^2}{2(2 + 3L^2\eta^2)} \norms{w^k}^2  + \frac{2\eta^2 \delta_k}{\kappa}.
\end{array} 
\end{equation}
Taking the full expectation on both sides of \eqref{eq:key_estimate2d}, we arrive at
\begin{equation*}
\arraycolsep=0.2em
\begin{array}{lcl}
\Exp{\Pc_{k+1} } & \leq & \Exp{\Pc_k}  - \frac{C\eta^2}{2(2 + 3L^2\eta^2)} \Exp{\norms{w^k}^2}  + \frac{2\eta^2 \delta_k}{\kappa}.
\end{array} 
\end{equation*}
Telescoping this inequality from $k=0$ to $K$ and noting that $\Pc_k \geq 0$, we obtain
\begin{equation}\label{eq:key_estimate2f}
\arraycolsep=0.2em
\begin{array}{lcl}
\frac{1}{K+1}\sum_{k=0}^K\Exp{\norms{w^k}^2} &\leq & \frac{2(2+3L^2\eta^2)}{C\eta^2 (K+1) }\Exp{\Pc_0} + \frac{4(2 + 3L^2\eta^2)}{\kappa C(K+1)}\sum_{k=0}^K\delta_k
\end{array} 
\end{equation}
Since $\gamma = \frac{1}{2\eta}$, $\Delta_{-1} = 0$, $x^{-2} = x^{-1} = x^0$ and $y^k = x^k + \eta\hat{w}^k$, we have $y^0 = x^0 + \eta w^0$.
Thus, we can show from \eqref{eq:Pk_func} that
\begin{equation*}
\arraycolsep=0.2em
\begin{array}{lcl}
\Pc_0 & = & \norms{x^0 + \eta w^0 - x^{\star}}^2 + (1-\gamma\eta)\eta^2\norms{w^0}^2 \leq 2\norms{x^0 - x^{\star}}^2 + \frac{5\eta^2}{2}\norms{w^0}^2.
\end{array}
\end{equation*}
Using this bound and $L^2\eta^2 \leq \frac{1}{18}$ into \eqref{eq:key_estimate2f}, we finally get 
\begin{equation*}
\arraycolsep=0.2em
\begin{array}{lcl}
\frac{1}{K+1}\sum_{k=0}^K\Exp{\norms{w^k}^2} &\leq & \frac{\ncmt{13}}{\ncmt{3}C (K+1) } \Big[ \frac{2}{\eta^2} \norms{x^0 - x^{\star}}^2 + \frac{5}{2}\norms{w^0}^2 \Big]  + \frac{26}{3\kappa C(K+1)} S_K,
\end{array} 
\end{equation*}
where $S_K := \sum_{k=0}^K\delta_k$.
This proves \eqref{eq:key_convergence_bound1} by noting that $w^k = Gx^k + \xi^k$ for $\xi^k \in Tx^k$.

Next, since $S_{\infty} := \sum_{k=0}^{\infty}\delta_k < +\infty$,  applying Lemma~\ref{le:RS_lemma} from Appendix~\ref{sec:useful_lemmas} to \eqref{eq:key_estimate2d}, we can show that $\sum_{k=0}^{\infty}\norms{w^k}^2 <+\infty$ almost surely.
Moreover, $\sets{\Pc_k}$ also converges to $\Pc^{*}$ almost surely.

If we apply again Lemma~\ref{le:RS_lemma} to \eqref{eq:key_estimate2c}, then we also obtain $\sum_{k=0}^{\infty}\norms{y^k - x^k}^2 <+\infty$ and $\sum_{k=0}^{\infty}\norms{y^k - x^{k-1}}^2 < +\infty$ almost surely.
In this case, we also have $\sum_{k=0}^{\infty}\norms{x^k - x^{k-1}}^2 < +\infty$ \ncmt{almost surely}.

From \eqref{eq:Vr_estimator},  we have
\begin{equation*}
\arraycolsep=0.2em
\begin{array}{lcl}
\Exps{k}{\Delta_k} & \leq & \Delta_{k-1} - \kappa\Delta_{k-1} + \Theta\norms{x^k - x^{k-1}}^2 + \hat{\Theta}\norms{x^{k-1} - x^{k-2}}^2 + \delta_k.
\end{array}
\end{equation*}
Since $\sum_{k=0}^{\infty}\delta_k < +\infty$ and $\sum_{k=0}^{\infty}\norms{x^k - x^{k-1}}^2 < +\infty$ almost surely, applying Lemma~\ref{le:RS_lemma} to the last estimate, we conclude that $\sum_{k=0}^{\infty}\Delta_{k-1} < +\infty$ almost surely.
Overall, we can conclude that the following limits hold almost surely:
\begin{equation}\label{eq:A1_th1_convergence1}
\lim_{k\to\infty}\norms{x^k - x^{k-1}}^2 = 0, \quad \lim_{k\to\infty}\norms{y^k - x^{k-1}}^2 = 0, \quad \textrm{and} \quad \lim_{k\to\infty}\Delta_k = 0.
\end{equation}
From \eqref{eq:Pk_func} and \ncmt{$\gamma = \frac{1}{2\eta}$}, we have
\begin{equation*} 
\Pc_k := \norms{y^k - x^{\star}}^2 + \frac{1}{2}\norms{y^k - x^{k-1}}^2 +   \frac{2\eta^2(1-\kappa)}{\kappa} \Delta_{k-1} + \frac{2\eta^2 \hat{\Theta}}{\kappa}\norms{x^{k-1} - x^{k-2}}^2.
\end{equation*}
Since $\lim_{k\to\infty}\Pc_k$ exists \ncmt{almost surely}, using this fact and \eqref{eq:A1_th1_convergence1} in the last relation, we conclude that $\lim_{k\to\infty}\norms{y^k - x^{\star}}$ exists almost surely.
Since $\lim_{k\to\infty}\norms{x^k - y^k} = 0$ almost surely, $\lim_{k\to\infty}\norms{x^k - x^{\star}}$ also exists almost surely for any \revise{$x^{\star} \in \zerm{M}{\Phi}$.} 
Note that we can also use standard separability argument as in \cite{tran2025vfog} if needed.

\revise{
Finally, since we assume that $\zerm{M}{\Phi} = \zer{\Phi}$, we also get $x^{\star} \in \zer{\Phi}$.
On the one hand, since  $\sum_{k=0}^{\infty}\norms{w^k}^2 <+\infty$ almost surely,  we have $\lim_{k\to\infty}\norms{w^k} = 0$ almost surely.
On the other hand, since $G$ is $L$-Lipschitz continuous and $\gra{T}$ is closed, $\gra{\Phi} = \gra{G+T}$ is also closed. 
Next, because $w^k = Gx^k + \xi^k \in Gx^k + Tx^k$ such that $w^k \to \mbf{0}$ as $k \to \infty$ and $\gra{\Phi}$ is closed, we conclude that any limit point  of $\sets{x^k}$ belongs to $\zer{\Phi}$.
Applying Lemma~\ref{le:Opial_lemma} from Appendix~\ref{sec:useful_lemmas}, we can conclude that $\sets{x^k}$ converges to a $\zer{\Phi}$-valued random variable $x^{\star}$ as a solution of \eqref{eq:GE}.
}
\Eproof
\end{proof}

Theorem \ref{th:iFRBS_convergence1} deserves some discussion.
First, note that we have not attempted to optimize the admissible range of $L\rho$ or the choice of $\eta$ in Theorem~\ref{th:iFRBS_convergence1}. 
We believe these ranges can be improved with a sharper analysis. 
Second, the almost-sure convergence of the iterates, as in Theorem~\ref{th:iFRBS_convergence1}, has been established for other methods, but primarily under co-coercive and monotone assumptions, see, e.g., \cite{davis2022variance}. 
\revise{
Third, the assumption $\zer{\Phi} = \zerm{M}{\Phi}$ in the last statement can be replaced by assuming that every cluster point $x^{\ast}$ of $\{x^k\}$ belongs to $\zerm{M}{\Phi}$.
}
Finally, the closedness of $\gra{T}$ is not restrictive. 
For instance, it is well-known that if $T$ is continuous, outer semicontinuous, or the subdifferential of a proper, closed, and convex function, then $\gra{T}$ is closed.

Let us compare Theorem~\ref{th:iFRBS_convergence1} with recent related work. 
First, Theorem~\ref{th:iFRBS_convergence1} differs fundamentally from many existing results, including \cite{alacaoglu2021stochastic,chavdarova2019reducing,gorbunov2022stochastic,iusem2017extragradient,kannan2019optimal}, because of our construction of the unbiased stochastic estimator $\widetilde{S}^k$. 
Second, our method also builds on the FRBS scheme as in \cite{alacaoglu2021forward}, but our construction of $\widetilde{S}^k$ is different. 
In addition, we cover a broad class of estimators rather than only SVRG as in \cite{alacaoglu2021forward}, and our Assumption~\ref{as:A1} is weaker, since it also allows for nonmonotone problems. 
Third, Theorem~\ref{th:iFRBS_convergence1} strengthens our recent results in \cite{TranDinh2024}, but for a different algorithm.
Note that the algorithm in \cite{TranDinh2024} does not reduce to the FRBS method we study here due to the choice of $\gamma \in (1/2, 1)$.
In addition, our method applies to both the finite-sum and expectation settings, permits more flexibility in constructing $\widetilde{S}^k$ in Definition~\ref{de:Vr_estimator} through the additional term $\delta_k$, and establishes almost-sure convergence of the iterates.

\revise{
\begin{remark}\label{re:range_of_Lrho}
The stepsize $\eta$ in Theorem~\ref{th:iFRBS_convergence1} satisfies $\eta \in \big[8\rho, \frac{1}{3\sqrt{2}L} \big)$.
Since $L\rho < \frac{1}{24\sqrt{2}}$, we have $8\rho < \frac{1}{3\sqrt{2}L}$, and hence such an $\eta$ exists.
If $\rho = 0$, i.e., \eqref{eq:GE} has a weak-Minty solution $x^{\star} \in \zer{\Phi}$ (in particular, if $\Phi$ is monotone), then the condition $L\rho < \frac{1}{24\sqrt{2}}$ is automatically satisfied.
In this case, we have $0 < \eta < \frac{1}{3\sqrt{2}L}$.
This range of $\eta$ is smaller than $\big(0, \frac{1}{2L}\big)$ for the deterministic FRBS method \cite{malitsky2020forward}.
\end{remark}
}

\beforesubsec
\subsection{\textbf{Oracle complexity of \ref{eq:iFRBS_scheme} with concrete estimators}}\label{subsec:oracle_complexity1}
\aftersubsec
We now estimate the oracle complexity of \eqref{eq:iFRBS_scheme} using concrete unbiased estimators given in Subsection~\ref{subsec:unbiased_estimators}.
Let us start with the increasing mini-batch SGD estimator \eqref{eq:sgd_estimator}.

\begin{corollary}\label{co:sgd_estimator}
Under the same conditions and setting as in Theorem~\ref{th:iFRBS_convergence1}, suppose that the SGD estimator \eqref{eq:sgd_estimator} is used in \eqref{eq:iFRBS_scheme} with an increasing mini-batch size $b_k := \BigOs{k}$.
Then, for any tolerance $\epsilon > 0$, to obtain $\bar{x}^K$ such that $\mbb{E}\big[\norms{G\bar{x}^K + \bar{\xi}^K }^2\big] \leq \epsilon^2$, the expected total number of evaluations of $\mbf{G}(\cdot,\xi)$ is at most $\Exp{\mcal{T}_G} = \BigOs{\epsilon^{-4}\ln(\epsilon^{-1})}$.
In addition, the expected total number of $J_{\eta T}$ evaluations is at most $\Exp{\mcal{T}_J}=\BigOs{\epsilon^{-2}\ln(\epsilon^{-1})}$.
\end{corollary}

\begin{proof}
It is easy to check that $\widetilde{S}^k$ generated by \eqref{eq:sgd_estimator} satisfies Definition~\ref{de:Vr_estimator} if \eqref{eq:b_choice} holds.
However, \eqref{eq:b_choice} holds if we set $\delta_k = \frac{\sigma^2}{b_k}$.
In this case, we get $\kappa = 1$ in Definition~\ref{de:Vr_estimator}.
If we choose $b_k := B(k+1)$ for some $B > 0$, then we have $S_K := \sum_{k=0}^K\delta_k = \frac{\sigma^2}{B}\sum_{k=0}^K\frac{1}{k+1} \leq \frac{2\sigma^2}{B}\ln(K+1)$.
For a given tolerance $\epsilon > 0$, to guarantee $\Exp{\norms{G\bar{x}^K + \bar{\xi}^K}^2} \leq \epsilon^2$, from \eqref{eq:key_convergence_bound1}, we require $ \frac{\ncmt{13}\mcal{R}^2_0}{\ncmt{3}C (K+1) } + \frac{26}{3\kappa C(K+1)} S_K \leq \epsilon^2$, where $\mcal{R}^2_0 := \frac{2}{\eta^2} \norms{x^0 - x^{\star}}^2 + \frac{5}{2}\norms{Gx^0+\xi^0}^2$.
The last condition is satisfied if $K = \BigO{\frac{\mcal{R}_0^2}{\epsilon^2}}$ and $\BigOs{\frac{\sigma^2\ln(K+1)}{K+1}} \leq \epsilon^2$ hold simultaneously.
Clearly, both conditions lead to the choice of $K = \BigOs{ \epsilon^{-2} \ln(\epsilon^{-1}) }$.
Consequently, the expected total number of $\mbf{G}(\cdot,\xi)$ evaluations is at most
\begin{equation*} 
\arraycolsep=0.2em
\begin{array}{lcl}
\Exp{\mcal{T}_G} & = & 2\sum_{k=0}^Kb_k = 2B\sigma^2\sum_{k=0}^K(k+1) = B\sigma^2(K+1)(K+2) = \BigOs{\epsilon^{-4}\ln(\epsilon^{-1})}.
\end{array} 
\end{equation*}
The expected total number of $J_{\eta T}$ evaluations  is at most $\Exp{\mcal{T}_J} = K = \BigOs{ \epsilon^{-2} \ln(\epsilon^{-1}) }$.
\Eproof
\end{proof}

\begin{corollary}\label{co:svrg_estimator}
Under the same condition and setting as in Theorem~\ref{th:iFRBS_convergence1}, suppose that the L-SVRG estimator \eqref{eq:svrg_estimator} is used in \eqref{eq:iFRBS_scheme}.
Then, for a given tolerance $\epsilon > 0$, to obtain $\bar{x}^K$ such that $\mbb{E}[ \norms{G\bar{x}^K + \bar{\xi}^K }^2]  \leq \epsilon^2$, we require the expected total number $\Exp{\mcal{T}_G}$ of evaluations of $\mbf{G}(\cdot, \xi)$, where
\begin{compactitem}
\item[$\mathrm{(a)}$] $\Exp{\mcal{T}_G} = \BigOs{n^{2/3}\epsilon^{-2}}$  for the finite-sum setting $($\textsf{F}$)$, where $\mbf{p} = \BigOs{n^{-1/3}}$ and $b = \BigOs{n^{2/3}}$.
\item[$\mathrm{(b)}$] $\Exp{\mcal{T}_G} = \BigOs{\epsilon^{-10/3}}$  for the expectation setting $($\textsf{E}$)$, where $\mbf{p} = \BigOs{\epsilon^{2/3}}$ and $b = \BigOs{\epsilon^{-4/3}}$.
\end{compactitem}
In both cases, the expected total number of $J_{\eta T}$ evaluations  is at most $\Exp{\mcal{T}_J} = \BigOs{\epsilon^{-2}}$.
\end{corollary}

\begin{proof}
(a)~For the finite-sum case (\textsf{F}), since $\widetilde{S}^k$ is constructed by \eqref{eq:svrg_estimator}, by Lemma~\ref{le:svrg_estimator}, we have $\kappa := \frac{\mbf{p}}{2}$, $\Theta = \frac{16L^2}{b\mbf{p}}$, $\hat{\Theta} = \frac{4L^2}{b\mbf{p}}$, and $\delta_k =  0$.
The second condition $\kappa L^2 \geq \Theta + \hat{\Theta}$ of \eqref{eq:iFRBS_cond1} holds if $\mbf{p} \geq \frac{40}{b\mbf{p}}$, which leads to the choice $b = \frac{40}{\mbf{p}^2}$.
For a given tolerance $\epsilon > 0$, to guarantee $\Exp{\norms{G\bar{x}^K + \bar{\xi}^K}^2} \leq \epsilon^2$, from \eqref{eq:key_convergence_bound1}, we require $ \frac{\ncmt{13}\mcal{R}^2_0}{\ncmt{3}C (K+1) } \leq \epsilon^2$, where $\mcal{R}^2_0 := \frac{2}{\eta^2} \norms{x^0 - x^{\star}}^2 + \frac{5}{2}\norms{Gx^0+\xi^0}^2$.
The last condition leads to $K = \BigOs{\frac{\mcal{R}^2_0}{\epsilon^2}}$.
The expected total number of evaluations of $G_i$  is at most
\begin{equation*} 
\arraycolsep=0.2em
\begin{array}{lcl}
\Exp{\mcal{T}_G} & = & K(n\mbf{p} + 2(1-\mbf{p})b) = \BigOs{\frac{\mcal{R}_0^2}{\epsilon^2}(n\mbf{p} + \frac{1}{\mbf{p}^2})}.
\end{array} 
\end{equation*}
The right-hand side achieves its minimum if we choose $p = \BigOs{n^{-1/3}}$ and $b = \BigOs{n^{2/3}}$, leading to $\Exp{\mcal{T}_G} =  \BigOs{\frac{n^{2/3}\mcal{R}_0^2}{\epsilon^2}}$, which completes the proof of (a).

(b)~For the expectation setting (\textsf{E}), since $\widetilde{S}^k$ is constructed by \eqref{eq:svrg_estimator}, by Lemma~\ref{le:svrg_estimator} again,  we have $\kappa := \frac{\mbf{p}}{2}$, $\Theta = \frac{32L^2}{b\mbf{p}}$, $\hat{\Theta} = \frac{8L^2}{b\mbf{p}}$, and $\delta_k =  \frac{\mbf{p}\sigma^2}{ n_k}$.
The second condition $\kappa L^2 \geq \Theta + \hat{\Theta}$ of \eqref{eq:iFRBS_cond1} holds if $\mbf{p} \geq \frac{80}{b\mbf{p}}$, which leads to the choice of  $b = \frac{80}{\mbf{p}^2}$.
Moreover, we have $S_K := \sum_{k=0}^K\delta_k = \mbf{p}\sigma^2 \sum_{k=0}^K\frac{1}{n_k} = \frac{\mbf{p}(K+1)\sigma^2}{n}$ \ncmt{if we use a fixed mega-batch size} $n_k = n > 0$.
For a given tolerance $\epsilon > 0$, to guarantee $\Exp{\norms{G\bar{x}^K + \bar{\xi}^K}^2} \leq \epsilon^2$, from \eqref{eq:key_convergence_bound1}, we require $ \frac{\ncmt{13}\mcal{R}^2_0}{\ncmt{3}C (K+1) } + \frac{26}{3\kappa C(K+1)} S_K \leq \epsilon^2$.
This condition holds if both conditions $K = \BigOs{\frac{\mcal{R}^2_0}{\epsilon^2}}$ and $\frac{104}{3\mbf{p} C(K+1)} S_K \leq \epsilon^2$ hold simultaneously.
Since $S_K =  \frac{\mbf{p}(K+1)\sigma^2}{n}$, the last \ncmt{condition holds} if $\frac{104\ncmt{\sigma^2}}{3Cn} \leq \epsilon^2$, leading to the choice of $n = \BigOs{\epsilon^{-2}}$.

The expected total number of  evaluations of $\mbf{G}(\cdot, \xi)$ is at most
\begin{equation*} 
\arraycolsep=0.2em
\begin{array}{lcl}
\Exp{\mcal{T}_G} & = & K(n\mbf{p} + 2(1-\mbf{p})b) = \BigOs{\frac{\mcal{R}_0^2}{\epsilon^2}(\frac{\mbf{p}}{\epsilon^2} + \frac{1}{\mbf{p}^2})}.
\end{array} 
\end{equation*}
The right-hand side achieves its minimum if we choose $p = \BigOs{\epsilon^{2/3}}$ and $b = \BigOs{\epsilon^{-4/3}}$, leading to $\Exp{\mcal{T}_G} =  \BigOs{\frac{1}{\epsilon^{10/3}}}$ as stated in Corollary~\ref{co:svrg_estimator}.
Finally, in both cases, the expected total number of $J_{\eta T}$ evaluations is at most $\Exp{\mcal{T}_J} = K =  \BigOs{\epsilon^{-2}}$.
\Eproof
\end{proof}

\begin{corollary}\label{co:saga_estimator}
Under the same conditions and setting as in Theorem~\ref{th:iFRBS_convergence1}, suppose that the SAGA estimator \eqref{eq:saga_estimator} is used in \eqref{eq:iFRBS_scheme}.
Then, for a given $\epsilon > 0$, to achieve $\bar{x}^K$ such that $\mbb{E}[ \norms{G\bar{x}^K + \bar{\xi}^K}^2]  \leq \epsilon^2$, we require the expected total number  of $G_i$ evaluations at most $\Exp{\mcal{T}_G} = \BigOs{ n + n^{2/3}\epsilon^{-2}}$, where $b = \BigOs{n^{2/3}}$.
The expected total number of $J_{\eta T}$ evaluations is at most $\Exp{\mcal{T}_J} = \BigOs{\epsilon^{-2}}$.
\end{corollary}

\begin{proof}
Since $\widetilde{S}^k$ is constructed by \eqref{eq:saga_estimator},  by Lemma~\ref{le:saga_estimator}, we have $\kappa := \frac{b}{2n}$, $\Theta = \frac{16L^2n}{b^2}$, $\hat{\Theta} = \frac{4L^2n}{b^2}$, and $\delta_k =  0$.
The second condition $\kappa L^2 \geq \Theta + \hat{\Theta}$ of \eqref{eq:iFRBS_cond1} holds if $\frac{b}{2n} \geq \frac{20 n}{b^2}$, which leads to the choice of $b = \ncmt{\sqrt[3]{40}} n^{2/3}$.
For a given tolerance $\epsilon > 0$, to guarantee $\Exp{\norms{G\bar{x}^K + \bar{\xi}^K}^2} \leq \epsilon^2$, from \eqref{eq:key_convergence_bound1}, we require $ \frac{\ncmt{13}\mcal{R}^2_0}{\ncmt{3}C (K+1) } \leq \epsilon^2$, where $\mcal{R}^2_0 := \frac{2}{\eta^2} \norms{x^0 - x^{\star}}^2 + \frac{5}{2}\norms{Gx^0+\xi^0}^2$.
The last condition leads to $K = \BigOs{\frac{\mcal{R}^2_0}{\epsilon^2}}$.
The expected total number of $G_i$ evaluations is at most $\Exp{\mcal{T}_G} = n +  2b K  = \BigOs{n + \frac{n^{2/3}\mcal{R}_0^2}{\epsilon^2}}$.
Finally, the expected total number of $J_{\eta T}$ evaluations is at most $\Exp{\mcal{T}_J} = K = \BigOs{\epsilon^{-2}}$.
\Eproof
\end{proof}

The oracle complexity in Corollary~\ref{co:sgd_estimator} is $\BigOs{\epsilon^{-4}\ln(\epsilon^{-1})}$, which matches known results in the literature for increasing mini-batch SGD estimators, see, e.g., \cite{boct2021minibatch}. 
For the finite-sum setting \textsf{(F)}, the oracle complexity $\BigOs{n^{2/3}\epsilon^{-2}}$ in Corollaries~\ref{co:svrg_estimator} and \ref{co:saga_estimator} coincides with our recent work \cite{TranDinh2024}. 
In contrast, for the expectation setting \textsf{(E)}, the oracle complexity $\BigOs{\epsilon^{-10/3}}$ in Corollary~\ref{co:svrg_estimator} is new. 
This complexity is comparable to those obtained for SVRG-type methods in stochastic nonconvex optimization, see, e.g., \cite{Reddi2016a}. 
Although some of these oracle complexity bounds align with existing results, our scheme \eqref{eq:iFRBS_scheme} is different from all of the aforementioned algorithms.

\beforesec
\section{Biased Variance-Reduced FRBS Methods}\label{sec:VrFRBS_method2}
\aftersec
In this section, we develop a class of biased variance-reduced estimators and a new variant of the \ref{eq:iFRBS_scheme} method for solving \eqref{eq:GE}. 
To the best of our knowledge, this is the first attempt to study such biased estimators for the class of nonmonotone problems covered by \eqref{eq:GE} under Assumption~\ref{as:A1}.

\beforesubsec
\subsection{\textbf{The class of biased variance-reduced estimators}}\label{subsec:biased_vr_estimators}
\aftersubsec
We introduce the following class of biased variance-reduced estimators $\widetilde{S}^k$ for $S^k$ defined by \eqref{eq:Sk_term}.

\revise{
\begin{definition}\label{de:Vr_estimator2}
Suppose that $\sets{x^k}_{k\geq 0}$ is generated by \eqref{eq:iFRBS_scheme} and $S^k$ is defined by \eqref{eq:Sk_term}.
For each $k\geq 0$, let $\widetilde{S}^k$ be an $\mathcal{F}_{k+1}$-measurable stochastic estimator of $S^k$, constructed using the fresh randomness sampled at iteration $k$. 
Let  $e^k:=\widetilde{S}^k - S^k$ be the \textit{approximation error}.
Suppose that there exist an $\mathcal{F}_{k+1}$-measurable nonnegative random variable $\Delta_k$, four nonnegative constants $\kappa \in (0, 1]$, $
\tau \in (0, 1]$, $\Theta \geq 0$, and $\hat{\Theta} \geq 0$, and a deterministic nonnegative sequence $\sets{\delta_k}_{k\geq 0}$ such that, almost surely:
\begin{equation}\label{eq:Vr_estimator2}
\arraycolsep=0.2em
\left\{\begin{array}{lcl}
\mathbb{E}_k[ e^k  ] &= & (1-\tau)e^{k-1}, \vspace{1ex}\\
\mathbb{E}_k[ \norms{e^k}^2 ] &\leq & \mathbb{E}_k[ \Delta_k ], \vspace{1ex}\\
\mathbb{E}_k[ \Delta_k ] &\leq & (1 - \kappa)\Delta_{k-1} + \Theta \norms{x^k - x^{k-1}}^2 + \hat{\Theta}\norms{x^{k-1} - x^{k-2} }^2 + \delta_k.
\end{array}\right.
\end{equation}
Here, we choose $\Delta_{-1} := 0$ and  $x^{-1} = x^{-2} := x^0$.
We also choose $(1-\tau)e^{-1} = s^0$ for a given vector $s^0$.
\end{definition}
}
Clearly, if $\tau \in (0, 1)$, then $\widetilde{S}^k$ in Definition~\ref{de:Vr_estimator2} is biased.
\revise{
Moreover, from \eqref{eq:Vr_estimator2}, we have $\Exps{0}{e^0} = (1-\tau)e^{-1} = s^0$ and $\Exps{0}{\norms{e^0}^2} \leq \Exps{0}{\Delta_0} \leq \delta_0$.
Therefore, we can show that $\norms{s^0}^2 = \norms{\Exps{0}{e^0}}^2 \leq \Exps{0}{\norms{e^0}^2} \leq \delta_0$, leading to $\norms{s^0}  \leq \sqrt{\delta_0}$.
}
Next, we show that there are at least three estimators satisfying Definition~\ref{de:Vr_estimator2}.
The first one is the loopless SARAH estimator (L-SARAH) based on \cite{li2020page}, while the second and third estimators are two instances of the Hybrid-SGD class  from \cite{Tran-Dinh2019a}.

\vspace{1ex}
\noindent\textbf{$\mathrm{(a)}$~The L-SARAH estimator.}
The SARAH estimator was introduced in \cite{nguyen2017sarah}, and its loopless variant was proposed in \cite{li2020page}.
In this paper, we extend it to the setting of \eqref{eq:GE}.

Given $\sets{x^k}$ and an i.i.d. mini-batch $\mcal{B}_k$ of a fixed size $b$, for all $k \geq 1$, we construct $\widetilde{S}^k$ as follows:
\begin{equation}\label{eq:sarah_estimator}
\widetilde{S}^k := \begin{cases}
\widetilde{S}^{k-1} + 2G_{\mcal{B}_k}x^k - 3G_{\mcal{B}_k}x^{k-1} + G_{\mcal{B}_k}x^{k-2} &\textrm{with probability $1 -\mbf{p}$}, \vspace{1ex}\\
\overline{S}^k &\textrm{with probability $\mbf{p}$},
\end{cases}
\tag{SARAH}
\end{equation} 
where $x^{-1} = x^0$, $\widetilde{S}^0 := \overline{S}^0$, $G_{\mcal{B}_k}x = \frac{1}{b}\sum_{\xi\in\mcal{B}_k}\mbf{G}(x,\xi)$,  and $\overline{S}^k$ is an unbiased estimator of $S^k$ such that $\Expsn{k}{\overline{S}^k} = S^k$ and $\Expsn{k}{\norms{\overline{S}^k - S^k}^2} \leq \sigma_k^2$ for some $\sigma_k \geq 0$.
We have at least two options of $\overline{S}^k$.
\begin{compactitem}
\item[(i)](\textit{Exact evaluation}). We choose $\overline{S}^k = S^k = 2Gx^k - Gx^{k-1}$ as the exact evaluation of $S^k$.
In this case, we have $\Expsn{k}{\norms{\overline{S}^k - S^k}^2} = 0$, leading to $\sigma_k = 0$.
\item[(ii)](\textit{Mega-batch evaluation}). We set $\overline{S}^k = \frac{1}{n_k}\sum_{\xi\in\mcal{M}_k}(2\mbf{G}(x^k, \xi) - \mbf{G}(x^{k-1}, \xi) )$, where $\mcal{M}_k$ is a mega-batch of size $n_k$.
In this case, we get $\sigma_k^2 = \frac{\sigma^2}{n_k}$ for $\sigma$ given in Assumption~\ref{as:A1}(b).
\end{compactitem}
We have the following lemma, whose proof can be found in Appendix~\ref{apdx:le:sarah_estimator}.

\begin{lemma}\label{le:sarah_estimator}
Let $\widetilde{S}^k$ be an L-SARAH estimator generated by \eqref{eq:sarah_estimator} to approximate $S^k$ defined by \eqref{eq:Sk_term}.
Then, it satisfies Definition~\ref{de:Vr_estimator2} with $\tau = \mbf{p}$, $\kappa := \mbf{p}$, $\Theta = \frac{8(1-\mbf{p})L^2}{b}$, $\hat{\Theta} = \frac{2(1-\mbf{p})L^2}{b}$, and $\delta_k := \mbf{p}\sigma^2_k$.
\end{lemma}

\vspace{1ex}
\noindent\textbf{$\mathrm{(b)}$~The Hybrid-SGD estimator.}
The class of Hybrid-SGD estimators was introduced in \cite{Tran-Dinh2019a} for nonconvex optimization, and now we extend it to handle \eqref{eq:GE}.

Given $\sets{x^k}_{k\geq 0}$ generated by our method, one \ncmt{i.i.d.}  mini-batch $\mcal{B}_k$ of size $b$, and another \ncmt{i.i.d.} mini-batch $\hat{\mcal{B}}_k$ of size $\hat{b}$ ($\hat{\mcal{B}}_k$ may overlap with $\mcal{B}_k$), we construct $\widetilde{S}^k$ as follows:
\begin{equation}\label{eq:hsgd_estimator}\tag{HSGD}
\widetilde{S}^k :=  (1 - \omega)\big[ \widetilde{S}^{k-1} + 2G_{\mcal{B}_k}x^k - 3G_{\mcal{B}_k}x^{k-1} + G_{\mcal{B}_k}x^{k-2} \big] + \omega\overline{S}_{\hat{\mcal{B}}_k}^k,
\end{equation} 
where $\omega \in \ncmt{(0, 1]}$, $x^{-2} = x^{-1} = x^0$, $\widetilde{S}^0$ is a given initial estimator of $S^0$, $G_{\mcal{B}_k}x = \frac{1}{b}\sum_{\xi\in\mcal{B}_k}\mbf{G}(x,\xi)$,  and $\overline{S}^k_{\hat{\mcal{B}}_k}$ is an unbiased estimator of $S^k$ constructed from $\hat{\mcal{B}}_k$ such that $\Expsn{\hat{\mcal{B}}_k}{\overline{S}^k_{\hat{\mcal{B}}_k}} = S^k$ and $\sigma_k^2 := \Expsn{\hat{\mcal{B}}_k}{\norms{\overline{S}^k_{\hat{\mcal{B}}_k} - S^k}^2} \leq \bar{\sigma}_k^2$.
We have the following lemma, whose proof is deferred to Appendix~\ref{apdx:le:hsgd_estimator}.

\begin{lemma}\label{le:hsgd_estimator}
Let $\widetilde{S}^k$ be the Hybrid-SGD estimator generated by \eqref{eq:hsgd_estimator} to approximate $S^k$ defined by \eqref{eq:Sk_term}.
Then, it satisfies Definition~\ref{de:Vr_estimator2} with $\tau := \ncmt{\omega} \in \ncmt{(0, 1]}$, $\kappa := \ncmt{\omega(2-\omega)} \in \ncmt{(0, 1]}$, $\Theta = \frac{8C(1-\omega)^2L^2}{b}$, $\hat{\Theta} = \frac{2C(1-\omega)^2L^2}{b}$, and $\delta_k := C\omega^2 \sigma^2_k$, where $C := 1$ if $\mcal{B}_k$ is independent of $\hat{\mcal{B}}_k$, and $C := 2$, otherwise.
\end{lemma}

\vspace{1ex}
\noindent\textbf{$\mathrm{(c)}$~The Hybrid-SVRG estimator.}
We modify the \eqref{eq:hsgd_estimator} estimator by replacing the unbiased SGD term $\overline{S}^k_{\hat{\mathcal{B}}_k}$ by the SVRG estimator \eqref{eq:svrg_estimator}, leading to the following Hybrid-SVRG estimator.
Given $\sets{x^k}_{k\geq 0}$ generated by our method, one \ncmt{i.i.d.} mini-batch $\mcal{B}_k$ of size $b$, and another \ncmt{i.i.d.}  mini-batch $\hat{\mcal{B}}_k$ of size $\hat{b}$ ($\hat{\mcal{B}}_k$ can overlap with $\mcal{B}_k$), we construct $\widetilde{S}^k$ as follows:
\begin{equation}\label{eq:hsvrg_estimator}\tag{HSVRG}
\widetilde{S}^k :=  (1 - \omega)\big[ \widetilde{S}^{k-1} + 2G_{\mcal{B}_k}x^k - 3G_{\mcal{B}_k}x^{k-1} + G_{\mcal{B}_k}x^{k-2} \big] + \omega\overline{S}_{\hat{\mcal{B}}_k}^k,
\end{equation} 
where $\omega \in \ncmt{(0, 1]}$, $x^{-2} = x^{-1} = x^0$, $\widetilde{S}^0$ is a given initial estimator of $S^0$, $G_{\mcal{B}_k}x = \frac{1}{b}\sum_{\xi\in\mcal{B}_k}\mbf{G}(x,\xi)$,  and $\overline{S}^k_{\hat{\mcal{B}}_k}$ is the \eqref{eq:svrg_estimator} estimator.
We have the following lemma, whose proof is in Appendix~\ref{apdx:le:hsvrg_estimator}.

\begin{lemma}\label{le:hsvrg_estimator}
Let $\widetilde{S}^k$ be the Hybrid-SVRG estimator generated by \eqref{eq:hsvrg_estimator} to approximate $S^k$. 
Then, it satisfies Definition~\ref{de:Vr_estimator2} with 
\begin{equation*}
\begin{array}{ll}
&\tau := \omega \in (0,1], 
\quad \kappa := \min\sets{\omega(2-\omega), \frac{\mbf{p}}{4}} \in (0,1], 
\quad \Theta := 8L^2 \big[ \frac{C(1-\omega)^2}{b} + \frac{8\ncmt{(1+\mu)}\omega^2}{\hat{b}\mbf{p}^2} \big], \vspace{1ex}\\
&\hat{\Theta} := 2L^2 \big[\frac{C(1-\omega)^2}{b} + \frac{8\ncmt{(1+\mu)}\omega^2}{\hat{b}\mbf{p}^2} \big], 
\quad \text{and} \quad \delta_k := \frac{2(1+\mu)\omega^2 \sigma^2}{\mu n_k},
\end{array} 
\end{equation*}
where
\begin{compactitem}[$\diamond$]
	\item $C := 1$ if $\mcal{B}_k$ is independent of $\hat{\mcal{B}}_k$, and $C := 2$, otherwise;
	\item \ncmt{ $\mu := 0$ if $\overline{G}w^k$ is evaluated exactly in the SVRG term $\overline{S}^k_{\hat{\mcal{B}}_k}$ with $\overline{G}w^k = Gw^k$; and $\mu > 0$ arbitrarily if $\overline{G}w^k = \frac{1}{n_k}\sum_{\xi \in \Mc_k} \mbf{G}(w^k, \xi)$ is a mega-batch estimator of size $n_k$ in the SVRG term $\overline{S}^k_{\hat{\mcal{B}}_k}$.}
\end{compactitem}
\end{lemma}
Here, we use the convention $\frac{c}{0} := 0$ for any $c > 0$.
Note that if we choose $\Bc_k = \hat{\Bc}_k$, then the \eqref{eq:hsvrg_estimator} estimator requires $2b$ evaluations of $\mbf{G}(\cdot,\xi)$ per iteration \ncmt{(with probability $1 - \mbf{p}$)}, which is the same as in  \eqref{eq:hsgd_estimator} as well as in \eqref{eq:svrg_estimator}.

\beforesubsec
\subsection{\textbf{Convergence analysis of \ref{eq:iFRBS_scheme} with biased variance-reduced estimators}}\label{subsec:bvrFRBS_method}
\aftersubsec
We analyze the convergence of \eqref{eq:iFRBS_scheme} using a biased estimator $\widetilde{S}^k$ satisfying Definition~\ref{de:Vr_estimator2}.
For this purpose, we consider the following intermediate function:
\begin{equation}\label{eq:Qk_func}
\arraycolsep=0.2em
\begin{array}{lcl}
\Qc_k &:= & \norms{y^k - x^{\star}}^2 + (1 - \gamma \eta)\norms{y^k - x^{k-1}}^2 + \frac{2\eta}{c_1\tau}\norms{x^{k-1} -  y^{k-1}}^2  +  \frac{2\eta( 1-\tau) }{\tau}\iprods{e^{k-1}, x^{\star} - y^{k-1}},
\end{array}
\end{equation}
where $x^{\star} \in \zerm{M}{\Phi}$, and $\gamma > 0$ and $c_1 > 0$ are given constants, determined later.

We start our analysis with the following two technical lemmas,  whose proofs can be found in  Appendix~\ref{apdx:le:A2_key_estimate1} and Appendix~\ref{apdx:le:A2_key_estimate2}, respectively.

\begin{lemma}\label{le:A2_key_estimate1}
Suppose that Assumption~\ref{as:A1} holds for \eqref{eq:GE}.
Let $\sets{x^k}$ be generated by \eqref{eq:iFRBS_scheme} using $\widetilde{S}^k$ that satisfies Definition~\ref{de:Vr_estimator2}.
Then, for any positive constants $\gamma, c, c_1, c_2$, we have
\begin{equation}\label{eq:A2_key_estimate1}
\arraycolsep=0.2em
\begin{array}{lcl}
\Exps{k}{ \Qc_{k+1} }  &\leq &\Qc_k  - \Big[ 1 - \frac{2\eta}{c_1\tau}  - \frac{4\rho}{\eta} - \frac{2}{c_2} - 2L^2\eta\big(c + 4\rho + \frac{1}{\gamma} + \frac{2\eta}{c_2}\big) \Big] \norms{y^k - x^k}^2 \vspace{1ex}\\
&& - {~} \Big\{ 1 - \gamma \eta - \Big[ \frac{2\eta}{c_1 \tau} + 2L^2\eta\big(c + 4\rho + \frac{1}{\gamma} + \frac{2\eta}{c_2}\big)   \Big] \Big\} \norms{y^k - x^{k-1}}^2 \vspace{1ex}\\
&& - {~} cL^2\eta \norms{x^k -  x^{k-1} }^2 + \eta\big( 2\eta + \frac{c_1}{\tau} + c_2\eta \big)\Exps{k}{\norms{e^k}^2}.
\end{array} 
\end{equation}
In addition, we also have
\begin{equation}\label{eq:A2_Qk_lower_bound}
\hspace{-2ex}
\arraycolsep=0.2em
\begin{array}{lcl}
\Qc_k & \geq &  \frac{3}{4} \norms{y^k - x^{\star}}^2 + \frac{1 - 2\gamma \eta}{2} \norms{y^k - x^{k-1}}^2  + \big( \frac{2\eta}{c_1\tau} - \frac{1}{2} \big) \norms{x^{k-1} -  y^{k-1}}^2  - \frac{8(1-\tau)^2\eta^2}{\tau^2}\norms{e^{k-1}}^2.
\end{array} 
\hspace{-1ex}
\end{equation}
\end{lemma}

\begin{lemma}\label{le:A2_key_estimate2}
For $\Qc_k$ defined by \eqref{eq:Qk_func} \revise{with $c_1 :=  \frac{4\eta}{\tau}$ and $\gamma := \frac{1}{4\eta}$}, we consider the following Lyapunov function:
\begin{equation}\label{eq:A2_Lk_func}
\hspace{-1ex}
\arraycolsep=0.2em
\begin{array}{lcl}
\Lc_k &:= & \Qc_k + \frac{\hat{\Theta}\eta^2}{\kappa}\big(\frac{12}{\tau^2} + \ncmt{10} \big)\norms{x^{k-1} - x^{k-2}}^2 + \frac{(1-\kappa)\eta^2}{\kappa}\big(\frac{12}{\tau^2} + \ncmt{10} \big) \Delta_{k-1} +  \frac{8(1-\tau)^2\eta^2}{\tau^2}\norms{e^{k-1}}^2.
\end{array}
\hspace{-1ex}
\end{equation}
Under the same condition and setting as in Lemma~\ref{le:A2_key_estimate1}, we have the following almost sure estimate:
\begin{equation}\label{eq:A2_key_estimate2}
\arraycolsep=0.2em
\begin{array}{lcl}
\Exps{k}{ \Lc_{k+1} }  &\leq &\Lc_k  - \Big[ \frac{1}{4}  -  \frac{4\rho}{\eta}  - \ncmt{\frac{L^2\eta}{2}\big(33\eta + 16\rho\big)} \Big] \norms{y^k - x^k}^2 \vspace{1ex}\\
&& - {~} \Big[ \frac{1}{4} - \ncmt{\frac{L^2\eta}{2}\big(33\eta + 16\rho\big)} \Big] \ \norms{y^k - x^{k-1}}^2 - \frac{8(1-\tau)^2\eta^2}{\tau^2}\norms{e^{k-1}}^2 \vspace{1ex}\\
&& - {~} \eta^2 \big[ 4L^2 - \big(\frac{12}{\tau^2} + \ncmt{10} \big)\frac{\Theta + \hat{\Theta}}{\kappa} \big] \norms{x^k -  x^{k-1} }^2 + \big(\frac{12}{\tau^2} + \ncmt{10} \big)\frac{\eta^2\delta_k}{\kappa}.
\end{array} 
\end{equation}
Moreover, we also have the following almost sure lower bound
\begin{equation}\label{eq:A2_Lk_lower_bound}
\arraycolsep=0.2em
\begin{array}{lcl}
\Lc_k & \geq & \frac{3}{4}\norms{y^k - x^{\star}}^2.
\end{array}
\end{equation}
\end{lemma}
Our second main result of this paper is the following theorem.

\begin{theorem}\label{th:iFRBS_convergence2}
Suppose that Assumption~\ref{as:A1} holds for \eqref{eq:GE} such that $L\rho < \frac{1}{32\sqrt{134}}$.
Let $\sets{(x^k , \bar{x}^k)}$ be generated by \eqref{eq:iFRBS_scheme} using a biased estimator $\widetilde{S}^k$ for $S^k$ satisfying Definition~\ref{de:Vr_estimator2}.
Suppose further that the stepsize $\eta$ and the parameter quadruple  $(\tau, \kappa, \Theta, \hat{\Theta})$ in Definition~\ref{de:Vr_estimator2} satisfy
\begin{equation}\label{eq:A2_iFRBS_cond1}
32\rho \leq \eta < \tfrac{1}{L\sqrt{\ncmt{134}}},  \quad \text{and} \quad 4L^2\tau^2\kappa \geq \ncmt{22}(\Theta + \hat{\Theta}).
\end{equation}
Then, for any $K \geq 0$, \revise{any $x^{\star}\in \zerm{M}{\Phi}$}, and for $\bar{x}^K$ constructed by \eqref{eq:iFRBS_output}, we have
\begin{equation}\label{eq:A2_key_convergence_bound1}
\arraycolsep=0.2em
\begin{array}{lcl}
\Exp{\norms{G\bar{x}^K + \bar{\xi}^K}^2}  & = & \frac{1}{K+1}\sum_{k=0}^K\Exp{\norms{Gx^k + \xi^k}^2} \vspace{1ex}\\
& \leq &  \frac{\ncmt{17}}{C_1 (K+1) } \Big[ \frac{2}{\eta^2} \norms{x^0 - x^{\star}}^2 + \ncmt{\frac{11}{4}}\norms{Gx^0+\xi^0}^2 \Big]  +  \frac{\ncmt{356}}{C_1\tau^2\kappa (K+1)} S_K,
\end{array}
\end{equation}
where $S_K := \sum_{k=0}^K\delta_k$, $C_1 := 1 - \ncmt{134L^2\eta^2} > 0$, $\xi^k \in Tx^k$ for $0 \leq k\leq K$, and $\bar{\xi}^K \in T\bar{x}^K$.

If additionally $S_{\infty} := \sum_{k=0}^{\infty}\delta_k < +\infty$, then almost surely, we have
\begin{equation}\label{eq:key_convergence_bound1b}
\sum_{k=0}^{\infty} \norms{Gx^k + \xi^k}^2 < +\infty \quad \textrm{and} \quad  \lim_{k\to\infty} \norms{Gx^k + \xi^k} = 0 \quad \textrm{for some}~\xi^k \in Tx^k.
\end{equation}
Moreover,  if $\gra{T}$ is closed \revise{and $\zerm{M}{\Phi} = \zer{\Phi}$}, then $\sets{x^k}$ converges almost surely to a  $\zer{\Phi}$-valued random variable $x^{\star}$, a  solution of \eqref{eq:GE}.
\end{theorem}

\begin{proof}
\ncmt{First, let us choose $\eta > 0$ such that $\frac{4\rho}{\eta} \leq \frac{1}{8}$, leading to $\eta \geq 32\rho$.
In this case, we have $\frac{1}{4}  -  \frac{4\rho}{\eta}  - \frac{L^2\eta}{2}\big(33\eta + 16\rho\big) \geq \frac{1}{8}(1 - 134L^2\eta^2) =: \frac{C_1}{8} > 0$ if we choose $32\rho \leq \eta < \frac{1}{L\sqrt{134}}$ as in \eqref{eq:A2_iFRBS_cond1}, provided that $L\rho < \frac{1}{32\sqrt{134}}$.
Moreover, since $\tau^2 \leq 1$, the condition $4L^2\tau^2\kappa \geq \ncmt{22}(\Theta + \hat{\Theta})$ from  \eqref{eq:A2_iFRBS_cond1} also implies $4L^2 - \big(\frac{12}{\tau^2} + \ncmt{10} \big)\frac{\Theta + \hat{\Theta}}{\kappa} \geq 4L^2 - \frac{22(\Theta + \hat{\Theta})}{\tau^2\kappa} \geq 0$.}
In this case, we can simplify \eqref{eq:A2_key_estimate2} as follows:
\begin{equation}\label{eq:A2_th1_proof1}
\hspace{-1ex}
\arraycolsep=0.2em
\begin{array}{lcl}
\Exps{k}{ \Lc_{k+1} }  &\leq &\Lc_k  - \frac{C_1}{\ncmt{8}} \big[ \norms{y^k - x^k}^2 + \norms{y^k - x^{k-1}}^2 \big] - \frac{8(1-\tau)^2\eta^2}{\tau^2}\norms{e^{k-1}}^2 +  \frac{\ncmt{22}\eta^2\delta_k}{\tau^2\kappa}.
\end{array} 
\hspace{-1ex}
\end{equation}
Next, taking the full expectation on both sides of \eqref{eq:A2_th1_proof1} and dropping the penultimate term, we obtain
\begin{equation*} 
\arraycolsep=0.2em
\begin{array}{lcl}
\Exp{ \Lc_{k+1} }  &\leq & \Exp{ \Lc_k }  - \frac{C_1}{\ncmt{8}} \Exp{ \norms{y^k - x^k}^2 + \norms{y^k - x^{k-1}}^2 } +  \frac{\ncmt{22}\eta^2\delta_k}{\tau^2\kappa} \vspace{1ex}\\
& \leq & \Exp{ \Lc_k }  - \frac{C_1\eta^2}{ \ncmt{8}( 2 + 3L^2\eta^2 )}\norms{w^k}^2 + \frac{\ncmt{22}\eta^2\delta_k}{\tau^2\kappa},
\end{array} 
\end{equation*}
where we have used \eqref{eq:A1_th1_proof5} in the last inequality.
Telescoping this inequality from $k :=0$ to $k := K$ and denoting $S_K := \sum_{k=0}^K\delta_k$, we can show that
\begin{equation}\label{eq:A2_th1_proof3} 
\arraycolsep=0.2em
\begin{array}{lcl}
\frac{1}{K+1} \sum_{k=0}^K\Exp{\norms{w^k}^2 } &\leq & \frac{\ncmt{8}( 2 + 3L^2\eta^2)}{C_1\eta^2 (K+ 1)}\Exp{\Lc_0} +  \frac{\ncmt{176}(2 + 3L^2\eta^2)}{C_1\tau^2\kappa(K+1)}S_K.
\end{array} 
\end{equation}
Now, since \ncmt{$\gamma = \frac{1}{4\eta}$,} $x^{-2} = x^{-1} \ncmt{= y^{-1}} = x^0$, $\ncmt{e^{-1}} = 0$, \ncmt{and $y^0 = x^0 + \eta w^0$,}  we have
\begin{equation*} 
\arraycolsep=0.2em
\begin{array}{lcl}
\Exp{\Lc_0} &\leq & \ncmt{\Exp{\Qc_0}} \leq  \norms{x^0 + \eta w^0 - x^{\star}}^2 + (1-\gamma\eta)\eta^2\norms{w^0}^2 \leq 2\norms{x^0 - x^{\star}}^2 + \ncmt{\frac{11\eta^2}{4}}\norms{w^0}^2.
\end{array} 
\end{equation*}
Using this inequality and $L^2\eta^2 \leq \frac{1}{\ncmt{134}}$, we obtain from \eqref{eq:A2_th1_proof3} the estimate \eqref{eq:A2_key_convergence_bound1}.

Finally, the almost sure summability bound $\sum_{k=0}^{\infty}\norms{Gx^k + \xi^k}^2 < +\infty$ and the almost-sure convergence of $\sets{x^k}$ in the remaining parts are similar to the proof of Theorem~\ref{th:iFRBS_convergence1}.
\Eproof
\end{proof}

In terms of convergence rates, Theorem~\ref{th:iFRBS_convergence2} still achieves a $\BigOs{1/k}$ best-iterate rate for the expected squared residual norm $\mbb{E}[\norms{Gx^k + \xi^k}^2]$, as in Theorem~\ref{th:iFRBS_convergence1}. 
However, the second condition $4\kappa\tau^2L^2 \geq \ncmt{22}(\Theta + \hat{\Theta})$ in \eqref{eq:A2_iFRBS_cond1} differs from the requirement $\kappa L^2 \geq \Theta + \hat{\Theta}$ in \eqref{eq:iFRBS_cond1}, which leads to different oracle complexities. 
To the best of our knowledge, Theorem~\ref{th:iFRBS_convergence2} is the first result in the literature to establish both a convergence rate and almost-sure convergence for \mytbi{biased variance-reduced methods} that solve \eqref{eq:GE} under Assumption~\ref{as:A1}.

\beforesubsec
\subsection{\textbf{Oracle complexity of \ref{eq:iFRBS_scheme} with concrete estimators}}\label{subsec:oracle_complexity2}
\aftersubsec
In this subsection, we derive the oracle complexity of \eqref{eq:iFRBS_scheme} using three concrete estimators: L-SARAH, Hybrid-SGD, and Hybrid-SVRG presented in Subsection~\ref{subsec:biased_vr_estimators}.

\begin{corollary}\label{co:sarah_compexity}
Under the same conditions and settings as in Theorem~\ref{th:iFRBS_convergence2}, we assume that the L-SARAH estimator $\widetilde{S}^k$ defined by \eqref{eq:sarah_estimator} is used in \eqref{eq:iFRBS_scheme}.
Then, for a given tolerance $\epsilon > 0$, the expected total number $\Exp{\mcal{T}_G}$ of $\mbf{G}(\cdot,\xi)$ calls to achieve $\mathbb{E}[ \norms{G\bar{x}^K + \bar{\xi}^K }^2 ] \leq \epsilon^2$ is at most 
\begin{compactitem}
\item[$\mathrm{(a)}$] $\BigOs{n^{3/4}\epsilon^{-2}}$ for the finite-sum setting $($\textsf{F}$)$, where $\overline{S}^k := S^k$, $\mbf{p} := \BigOs{n^{-1/4}}$, and $b := \BigOs{n^{3/4}}$.
\item[$\mathrm{(b)}$] $\BigOs{\epsilon^{-5}}$ for the expectation setting $($\textsf{E}$)$, where $\mbf{p} = \BigOs{\epsilon}$, $b = \BigOs{\epsilon^{-3} }$, and $n_k = n = \BigOs{\epsilon^{-4}}$.
\end{compactitem}
In both settings, the expected total number of $J_{\eta T}$ evaluations is at most $\Exp{\Tc_J} = \BigOs{\epsilon^{-2}}$.
\end{corollary}

\begin{proof}
(a)~For the finite-sum setting (\textsf{F}), since $\overline{S}^k = S^k$, we have $S_K \leq S_{\infty} = 0$, and \eqref{eq:A2_key_convergence_bound1} reduces to
\begin{equation*} 
\arraycolsep=0.2em
\begin{array}{lcl}
\Exp{\norms{G\bar{x}^K + \bar{\xi}^K}^2}  & \leq &  \frac{\ncmt{17}\mcal{R}_0^2}{C_1 (K+1) }, \quad \textrm{where} \quad \mcal{R}_0^2 := \frac{2}{\eta^2} \norms{x^0 - x^{\star}}^2 + \ncmt{\frac{11}{4}}\norms{Gx^0+\xi^0}^2.
\end{array}
\end{equation*}
For a given tolerance $\epsilon > 0$, to achieve $\Exp{\norms{G\bar{x}^K + \bar{\xi}^K}^2} \leq \epsilon^2$, from the last inequality, we require $K = \frac{\ncmt{17}\mcal{R}_0^2}{C_1\epsilon^2}$.
Moreover, from Lemma~\ref{le:sarah_estimator}, since $\tau = \mbf{p}$, $\kappa = \frac{\mbf{p}}{2}$, $\Theta = \frac{8(1-\mbf{p})L^2}{b}$, and $\hat{\Theta} = \frac{2(1-\mbf{p})L^2}{b}$, the second condition $4L^2\tau^2\kappa \geq \ncmt{22}(\Theta + \hat{\Theta})$ in \eqref{eq:A2_iFRBS_cond1} holds if $\mbf{p}^3 \geq \frac{\ncmt{110}}{b}$.
Hence, we can choose $b := \BigOs{\frac{1}{\mbf{p}^3}} \geq \frac{\ncmt{110}}{\mbf{p}^3}$ to fulfill this condition. 
In this case, the expected total number of $G_i$ evaluations is at most 
\begin{equation*}
\arraycolsep=0.2em
\begin{array}{lcl}
\Exp{\mcal{T}_{G_i}} &= & K( 2\mbf{p}n + 2\ncmt{(1-\mbf{p})}b ) = \BigOs{ \frac{\mcal{R}_0^2( \mbf{p}n + \mbf{p}^{-3})}{\epsilon^2} }.
\end{array}
\end{equation*}
Clearly, if we choose $\mbf{p} := \BigOs{n^{-1/4}}$, then $\Exp{\mcal{T}_{G_i}} = \BigOs{\frac{n^{3/4}\mcal{R}_0^2}{\epsilon^2} }$.
The expected total number of $J_{\eta T}$ evaluations is $\Exp{\mcal{T}_J} = K = \BigOs{\epsilon^{-2}}$.

(b)~For the expectation setting (\textsf{E}), we have $S_K := \sum_{k=0}^K\delta_k = \sum_{k=0}^K\frac{\mbf{p}\sigma^2}{n} = \frac{\mbf{p}(K+1)\sigma^2}{n}$.
From Lemma~\ref{le:sarah_estimator}, since $\tau = \mbf{p}$, $\kappa = \frac{\mbf{p}}{2}$, $\Theta = \frac{8(1-\mbf{p})L^2}{b}$, and $\hat{\Theta} = \frac{2(1-\mbf{p})L^2}{b}$, to achieve $\Exp{\norms{G\bar{x}^K + \bar{\xi}^K}^2} \leq \epsilon^2$, from \eqref{eq:A2_key_convergence_bound1}, we require 
\begin{equation*}
\arraycolsep=0.2em
\begin{array}{lcl}
\frac{\ncmt{17}\mcal{R}_0^2}{C_1 (K+1) } +  \frac{\ncmt{712}}{C_1\mbf{p}^3 (K+1)}S_K =   \frac{\ncmt{17}\mcal{R}_0^2}{C_1 (K+1) } +  \frac{\ncmt{712}\sigma^2}{C_1\mbf{p}^2n}  \leq \epsilon^2.
\end{array}
\end{equation*}
This condition holds if both $\frac{\ncmt{34}\mcal{R}_0^2}{C_1 (K+1) } \leq \epsilon^2$ and $ \frac{\ncmt{1424}\sigma^2}{C_1\mbf{p}^2n}  \leq \epsilon^2$ are simultaneously satisfied.
These conditions hold if we choose $K = \BigOs{\frac{\mcal{R}^2_0}{\epsilon^2}}$ and $\mbf{p} = \BigOs{\frac{1}{\epsilon\sqrt{n}}}$.
In this case, we still choose $b = \BigOs{ \frac{1}{\mbf{p}^3}}$ to guarantee the second condition of \eqref{eq:A2_iFRBS_cond1}.
As a consequence,  the expected total number of $\mbf{G}(\cdot,\xi)$ calls is at most
\begin{equation*}
\arraycolsep=0.2em
\begin{array}{lcl}
\Exp{\Tc_G} &= & K(2n\mbf{p} + 2\ncmt{(1-\mbf{p})}b) = \BigO{\frac{\mcal{R}_0^2}{\epsilon^2}\big( \ncmt{\frac{\sqrt{n}}{\epsilon} + n^{3/2}\epsilon^3}\big)}.
\end{array}
\end{equation*}
Clearly, if we choose $n^{3/2}\epsilon^3 = \BigOs{ \frac{n^{1/2}}{\epsilon}}$, or equivalently, $n = \BigOs{\frac{1}{\epsilon^4}}$, then we obtain $\Exp{\Tc_G} = \BigOs{\epsilon^{-5}}$.
The expected total number of $J_{\eta T}$ evaluations remains $\Exp{\mcal{T}_J} = K = \BigOs{\epsilon^{-2}}$.
\Eproof
\end{proof}

\begin{corollary}\label{co:hsgd_compexity}
Under the same conditions and settings as in Theorem~\ref{th:iFRBS_convergence2}, we assume that the Hybrid-SGD estimator $\widetilde{S}^k$ defined by \eqref{eq:hsgd_estimator} is used in \eqref{eq:iFRBS_scheme}  with $b = \hat{b} = \BigOs{\epsilon^{-3}}$ for a given tolerance $\epsilon > 0$.
Then, the expected total numbers of $\mbf{G}(\cdot,\xi)$ calls and $J_{\eta T}$ evaluations to achieve $\mathbb{E}[ \norms{G\bar{x}^K + \bar{\xi}^K }^2 ] \leq \epsilon^2$ are  at most $\mbb{E}[ \Tc_G ] = \BigOs{\epsilon^{-5} }$ and $\Exp{\mcal{T}_J} = \BigOs{\epsilon^{-2}}$, respectively.
\end{corollary}

\begin{proof}
Since $\widetilde{S}^k$ is constructed by \eqref{eq:hsgd_estimator}, by Lemma~\ref{le:hsgd_estimator}, we have \ncmt{$\tau := \omega$, $\kappa := \omega(2-\omega)$,} $\Theta = \frac{16(1-\omega)^2L^2}{b}$, $\hat{\Theta} = \frac{4(1-\omega)^2L^2}{b}$, and $\delta_k := 2\omega^2 \sigma^2_k$.
Thus, we get $S_K = \sum_{k=0}^K\delta_k = \frac{2\sigma^2\omega^2(K+1)}{\hat{b}}$.
The second condition $4L^2\tau^2\kappa \geq \ncmt{22}(\Theta + \hat{\Theta})$ in \eqref{eq:A2_iFRBS_cond1} holds if $b = \BigOs{\frac{1}{\omega^3}}$.
To achieve $\Exp{\norms{G\bar{x}^K + \bar{\xi}^K}^2} \leq \epsilon^2$, from \eqref{eq:A2_key_convergence_bound1}, we require 
\begin{equation*}
\arraycolsep=0.2em
\begin{array}{lcl}
\frac{\ncmt{17}\mcal{R}_0^2}{C_1 (K+1) } +  \frac{\ncmt{356}}{C_1\kappa\tau^2 (K+1)}S_K \leq \frac{\ncmt{17}\mcal{R}_0^2}{C_1 (K+1) } +  \frac{\ncmt{712}\sigma^2}{C_1\omega \hat{b}}  \leq \epsilon^2.
\end{array}
\end{equation*}
This condition holds if we choose $K = \BigOs{\frac{\mcal{R}_0^2}{\epsilon^2}}$ and $\hat{b} = \BigOs{\frac{\sigma^2}{\epsilon^2\omega}}$.
Therefore, the expected total number of $\mbf{G}(\cdot,\xi)$ evaluations is at most
\begin{equation*}
\arraycolsep=0.2em
\begin{array}{lcl}
\Exp{\Tc_G} &= & K(2b + 2\hat{b}) = \BigOs{\frac{\mcal{R}_0^2}{\epsilon^2}\big( \frac{1}{\omega^3} + \frac{\sigma^2}{\epsilon^2\omega} \big)}.
\end{array}
\end{equation*}
In particular, if we choose $\omega = \BigOs{\epsilon}$, then we get $\Exp{\Tc_G} = \BigOs{\epsilon^{-5}}$.
In this case, both $b$ and $\hat{b}$ are $\BigOs{\epsilon^{-3}}$.
The expected total number of $J_{\eta T}$ evaluations remains  $\Exp{\mcal{T}_J} = K = \BigOs{\epsilon^{-2}}$.
\Eproof
\end{proof}

\begin{corollary}\label{co:hsvrg_compexity}
Under the same conditions and settings as in Theorem~\ref{th:iFRBS_convergence2}, we assume that the H-SVRG estimator $\widetilde{S}^k$ defined by \eqref{eq:hsvrg_estimator} is used in \eqref{eq:iFRBS_scheme}.
Then, for a given tolerance $\epsilon > 0$, the expected total number $\Exp{\mcal{T}_G}$ of $\mbf{G}(\cdot,\xi)$ calls to achieve $\mathbb{E}[ \norms{G\bar{x}^K + \bar{\xi}^K }^2 ] \leq \epsilon^2$ is at most 
\begin{compactitem}
	\item[$\mathrm{(a)}$] $\BigOs{n^{3/4}\epsilon^{-2}}$ for the finite-sum setting $($\textsf{F}$)$, where we choose $\overline{S}^k := S^k$, $\mbf{p} := \BigOs{n^{-1/4}}$, $b := \BigOs{n^{3/4}}$, and $\hat{b} := \BigOs{n^{3/4}}$.
	\item[$\mathrm{(b)}$] \ncmt{$\BigOs{\epsilon^{-5}}$ for the expectation setting $($\textsf{E}$)$, where we choose $\mbf{p} = \BigOs{\epsilon}$, $b = \BigOs{\epsilon^{-3} }$, $\hat{b} = \BigOs{\epsilon^{-3} }$, and $n_k = n = \BigOs{\epsilon^{-3}}$.}
\end{compactitem}
In both settings, the expected total number of $J_{\eta T}$ evaluations is at most $\Exp{\Tc_J} = \BigOs{\epsilon^{-2}}$.
\end{corollary}

\begin{proof}
\ncmt{(a)~For the finite-sum setting (\textsf{F}), since we use the exact evaluation for $\overline{G}w^k$, we have $\delta_k = 0$, leading to $S_K \leq S_{\infty} = 0$.}
Then, \eqref{eq:A2_key_convergence_bound1} reduces to
\begin{equation*} 
\arraycolsep=0.2em
\begin{array}{lcl}
\Exp{\norms{G\bar{x}^K + \bar{\xi}^K}^2}  & \leq &  \frac{\ncmt{17}\mcal{R}_0^2}{C_1 (K+1) } \quad \textrm{with} \quad \mcal{R}_0^2 := \frac{2}{\eta^2} \norms{x^0 - x^{\star}}^2 + \ncmt{\frac{11}{4}}\norms{Gx^0+\xi^0}^2.
\end{array}
\end{equation*}
To achieve $\Exp{\norms{G\bar{x}^K + \bar{\xi}^K}^2} \leq \epsilon^2$, we require $K := \frac{\ncmt{17}\mcal{R}_0^2}{C_1\epsilon^2} = \BigOs{\frac{\mcal{R}_0^2}{\epsilon^2}}$.
If we choose $\omega$ sufficiently small such that $\omega(2-\omega) \leq \frac{\mbf{p}}{4}$, then we have $\kappa = \omega(2-\omega)$.
From Lemma~\ref{le:hsvrg_estimator}, since $\tau = \omega$, $\kappa = \omega(2-\omega)$, \ncmt{$\Theta = 16L^2\left[\frac{(1-\omega)^2}{b} + \frac{4\omega^2}{\hat{b}\mbf{p}^2}\right]$}, and \ncmt{$\hat{\Theta} = 4L^2\left[\frac{(1-\omega)^2}{b} + \frac{4\omega^2}{\hat{b}\mbf{p}^2}\right]$}, the condition $4L^2\tau^2\kappa \geq \ncmt{22}(\Theta + \hat{\Theta})$ in \eqref{eq:A2_iFRBS_cond1}
becomes 
$4L^2\omega^3(2-\omega) \geq \ncmt{440}L^2 \left[\frac{(1-\omega)^2}{b} + \frac{\ncmt{4}\omega^2}{\hat{b}\mbf{p}^2}\right]$, which holds if
\begin{equation*} 
\arraycolsep=0.2em
\begin{array}{lcl}
\frac{(1-\omega)^2}{b} \leq \frac{\omega^3(2-\omega)}{\ncmt{220}} 
\qquad \text{and} \qquad 
\frac{\ncmt{4}\omega^2}{\hat{b}\mbf{p}^2} \leq \frac{\omega^3(2-\omega)}{\ncmt{220}}.
\end{array}
\end{equation*}
Since $\omega \in \ncmt{(0,1]}$, the first condition holds if $\frac{1}{b} \leq \frac{\omega^3}{\ncmt{220}}$.
Thus, we can choose $b := \BigOs{\frac{1}{\omega^3}} \geq \frac{\ncmt{220}}{\omega^3}$.
Similarly, the second condition holds if $\frac{\ncmt{4}}{\hat{b}\mbf{p}^2} \leq \frac{\omega}{\ncmt{220}}$.
In this case, we can choose $\hat{b} := \BigOs{\frac{1}{\omega \mbf{p}^2}} \geq \frac{\ncmt{880}}{\omega \mbf{p}^2}$.
Therefore, the expected total number of $G_i$ evaluations is at most
\begin{equation*}
\arraycolsep=0.2em
\begin{array}{lcl}
\Exp{\Tc_G} &= & K(2b + 2\hat{b} + n\mbf{p}) = \BigO{\frac{\mcal{R}_0^2}{\epsilon^2}\big( \frac{1}{\omega^3} + \frac{1}{\omega\mbf{p}^2} + n\mbf{p} \big)}.
\end{array}
\end{equation*}
In particular, if we choose $\omega := \BigOs{\mbf{p}} \leq \frac{\mbf{p}}{8}$, then $\Exp{\Tc_G} = \BigOs{\frac{\mcal{R}_0^2}{\epsilon^2}\big( \frac{1}{\mbf{p}^3} + n\mbf{p} \big)}$.
Moreover, if $\mbf{p} = \BigOs{n^{-1/4}}$, then we get $\Exp{\Tc_G} = \BigOs{\frac{\mcal{R}_0^2 n^{3/4}}{\epsilon^2}}$.
In this case, we obtain $b = \BigOs{n^{3/4}}$ and $\hat{b} = \BigOs{n^{3/4}}$.
The expected total number of $J_{\eta T}$ evaluations is $\Exp{\mcal{T}_J} = K = \BigOs{\epsilon^{-2}}$.

{(b)~For the expectation setting (\textsf{E}), we have $S_K := \sum_{k=0}^K\delta_k = \sum_{k=0}^K\frac{2\omega^2\sigma^2}{n} = \frac{2\omega^2(K+1)\sigma^2}{n}$.
If we choose $\omega$ sufficiently small such that $\omega(2-\omega) \leq \frac{\mbf{p}}{4}$, then we have $\kappa = \omega(2-\omega)$.
From Lemma~\ref{le:hsvrg_estimator}, since $\tau = \omega$, $\kappa = \omega(2-\omega)$, \ncmt{$\Theta = 16L^2\left[\frac{(1-\omega)^2}{b} + \frac{8\omega^2}{\hat{b}\mbf{p}^2}\right]$}, and \ncmt{$\hat{\Theta} = 4L^2\left[\frac{(1-\omega)^2}{b} + \frac{8\omega^2}{\hat{b}\mbf{p}^2}\right]$}, we can still choose $b := \BigOs{\frac{1}{\omega^3}}$ and $\hat{b} := \BigOs{\frac{1}{\omega \mbf{p}^2}}$ to satisfy the condition $4L^2\tau^2\kappa \geq \ncmt{22}(\Theta + \hat{\Theta})$ in \eqref{eq:A2_iFRBS_cond1}.

Now, to achieve $\Exp{\norms{G\bar{x}^K + \bar{\xi}^K}^2} \leq \epsilon^2$, from \eqref{eq:A2_key_convergence_bound1}, we require 
\begin{equation*}
	\arraycolsep=0.2em
	\begin{array}{lcl}
		\frac{\ncmt{17}\mcal{R}_0^2}{C_1 (K+1) } +  \frac{\ncmt{356}}{C_1\tau^2\kappa (K+1)}S_K \leq  \frac{\ncmt{17}\mcal{R}_0^2}{C_1 (K+1) } +  \frac{\ncmt{712}\sigma^2}{C_1\omega n}  \leq \epsilon^2.
	\end{array}
\end{equation*}
This condition holds if both $\frac{\ncmt{34}\mcal{R}_0^2}{C_1 (K+1) } \leq \epsilon^2$ and $ \frac{\ncmt{1424}\sigma^2}{C_1\omega n}  \leq \epsilon^2$ are simultaneously satisfied.
These conditions hold if we choose $K = \BigOs{\frac{\mcal{R}^2_0}{\epsilon^2}}$ and $\omega = \BigOs{\frac{1}{\epsilon^2 n}}$.
As a consequence,  the expected total number of $\mbf{G}(\cdot,\xi)$ calls is at most
\begin{equation*}
\arraycolsep=0.2em
\begin{array}{lcl}
	\Exp{\Tc_G} &= & K(2b + 2\hat{b} + n\mbf{p}) = \BigO{\frac{\mcal{R}_0^2}{\epsilon^2}\big( \frac{1}{\omega^3} + \frac{1}{\omega\mbf{p}^2} + n\mbf{p} \big)}.
\end{array}
\end{equation*}
In particular, if we choose $\omega := \BigOs{\mbf{p}} \leq \frac{\mbf{p}}{8}$, then we also have $\mbf{p} = \BigOs{\frac{1}{\epsilon^2 n}}$.
This choice leads to $\Exp{\Tc_G} = \BigOs{\frac{\mcal{R}_0^2}{\epsilon^2}\big( \frac{1}{\mbf{p}^3} + n\mbf{p} \big)} = \BigOs{\frac{\mcal{R}_0^2}{\epsilon^2}\big( \epsilon^6 n^3 + \frac{1}{\epsilon^3} \big)}$.
If we choose $n = \BigOs{\epsilon^{-3}}$, then this complexity becomes $\Exp{\Tc_G} = \BigOs{\epsilon^{-5}}$.
In this case, we have $\mbf{p} = \BigOs{\epsilon}$, $\mbf{\omega} = \BigOs{\epsilon}$, $b = \BigOs{\epsilon^{-3}}$, and $\hat{b} = \BigOs{\epsilon^{-3}}$.
The expected total number of $J_{\eta T}$ evaluations is $\Exp{\mcal{T}_J} = K = \BigOs{\epsilon^{-2}}$.}
\Eproof
\end{proof}

For the finite-sum setting \textsf{(F)}, Corollaries~\ref{co:sarah_compexity} and \ref{co:hsvrg_compexity} both yield an oracle complexity of $\BigOs{n^{3/4}\epsilon^{-2}}$. 
This improves upon deterministic FRBS by a factor of $n^{1/4}$. 
However, it is worse than the $\BigOs{n^{2/3}\epsilon^{-2}}$ complexity obtained by the unbiased SVRG and SAGA variants in Subsection~\ref{subsec:oracle_complexity1} by a factor $n^{\frac{1}{12}}$. 
We believe this gap is likely an artifact of our analysis rather than a limitation of the estimators themselves. 
This behavior also contrasts with nonconvex optimization, where biased variance-reduced estimators such as SARAH and Hybrid SGD can achieve optimal oracle complexities that are better than those of SVRG and SAGA. 
It therefore remains an open question whether biased variance-reduced estimators in \eqref{eq:iFRBS_scheme} can achieve a better oracle complexity than unbiased ones when solving \eqref{eq:GE}.

For the expectation setting \textsf{(E)}, Corollaries~\ref{co:sarah_compexity}, \ref{co:hsgd_compexity}, \ncmt{and \ref{co:hsvrg_compexity}} give an oracle complexity of $\BigOs{\epsilon^{-5}}$, which is worse than $\BigOs{\epsilon^{-10/3}}$ for SVRG and $\BigOs{\epsilon^{-4}\ln(\epsilon^{-1})}$ for the increasing mini-batch SGD variant. 
Again, we conjecture that this gap is an artifact of our analysis.

\beforesec
\section{Numerical Experiments}\label{sec:numerical_experiments}
\aftersec
In this section, we provide two numerical examples to evaluate our methods and compare them with recent algorithms from the literature.
All algorithms are implemented in Python and run on the Longleaf HPC cluster using NVIDIA A100 (40 GB) and L40 (48 GB) GPUs.

\beforesubsec
\subsection{\textbf{AUC optimization for classification with imbalanced data}}\label{subsec:AUC_optimization}
\aftersubsec
$\textrm{(a)}$~\textbf{\textit{Mathematical model.}}
Many real-world machine learning tasks such as bot detection, financial fraud identification, and medical diagnosis frequently require solving  classification problems with highly imbalanced data. 
In these scenarios, training a classifier based on accuracy is often ineffective (i.e., the model only needs to assign the negative label to all instances to achieve near perfect accuracy), necessitating the use of alternative objectives.

The Area Under the ROC Curve (AUC) is a robust metric for such cases, quantifying the probability that a classifier ranks a randomly chosen positive instance higher than a randomly chosen negative instance. 
However, directly optimizing AUC is computationally expensive because its pairwise nature requires comparing $O(n^2)$ pairs for $n$ data instances, which is prohibitive for large-scale datasets. 
To overcome this challenge, \cite{ying2016stochastic} introduced a reformulation based on the square loss, transforming the problem into the following minimax optimization model:
\begin{equation}\label{eq:AUC_minimax}
\min\limits_{(\mbf{w}, a, b) \in \Omega_1} \max\limits_{\alpha \in \Omega_2} \ \int_{\Zc} \Lc(\mathbf{w}, a, b, \alpha; \mbf{z}) d\rho(\mbf{z}),
\end{equation}
where $\mbf{z} := (\mbf{x}, y) \in \Zc$  represents  i.i.d. samples drawn from an unknown distribution $\rho$ on $\Zc = \Xc \times \Yc$, $\kappa := \sup_{\mbf{x} \in \Xc} \norms{\mbf{x}} < +\infty$, $\Omega_1 := \sets{(\mbf{w}, a, b) \in \R^{d+2} : \norms{\mbf{w}} \leq R, \ \vert a \vert \leq R\kappa, \ \vert b \vert \leq R\kappa}$, $\Omega_2 := \sets{\alpha \in \R : \vert \alpha \vert \leq 2R\kappa}$, and the function $\Lc$ is defined for a data point $\mbf{z} = (\mbf{x},y)$ as follows:
\begin{equation*}
\arraycolsep=0.2em
\begin{array}{lcl}
\Lc(\mathbf{w}, a, b, \alpha; \mbf{z}) &:=& (1 - \mbf{q} )(\mathbf{w}^\top \mathbf{x} - a)^2 \mathbb{I}_{[y=1]} + \mbf{q}(\mathbf{w}^\top \mathbf{x} - b)^2 \mathbb{I}_{[y=-1]} \vspace{1ex}\\
&& + {~} 2(1+\alpha)\left( \mbf{q} \mathbf{w}^\top \mathbf{x} \mathbb{I}_{[y=-1]} - (1 - \mbf{q}) \mathbf{w}^\top \mathbf{x} \mathbb{I}_{[y=1]} \right) - \mbf{q}(1 - \mbf{q})\alpha^2.
\end{array}
\end{equation*}
Here, $\mathbf{w}$ represents the model parameters, $a$ and $b$ are the estimated mean scores for the positive and negative classes, respectively, $\alpha$ is the dual variable, and $\mbf{q}$ is the prior probability of the positive class.

If we assume that $\rho$ is a uniform distribution, then \eqref{eq:AUC_minimax} can equivalently be written as the following finite-sum minimax optimization problem:
\begin{equation}\label{eq:AUC_finitesum_minimax}
\arraycolsep=0.2em
\begin{array}{lcl}
\min\limits_{\mbf{w},a,b \in \R^{d+2}} \max\limits_{\alpha \in \R} \ \Big\{f(\mbf{w},a,b) + \frac{1}{n} \sum_{i=1}^n \Lc_i(\mathbf{w}, a, b, \alpha; \mbf{z}_i) - g(\alpha)\Big\},
\end{array}
\end{equation}
where $\Lc_i(\mathbf{w}, a, b, \alpha) := \Lc(\mathbf{w}, a, b, \alpha; \mbf{z}_i)$, and $f(\mbf{w}, a, b) := \delta_{\Omega_1}(\mbf{w}, a, b)$ and $g(\alpha) := \delta_{\Omega_2}(\alpha)$ are the indicators of $\Omega_1$ and $\Omega_2$, respectively.
Let us define
\begin{equation*}
x := \begin{bmatrix}
\mbf{w}\\ a\\ b\\ \alpha
\end{bmatrix}, \quad 
G_ix := \begin{bmatrix}
\nabla_{\mbf{w}}\Lc_i(\mathbf{w}, a, b, \alpha) \\
\nabla_{a} \Lc_i(\mathbf{w}, a, b, \alpha) \\
\nabla_{b} \Lc_i(\mathbf{w}, a, b, \alpha) \\
-\nabla_{\alpha} \Lc_i(\mathbf{w}, a, b, \alpha)
\end{bmatrix}, \quad \text{and} \quad
Tx := \begin{bmatrix}
\partial_{\mbf{w}} f(\mbf{w},a,b) \\
\partial_{a} f(\mbf{w},a,b) \\
\partial_{b} f(\mbf{w},a,b) \\
\partial_{\alpha} g(\alpha)
\end{bmatrix},
\end{equation*}
Then, the optimality condition of problem \eqref{eq:AUC_finitesum_minimax} can be written as $0 \in Gx + Tx$, where $Gx = \frac{1}{n}\sum_{i=1}^n G_ix$.
This is exactly of the finite-sum form \eqref{eq:finite_sum_form} in the setting \textsf{(F)}.

Moreover, if we denote by $I_{+} := \sets{i \in [n] : y_i = 1}$ and $I_{-} := \sets{ i \in [n] : y_i = -1}$ the index sets of positive and negative classes, respectively, by $n_{+} := \vert I_{+} \vert$ and $n_{-} := \vert I_{-}\vert$ their corresponding cardinalities, by $\boldsymbol{\mu}_{+} := \frac{1}{n_{+}}\sum_{i \in I_{+}} \mbf{x}_i$ and $\boldsymbol{\mu}_{-} := \frac{1}{n_{-}}\sum_{i \in I_{-}} \mbf{x}_i$ their corresponding class means, and by $\mbf{S}_{+} := \frac{1}{n_{+}}\sum_{i \in I_{+}} \mbf{x}_i \mbf{x}_i^\top$ and $\mbf{S}_{-} := \frac{1}{n_{-}}\sum_{i \in I_{-}} \mbf{x}_i \mbf{x}_i^\top$ their corresponding class second-moment matrices, then we can express  $G$ as an affine mapping $Gx = \mbf{Q}x + \mbf{r}$, where the matrix $\mbf{Q}$ is given by
$$\mbf{Q} := 2 \mbf{q}(1 - \mbf{q})\begin{bmatrix}
\mbf{S}_{+} + \mbf{S}_{-} & -\boldsymbol{\mu}_{+} & -\boldsymbol{\mu}_{-} & \boldsymbol{\mu}_{-} - \boldsymbol{\mu}_{+} \\
-\boldsymbol{\mu}_{+}^\top & 1 & 0 & 0 \\
-\boldsymbol{\mu}_{-}^\top & 0 & 1 & 0 \\
-(\boldsymbol{\mu}_{-} - \boldsymbol{\mu}_{+})^\top & 0 & 0 & 1
\end{bmatrix},$$
which completely depends on the dataset. 
Thus, $G$ is $L$-Lipschitz continuous with $L := \norms{\mbf{Q}}_2$.

\vspace{0.5ex}
\noindent$\textrm{(b)}$~\textbf{\textit{Data generation and experiment setup.}}
We generate a synthetic imbalanced dataset consisting of $n$ samples and $d$ features. 
The feature matrix $\mbf{X} \in \mathbb{R}^{n \times d}$ is sampled from a standard normal distribution $\mathcal{N}(0, 1)$, and a ground truth weight vector $w^* \in \mathbb{R}^d$ is generated from the same distribution and normalized to unit norm. 
The noisy scores are then computed as $s = \mbf{X}w^* + \epsilon$, where $\epsilon \sim \mathcal{N}(0, \sigma^2)$ represents additive noise. 
To strictly enforce a target positive class ratio $\mbf{q}$, binary labels $y \in \{1, -1\}^n$ are assigned by thresholding the scores at the $(1 - \mbf{q})$-th percentile. 

We consider two experiments: \textit{Experiment 1} with $n = 50,000$ and $d = 250$, leading to $p = 253$, and \textit{Experiment 2} with $n = 100,000$ and $d = 500$, leading to $p = 503$.
Each experiment is run five times, and then we report the average relative forward-backward splitting residual norm $\frac{\norms{\Fc_{\eta}x^k}}{\norms{\Fc_{\eta}x^0}}$ over $1000$ epochs, where one epoch corresponds to $n$ component-operator evaluations of $G_i$.

\vspace{0.5ex}
\noindent$\textrm{(c)}$~\textbf{\textit{Algorithms and parameters.}}
We consider $6$ variants of our \ref{eq:iFRBS_scheme} method with different variance-reduced estimators: $3$ unbiased ones --- \ref{eq:svrg_estimator} (\texttt{VrFRBS-SVRG}), \ref{eq:saga_estimator} (\texttt{VrFRBS-SAGA}), and increasing mini-batch SGD (\texttt{VrFRBS-SGD (imb)}), and $3$ biased instances --- \ref{eq:sarah_estimator} (\texttt{VrFRBS-SARAH}), \ref{eq:hsgd_estimator} (\texttt{VrFRBS-HSGD}), and  \ref{eq:hsvrg_estimator} (\texttt{VrFRBS-HSVRG}).
The two competitors   are the variance-reduced forward-reflected-backward splitting scheme (\texttt{VFRBS}) in \cite{alacaoglu2021forward} and the variance-reduced extragradient method (\texttt{VEG}) in \cite{alacaoglu2021stochastic}.
Both \texttt{VFRBS} and \texttt{VEG} are variants of EG and use loopless SVRG estimators.

After tuning each method, the parameters used in these algorithms are chosen as follows:
\begin{compactitem}[$\diamond$]
\item For \texttt{VrFRBS-SVRG}, we choose $\eta = \frac{1}{5L}$, a mini-batch size $b = \lfloor 0.5n^{2/3} \rfloor$, and $\mbf{p} = n^{-1/3}$.
\item For \texttt{VrFRBS-SAGA}, we choose $\eta = \frac{1}{14L}$, a mini-batch size $b = \lfloor 0.5n^{2/3} \rfloor$.
\item For \texttt{VrFRBS-SGD (imb)}, we choose $\eta = \frac{1}{2L}$ (\textit{Experiment 1}) and $\eta = \frac{1}{4.5L}$ (\textit{Experiment 2}), and an increasing mini-batch size $b_k = \min\sets{n, \max\sets{1, \lfloor 0.01n(k+1)^{3/4}\rfloor}}$.
\item For \texttt{VrFRBS-SARAH}, we choose $\eta = \frac{1}{3.5L}$, a mini-batch size $b = \lfloor 0.25n^{3/4} \rfloor$, and $\mbf{p} = n^{-1/4}$.
\item For \texttt{VrFRBS-HSGD}, we choose $\eta = \frac{1}{1.5L}$, a mini-batch size $b = \lfloor 0.25n^{3/4} \rfloor$, $\omega = 0.5$.
\item For \texttt{VrFRBS-HSVRG}, we choose $\eta = \frac{1}{5.5L}$, a mini-batch size $b = \lfloor 0.25n^{3/4} \rfloor$, $\mbf{p} = n^{-1/3}$, and $\omega = 0.5$.
\item For \texttt{VFRBS}, we set $\eta = 7\eta_{\text{def}}$ and a mini-batch size $b = \lfloor 0.5n^{2/3}\rfloor$, and $\mbf{p} = n^{-1/3}$, where $\eta_{\text{def}} := \frac{0.95(1-\sqrt{1 - \mbf{p}})}{2L}$ is its theoretical default stepsize.
\item For \texttt{VEG}, we choose $\eta = 6\eta_{\text{def}}$ and a mini-batch size $b = \lfloor 0.5n^{2/3}\rfloor$, and $\mbf{p} = n^{-1/3}$, where $\eta_{\text{def}} := \frac{0.95\sqrt{1-\alpha}}{L}$ is its theoretical default stepsize with $\alpha := 1 - \mbf{p}$.
\end{compactitem}
Note that the stepsize $\eta$ above for each algorithm is manually tuned to achieve the best performance. 
Consequently, some values exceed the theoretical thresholds established in the  theorems, which are sufficient but not necessarily optimal. 
Similarly, the stepsizes for the \texttt{VFRBS} and \texttt{VEG} baselines are larger than those originally proposed, as the theoretical values did not yield the best results in our experiments.

\vspace{0.5ex}
\noindent$\textrm{(d)}$~\textbf{\textit{Numerical results.}}
The numerical results are presented in Figure~\ref{fig:AUC_results}, where the vertical axis shows the relative residual norm $\frac{\norms{\Fc_{\eta}x^k}}{\norms{\Fc_{\eta}x^0}}$, while the horizontal axis shows the number of epochs.

\begin{figure}[hpt!]
\vspace{-0ex}
\centering
\includegraphics[width=\textwidth]{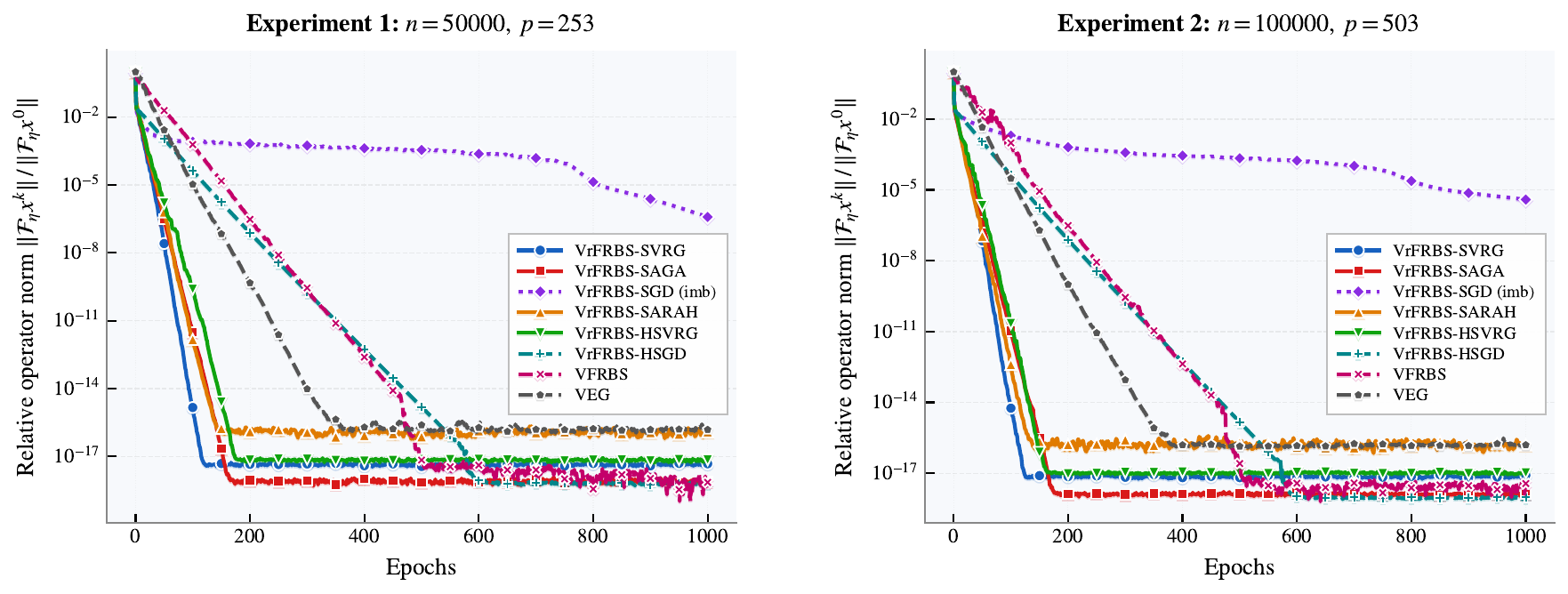}
\vspace{-4ex}
\caption{The performance of the 6 variants of \eqref{eq:iFRBS_scheme} and their competitors on the AUC optimization problem \eqref{eq:AUC_finitesum_minimax}, averaged over 5 problem instances}
\label{fig:AUC_results}
\vspace{-2ex}
\end{figure}

As illustrated in Figure~\ref{fig:AUC_results}, \texttt{VrFRBS-SVRG} emerges as the most efficient algorithm in this experiment. 
It converges slightly faster than the three variants of \eqref{eq:iFRBS_scheme} utilizing \eqref{eq:saga_estimator}, \eqref{eq:sarah_estimator}, and \eqref{eq:hsvrg_estimator} estimators, while significantly outperforming the remaining methods.
Interestingly, unlike in stochastic optimization, so far, we observe that biased estimators such as SARAH or HSVRG do not surpass unbiased ones in this more general minimax setting. 
This observation is consistent with our theoretical oracle-complexity results.
The two baseline competitors, \texttt{VFRBS} and \texttt{VEG}, do not converge as rapidly as our top-performing variants, though they remain faster than or at least comparable to the \texttt{VrFRBS-HSGD} variant.
Comparing the two hybrid-type biased estimators, the variant using \eqref{eq:hsvrg_estimator} strongly outperforms the one using \eqref{eq:hsgd_estimator}. 
This superiority can be attributed to the variance-reduction property inherent in the unbiased component of HSVRG.
Finally, \texttt{VrFRBS-SGD (imb)} exhibits the least competitive performance, barely reaching a relative operator norm of $10^{-6}$ after 1000 epochs, which reflects the poor oracle complexity derived in our theoretical analysis.

We also observe highly consistent results between \textit{Experiments 1} and \textit{Experiment 2}, even though both the sample size and dimension are doubled, illustrating the stability and scalability of our methods with respect to problem size.

\beforesubsec
\subsection{\textbf{Policy evaluation with linear approximation}}\label{subsec:policy_evaluation}
\aftersubsec
$\textrm{(a)}$~\textbf{\textit{Mathematical model.}}
Policy evaluation, which estimates the value function of a Markov Decision Process (MDP) under a given policy, is a crucial step in many reinforcement learning algorithms.
Consider an infinite-horizon discounted-reward problem in which an MDP is characterized by $(\Sc, \Ac, \Pc_{s,s'}^a, \Rc, \gamma)$, where $\Sc$ is the state space, $\Ac$ is the action space, $\Pc_{s,s'}^a$ denotes the transition probabilities, $\Rc(s,a)$ is the immediate reward received after taking action $a$ at state $s$, and $\gamma$ is the discount factor.
The value function associated with a given policy $\pi$ is defined by 
\begin{equation*}
\begin{array}{l}
V^\pi := [V^{\pi}(1), \cdots, V^{\pi}(|S|)]^{\top}, \quad \textrm{where} \quad V^{\pi}(s) := \Expk{\sum_{t=0}^\infty \gamma^t R(s_t, a_t) \mid s_0 = s, \pi}.
\end{array}
\end{equation*}
Then, it is well-known that $V^\pi$ is the unique fixed point of the Bellman operator $T^\pi$ defined by $T^\pi(u) := R^\pi + \gamma P^\pi u$ for $u \in \R^{|\Sc|}$, where $R^\pi := [R^\pi(1), \cdots, R^\pi(|\Sc|)]^\top$ is the expected reward vector under policy $\pi$ with $R^\pi(s) := \Expsk{\pi(a|s)}{\Rc(s,a)}$ and $P^\pi$ is the transition matrix induced by the policy $\pi$, defined elementwise as $P^{\pi}(s,s') := \Expsk{\pi(a|s)}{\Pc^a_{ss'}}$.

In many scenarios, the state space $\Sc$ is large or even infinite, and consequently, evaluating the exact value function is extremely expensive or even infeasible.
One approach is to approximate $V^{\pi}$ by a linear function $\hat{V}^\pi(s) := \phi(s)^\top \theta$, where $\phi: \Sc \to \R^d$ is a given feature map and $\theta \in \R^d$ is the model parameter that we need to learn.
This reduces the problem dimension from $|\Sc|$ (very large or infinite) to $d \ll |\Sc|$.
Since $\hat{V}^\pi(s)$ approximates $V^\pi$, the unique fixed-point of $T^\pi$, we can find $\theta$ by minimizing the [Empirical] Mean Squared Projected Bellman Error (EM-MSPBE) as follows:
\begin{equation}\label{eq:PE_min}
\arraycolsep=0.2em
\begin{array}{lcl}
\min\limits_{\theta \in \mathbb{R}^d}\Big\{ \frac{1}{2}  \norms{\hat{A}\theta - \hat{b}} _{\hat{C}^{-1}}^2 + \revise{\nu} f(\theta) \Big\}.
\end{array}
\end{equation}
Here, given a dataset of $n$ transitions $\Dc := \sets{(s_t, a_t, r_t, s_{t+1})}_{t=1}^n$, denote $\phi_t := \phi(s_t)$, $\phi'_t := \phi(s_{t+1})$, and let $\hat{A} := \frac{1}{n}\sum_{t=1}^n A_t$, $\hat{b} := \frac{1}{n}\sum_{t=1}^n b_t$, and $\hat{C} := \frac{1}{n}\sum_{t=1}^n C_t$ with $A_t := \phi_t(\phi_t - \gamma\phi_t')^\top$, $b_t := r_t \phi_t$, and $C_t := \phi_t \phi_t^\top$.
Moreover, $f(\theta)$ is a regularizer added to the objective function and $\revise{\nu} > 0$ is a regularization parameter.
\revise{Following \cite{du2017stochastic}, we assume that $\hat{C}$ is positive definite and $\hat{A}$ has full rank.}

Due to the weighted norm $\norms{\cdot}_{\hat{C}^{-1}}$, the current objective function of \eqref{eq:PE_min} does not admit a finite-sum structure.
We follow \cite{du2017stochastic} to reformulate \eqref{eq:PE_min} into the following finite-sum minimax problem:
\begin{equation}\label{eq:PE_minimax}
\arraycolsep=0.2em
\begin{array}{lcl}
\min\limits_{\theta \in \mathbb{R}^d} \max\limits_{w \in \mathbb{R}^d} \Big\{  \revise{\nu} f(\theta) + \frac{1}{n} \sum_{t=1}^n \mathcal{L}_t(\theta, w) \Big\},
\end{array}
\end{equation}
where $w \in \R^d$ is the dual variable acting as a critic, and $\mathcal{L}_t(\theta, w) = - w^\top A_t \theta - \frac{1}{2} \| w \|_{C_t}^2 + w^\top b_t$.
\revise{
We choose $f(\theta) := \sum_{j=1}^d p_{\lambda, a}(\theta_j)$ as the SCAD (Smoothly Clipped Absolute Deviation) regularizer \cite{fan2001variable}, where $p_{\lambda, a}$ denotes the SCAD penalty with parameters $\lambda > 0$ and $a > 2$.}
If we define $x := [\theta, w] \in \R^{2d}$, $G_t x := \left[ \nabla_{\theta} \Lc_t(\theta, w), -\nabla_{w} \Lc_t(\theta, w) \right] = [-A_t^\top w, \ A_t \theta + C_t w - b_t]$, $Gx = \frac{1}{n}\sum_{t=1}^n G_t x$, and $Tx := [\revise{\nu}\partial f(\theta), 0]$, then \eqref{eq:PE_minimax} reduces to $0 \in Gx + Tx$, a special case of \eqref{eq:GE}.

\revise{
The following proposition shows that, under appropriate conditions of $A$ and $\hat{C}$, $G+T$ is $\rho$-co-hypomonotone, and thus every $x^{\star} \in \zer{G+T}$ is a $\rho$-weak-Minty solution.
The proof of this proposition can be found in Appendix~\ref{apdx:sec:numerical_experiments}.

\begin{proposition}\label{prop:PE_cohypomonotone}
	Suppose that $C \succ 0$.
	Let $\mu := \lambda_{\min}(\hat{A}^\top \hat{C}^{-1}\hat{A})$ and $c := \lambda_{\min}(\hat{C})$.
	If  $\nu < (a-1)\mu$, then $\Phi := G+T$ in the \eqref{eq:GE} reformulation of \eqref{eq:PE_minimax} is $\rho$-co-hypomonotone with 
	$\rho := \frac{\nu}{(a-1)c\big(\sqrt{\mu} - \sqrt{\frac{\nu}{a-1}}\big)^2}$.
\end{proposition}
}

\vspace{0.5ex}
\noindent$\textrm{(b)}$~\textbf{\textit{Experiment setup and data.}}
We evaluate the performance of the $6$ different variants of \eqref{eq:iFRBS_scheme} through two experiments: \textit{Experiment 1} and \textit{Experiment 2}.
\begin{compactitem}[$\bullet$]
\item In \textit{Experiment 1}, we consider a randomly generated MDP with $1{,}000$ states and $20$ actions.
The transition probabilities are chosen as $P(s'|s,a) \propto p^a_{ss'} + 10^{-5}$, where $p^a_{ss'} \sim U[0,1]$.
The action-selection policy and initial distribution were generated in a similar way.
The dataset consists of $n=20{,}000$ transitions.
Each state is represented by a $201$-dimensional feature vector, where $200$ of the features were sampled from a uniform distribution, and the last feature was set to the constant value $1$ to handle the bias term.
The discount factor is chosen as $\gamma = 0.95$.

\item In \textit{Experiment 2}, we consider the Mountain Car problem \cite[Example 9.2]{sutton2018reinforcement}.
Instead of a random policy, we first use the SARSA algorithm with $d=300$ CMAC features to obtain a good policy.
Then, we run this policy for $n = 20{,}000$ transitions to obtain the dataset.
Here, we choose the discount factor to be $\gamma = 0.99$.
\end{compactitem}

\revise{
For each problem instance, we first check the raw empirical matrices, and found that they satisfy the assumptions required by Proposition~\ref{prop:PE_cohypomonotone} in \textit{Experiment 1}.
In \textit{Experiment 2}, some CMAC features are inactive or redundant along the finite trajectory, resulting in rank-deficient empirical matrices.
This issue can be addressed, for example, by removing redundant features, or applying an appropriate dimension-reduction procedure.
However, to preserve the original feature dimension and experimental setup, we apply, only when needed, a small diagonal stabilization with relative magnitude $10^{-3}$.
Otherwise, the raw matrices are left unchanged.

For the SCAD regularizer, we choose $\lambda = 0.1$, $a = 3.7$, and $\nu = 0.99\min_j \nu_{\max, j}$, where $\nu_{\max, j} = (a-1)\mu_j \big(\frac{q_j}{1+q_j}\big)^2$ with $q_j = \sqrt{\frac{c_j}{32\sqrt{134}L_j}}$, $c_j = \lambda_{\min}(\hat{C_j})$, and $\mu_j = \lambda_{\min}(\hat{A}_j^\top \hat{C}_j^{-1}\hat{A}_j)$ are obtained from the data of the problem instance $j$.
This choice ensures that $L_j\rho_j \leq \frac{1}{32\sqrt{134}}$ and $\nu < (a-1)\mu_j$ for every problem instance $j$.
}
We run each experiment five times and report the average relative norm $\frac{\norms{\Fc_{\eta}x^k}}{\norms{\Fc_{\eta}x^0}}$ over $5000$ epochs, where one epoch is defined as $n$ evaluations of $G_i$.

\vspace{0.5ex}
\noindent$\textrm{(c)}$~\textbf{\textit{Algorithms and parameters.}}
Similar to Subsection~\ref{subsec:AUC_optimization}, we still consider 6 variants of our \ref{eq:iFRBS_scheme} algorithm and the two competitors \texttt{VFRBS} and \texttt{VEG}.
The parameters are chosen as follows:
First, since $Gx = \frac{1}{n}\sum_{t=1}^n \mbf{G}_tx + \mbf{g}$ is affine, it is $L$-average Lipschitz continuous with \revise{$L := \left[\lambda_{\max}\big(\frac{1}{n}\sum_{t=1}^n \mbf{G}_t^\top \mbf{G}_t\big)\right]^{1/2}$}.
Next, having $L$, we set other parameters as follows.
\begin{compactitem}[$\diamond$]
\item For \texttt{VrFRBS-SVRG}, we choose $\eta = \revise{\frac{0.99}{3\sqrt{2}L}}$, a mini-batch size $b = \lfloor 0.5n^{2/3} \rfloor$, and $\mbf{p} = n^{-1/3}$.
\item For \texttt{VrFRBS-SAGA}, we choose $\eta = \revise{\frac{0.99}{3\sqrt{2}L}}$, a mini-batch size $b = \lfloor 0.5n^{2/3} \rfloor$.
\item For \texttt{VrFRBS-SGD(imb)}, we set $\eta = \revise{\frac{0.99}{3\sqrt{2}L}}$, a mini-batch size $b_k = \min\sets{n, \max\sets{1, \lfloor 0.025n(k+1)^{3/4}\rfloor}}$.
\item For \texttt{VrFRBS-SARAH}, we choose $\eta = \revise{\frac{0.99}{\sqrt{134}L}}$, a mini-batch size $b = \lfloor 0.25n^{3/4} \rfloor$, and $\mbf{p} = n^{-1/4}$.
\item For \texttt{VrFRBS-HSVRG}, we choose $\eta = \revise{\frac{0.99}{\sqrt{134}L}}$, a mini-batch size $b = \lfloor 0.25n^{3/4} \rfloor$, $\mbf{p} = n^{-1/3}$, and $\omega = 0.5$.
\item For \texttt{VrFRBS-HSGD}, we choose $\eta = \revise{\frac{0.99}{\sqrt{134}L}}$, a mini-batch size $b = \lfloor 0.25n^{3/4} \rfloor$, and $\omega = 0.5$.
\item For \texttt{VFRBS}, we use its theoretical default stepsize $\revise{\eta = \eta_{\text{def}}} := \frac{0.95(1-\sqrt{1 - \mbf{p}})}{2L}$, a mini-batch size $b = \lfloor 0.5n^{2/3}\rfloor$, and $\mbf{p} = n^{-1/3}$.
\item For \texttt{VEG}, we use its theoretical default stepsize $\revise{\eta = \eta_{\text{def}}} := \frac{0.95\sqrt{1-\alpha}}{L}$ with $\alpha := 1 - \mbf{p}$, a mini-batch size $b = \lfloor 0.5n^{2/3}\rfloor$, and $\mbf{p} = n^{-1/3}$.
\end{compactitem}

\vspace{0.5ex}
\noindent$\textrm{(d)}$~\textbf{\textit{Numerical results.}}
The numerical results for this example are presented in Figure~\ref{fig:PE_results}.
\begin{figure}[hpt!]
\vspace{-2ex}
\centering
\includegraphics[width=\textwidth]{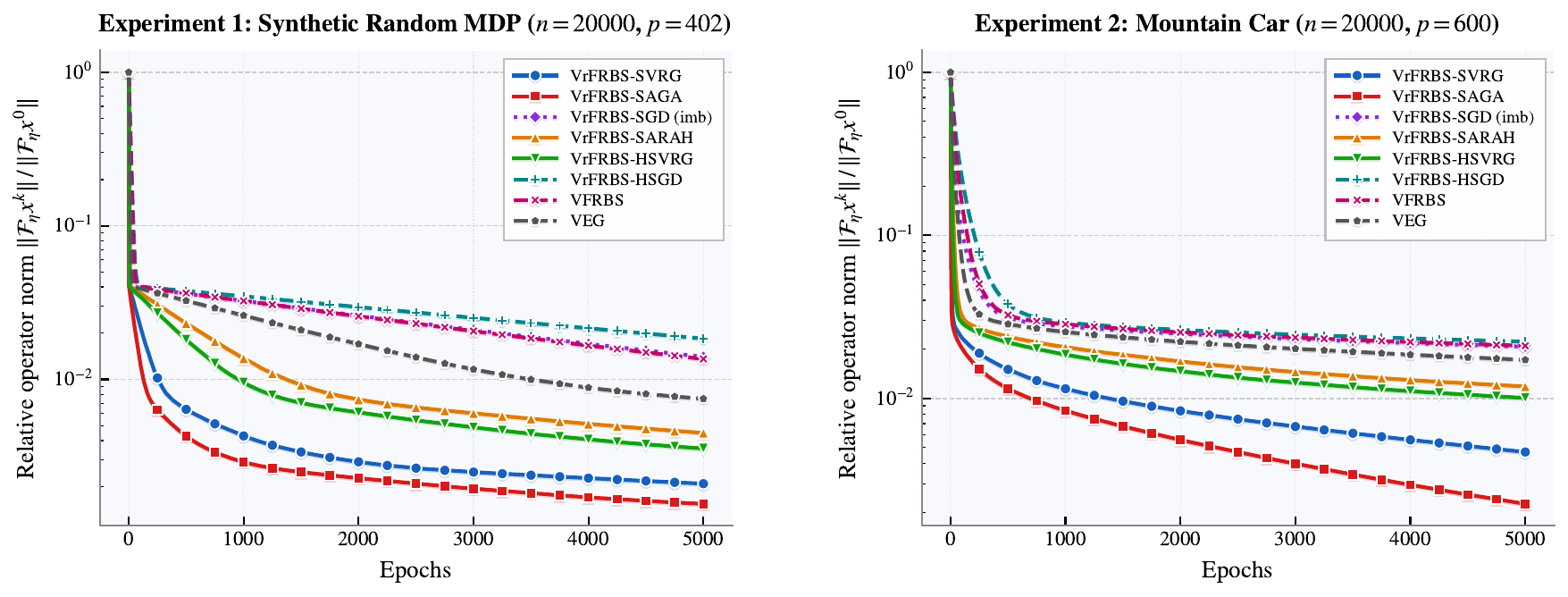}
\vspace{-4ex}
\caption{The performance of the 6 variants of \eqref{eq:iFRBS_scheme} and their competitors on the policy evaluation problem \eqref{eq:PE_minimax}, averaged over 5 problem instances}
\label{fig:PE_results}
\vspace{-3ex}
\end{figure}

As shown in Figure~\ref{fig:PE_results}, the top-performing variants from Subsection~\ref{subsec:AUC_optimization},  namely \texttt{VrFRBS-SVRG}, \texttt{VrFRBS-SAGA}, \texttt{VrFRBS-SARAH}, and \texttt{VrFRBS-HSVRG}, continue to yield the best results in this policy evaluation task. 
However, the variant utilizing the \eqref{eq:saga_estimator} estimator now converges faster than the other three, and this advantage is more pronounced in \textit{Experiment 2}. 
\revise{\texttt{VrFRBS-HSGD} and \texttt{VrFRBS-SGD(imb)}, together with the baseline competitors \texttt{VFRBS} and \texttt{VEG}, exhibit the slowest convergence overall.}

\beforesec
\section*{Acknowledgements and Declarations}
\aftersec
\noindent\textbf{Acknowledgements}.
This work is partly supported by the National Science Foundation: NSF-RTG grant No. NSF DMS-2134107 and the Office of Naval Research (ONR), grant No. N00014-23-1-2588.

\vspace{1ex}
\noindent\textbf{Conflicts of interest/competing interests.}
The authors declare that they have no conflicts of interest or competing interests related to this work.

\vspace{1ex}
\noindent\textbf{Data availability.}
We evaluate our methods and their competitors using both synthetic datasets, generated according to the procedure described herein, and publicly available data from \cite{sutton2018reinforcement}.

\appendix
\beforesec
\section{Appendix: Technical lemmas}\label{sec:useful_lemmas}
\aftersec
This appendix recalls two technical lemmas, which will be used in our convergence analysis.
The first lemma is the well-known Robbins-Siegmund supermartingale theorem in \cite{robbins1971convergence}.

\begin{lemma}[\cite{robbins1971convergence}]\label{le:RS_lemma}
Let $\sets{U_k}$, $\sets{\alpha_k}$, $\sets{V_k}$ and $\sets{R_k}$ be  sequences of nonnegative integrable random variables on some arbitrary probability space and adapted to the filtration $\set{\Fc_k}$ with $\sum_{k=0}^{\infty}\alpha_k < +\infty$ and $\sum_{k=0}^{\infty}R_k < +\infty$ almost surely, and 
\begin{equation}\label{eq:RS_martingale_cond}
\Exp{ U_{k+1} \mid \Fc_k} \leq (1 + \alpha_k)U_k - V_k + R_k, 
\end{equation}
almost surely for all $k \geq 0$.
Then, $\sets{U_k}$ converges almost surely to a nonnegative random variable $U$ and $\sum_{k=0}^{\infty}V_k <+\infty$ almost surely.
\end{lemma}

The following lemma can be proved similarly to \cite[Theorem 3.2.]{combettes2015stochastic} and \cite[Proposition 4.1.]{davis2022variance}.
We recall it here without repeating the proof to cover the composite mapping $\Phi$ in \eqref{eq:GE}.

\begin{lemma}\label{le:Opial_lemma}
Suppose that $\gra{\Phi}$ of $\Phi$ in \eqref{eq:GE} is closed.
Let $\sets{x^k}$ be a sequence of random vectors such that for all $x^{\star} \in \zer{\Phi}$, the sequence $\sets{\norms{x^k - x^{\star}}^2}$ converges almost surely  to a nonnegative random variable. 
In addition, assume that $\sets{\norms{w^k}}$ for $(x^k, w^k) \in \gra{\Phi}$ also almost surely converges to zero.
Then, $\sets{x^k}$ converges  almost surely to a $\zer{\Phi}$-valued random variable $x^{\star} \in \zer{\Phi}$.
\end{lemma}

\beforesec
\section{Appendix: The Proof of Technical Results in Section~\ref{sec:VrFRBS}}\label{apdx:sec:VrFRBS}
\aftersec
This appendix provides the missing proof of the technical results in Section~\ref{sec:VrFRBS}.

\beforesubsec
\subsection{\textbf{The proof of Lemma~\ref{le:svrg_estimator} --- The L-SVRG estimator}}\label{apdx:le:svrg_estimator}
\aftersubsec
First, let $G\tilde{x}^k$ be the exact evaluation of $G$ at $\tilde{x}^k$.
We define
\begin{equation}\label{eq:A1_lm2_proof1}
\widehat{S}^k := G\tilde{x}^k - G_{\mcal{B}_k}\tilde{x}^k + 2G_{\mcal{B}_k}x^k - G_{\mcal{B}_k}x^{k-1} \quad\textrm{and} \quad  r^k := \overline{G}\tilde{x}^k - G\tilde{x}^k.
\end{equation}
Then, we have $\widetilde{S}^k = \widehat{S}^k + r^k$.
Since $\mathbb{E}_{\mcal{B}_k}[\overline{G}\tilde{x}^k] = G\tilde{x}^k$, $\mathbb{E}_{\mcal{B}_k}[G_{\mcal{B}_k}x] = Gx$ for $x = x^k$ and $x = x^{k-1}$, we have $\mathbb{E}_{\mcal{B}_k}[r^k ] = 0$, and consequently
\begin{equation}\label{eq:A1_lm2_proof1b}
\mathbb{E}_{\mcal{B}_k}[\widetilde{S}^k] = G\tilde{x}^k - G\tilde{x}^k + 2Gx^k - Gx^{k-1} = S^k.
\end{equation}
Taking the conditional expectation $\mathbb{E}_{k}[\cdot]$ on both sides of this relation, we verify the first line of \eqref{eq:Vr_estimator}.

Next, let us define $X^k(\xi) := 2\mbf{G}(x^k, \xi) - \mbf{G}(x^{k-1}, \xi) - \mbf{G}(\tilde{x}^k, \xi)$.
Then, we have $\mathbb{E}_{\xi}[ X^k(\xi)] = 2Gx^k - Gx^{k-1} - G\tilde{x}^k = S^k - G\tilde{x}^k =: \bar{X}^k$ and $\widehat{S}^k = \frac{1}{b}\sum_{\xi\in\mcal{B}_k}X^k(\xi) + G\tilde{x}^k$ due to \eqref{eq:A1_lm2_proof1}.
Consequently, we have $\widehat{S}^k - S^k = \frac{1}{b}\sum_{\xi \in\mcal{B}_k}(X^k(\xi) - \bar{X}^k)$.
Utilizing this expression, we can show that
\begin{equation}\label{eq:A1_lm2_proof2}
\arraycolsep=0.2em
\begin{array}{lcl}
\mathbb{E}_{\mcal{B}_k }[ \norms{ \widehat{S}^k - S^k}^2 ] &= & \mathbb{E}_{\mcal{B}_k} \big[ \Vert \frac{1}{b}\sum_{\xi \in\mcal{B}_k}(X^k(\xi) - \bar{X}^k) \Vert^2 \big] \vspace{1ex}\\
& \stackrel{\ncmt{\tiny\textcircled{1}}}{=} & \frac{1}{b^2} \sum_{\xi \in \mcal{B}_k} \mathbb{E}_{\xi} \big[ \norms{X^k(\xi) - \bar{X}^k }^2  \big] \vspace{1ex}\\
& \stackrel{\ncmt{\tiny\textcircled{2}}}{\leq} & \frac{1}{b^2} \sum_{\xi \in \mcal{B}_k} \mathbb{E}_{\xi} \big[ \norms{X^k(\xi) }^2  \big]  \vspace{1ex}\\
& = & \frac{1}{b}\mathbb{E}_{\xi}\big[ \norms{2\mbf{G}(x^k, \xi) - \mbf{G}(x^{k-1}, \xi) - \mbf{G}(\tilde{x}^k, \xi)}^2 \big],
\end{array}
\end{equation}
\ncmt{where {\tiny\textcircled{1}} holds due to the i.i.d. property of $\Bc_k$ and {\tiny\textcircled{2}} holds due to the property of variance.}
Let us denote by $\hat{\Delta}_k := \frac{1}{b}\mathbb{E}_{\xi}\big[ \norms{2\mbf{G}(x^k, \xi) - \mbf{G}(x^{k-1}, \xi) - \mbf{G}(\tilde{x}^k, \xi)}^2 \big]$.
Then, by the update rule \eqref{eq:svrg_wk} of $\tilde{x}^k$, and Young's inequality, for any $c > 0$, we have 
\begin{equation*}
\arraycolsep=0.2em
\begin{array}{lcl}
\Tc_{[1]} &:= & \mathbb{E}_{\xi, \tilde{x}^k}\big[ \norms{2\mbf{G}(x^k, \xi) - \mbf{G}(x^{k-1}, \xi) - \mbf{G}(\tilde{x}^k, \xi)}^2 \big] \vspace{1ex}\\
&= & (1 - \mbf{p}) \mathbb{E}_{\xi}\big[ \norms{2\mbf{G}(x^k, \xi) - \mbf{G}(x^{k-1}, \xi) - \mbf{G}(\tilde{x}^{k-1}, \xi)}^2 \big] \vspace{1ex}\\
&& + {~} \mbf{p} \mathbb{E}_{\xi}\big[ \norms{2\mbf{G}(x^k, \xi) - \mbf{G}(x^{k-1}, \xi) - \mbf{G}(x^{k-1}, \xi)}^2 \big] \vspace{1ex}\\
&\leq & (1 - \mbf{p})(1+c) \mathbb{E}_{\xi}\big[ \norms{2\mbf{G}(x^{k-1}, \xi) - \mbf{G}(x^{k-2}, \xi) - \mbf{G}(\tilde{x}^{k-1}, \xi)}^2 \big] \vspace{1ex}\\
&& + {~} \frac{(1-\mbf{p})(1+c)}{c}\mathbb{E}_{\xi}\big[ \norms{2(\mbf{G}(x^k, \xi) - \mbf{G}(x^{k-1}, \xi)) - (\mbf{G}(x^{k-1}, \xi) - \mbf{G}(x^{k-2}, \xi)) }^2 \big] \vspace{1ex}\\
&& + {~} 4\mbf{p} \mathbb{E}_{\xi}\big[ \norms{\mbf{G}(x^k, \xi) - \mbf{G}(x^{k-1}, \xi) }^2 \big] \vspace{1ex}\\
& \leq & (1 - \mbf{p})(1+c) \mathbb{E}_{\xi}\big[ \norms{2\mbf{G}(x^{k-1}, \xi) - \mbf{G}(x^{k-2}, \xi) - \mbf{G}(\tilde{x}^{k-1}, \xi)}^2 \big] \vspace{1ex}\\
&& + {~} \big[ \frac{8(1-\mbf{p})(1+c)}{c} + 4\mbf{p} \big] \mathbb{E}_{\xi}\big[ \norms{\mbf{G}(x^k, \xi) - \mbf{G}(x^{k-1}, \xi) }^2 \big] \vspace{1ex}\\
&&  + {~} \frac{2(1-\mbf{p})(1+c)}{c} \mathbb{E}_{\xi}\big[ \norms{\mbf{G}(x^{k-1}, \xi) - \mbf{G}(x^{k-2}, \xi) }^2 \big].
\end{array}
\end{equation*}
Let us choose $c := \frac{\mbf{p}}{2(1 - \mbf{p})}> 0$. 
Then, by applying \eqref{eq:Lipschitz_cond} from Assumption~\ref{as:A1}, we obtain from the last inequality that 
\begin{equation*}
\arraycolsep=0.2em
\begin{array}{lcl}
\Tc_{[1]} &:= & \mathbb{E}_{\xi, \tilde{x}^k}\big[ \norms{2\mbf{G}(x^k, \xi) - \mbf{G}(x^{k-1}, \xi) - \mbf{G}(\tilde{x}^k, \xi)}^2 \big] \vspace{1ex}\\
& \leq & \big(1 - \frac{\mbf{p}}{2} \big) \mathbb{E}_{\xi}\big[ \norms{2\mbf{G}(x^{k-1}, \xi) - \mbf{G}(x^{k-2}, \xi) - \mbf{G}(\tilde{x}^{k-1}, \xi)}^2 \big] \vspace{1ex}\\
&& + {~}  \frac{4(4 - 6\mbf{p} + 3\mbf{p}^2)L^2}{\mbf{p}}\norms{x^k - x^{k-1}}^2 +  \frac{2(2 - 3\mbf{p} + \mbf{p}^2)L^2}{\mbf{p} }\norms{x^{k-1} - x^{k-2} }^2.
\end{array}
\end{equation*}
Multiplying this expression by $\frac{1}{b}$, then taking the conditional expectation $\mathbb{E}_k[\cdot]$ on both sides of the result, and using the definition of $\hat{\Delta}_k$, $4 - 6\mbf{p} + 3\mbf{p}^2 \leq 4$, and $2 - 3\mbf{p} + \mbf{p}^2 \leq 2$,  we get
\begin{equation}\label{eq:A1_lm2_proof3}
\arraycolsep=0.2em
\begin{array}{lcl}
\Expsn{k}{ \hat{\Delta}_k} &\leq & \big(1 - \frac{\mbf{p}}{2}\big) \hat{\Delta}_{k-1} + \frac{16L^2}{b\mbf{p}}\norms{x^k - x^{k-1} }^2 + \frac{4L^2}{b\mbf{p}}\norms{x^{k-1} - x^{k-2} }^2. 
\end{array}
\end{equation}
Now, we consider two cases.
\begin{compactitem}
\item[(a)] 
First, if $\overline{G}\tilde{x}^k = G\tilde{x}^k$ (exact evaluation), then $\widetilde{S}^k = \widehat{S}^k$, and thus, $\Expsn{k}{\widetilde{S}^k } = \Expsn{k}{\widehat{S}^k } = S^k$, which fulfills the first condition of \eqref{eq:Vr_estimator}.

Next, if we define $\Delta_k := \hat{\Delta}_k$, then we also have $\Expsn{k}{ \norms{\widetilde{S}^k - S^k}^2 } = \Expsn{k}{ \norms{\widehat{S}^k - S^k}^2 } \overset{\tiny\eqref{eq:A1_lm2_proof2}}{ \leq } \Expsn{k}{ \hat{\Delta}_k  } = \Expsn{k}{ \Delta_k}$, which satisfies the second condition of \eqref{eq:Vr_estimator}.

Finally, since $\Delta_k = \hat{\Delta}_k$, we obtain from \eqref{eq:A1_lm2_proof3} that
\begin{equation*} 
\arraycolsep=0.2em
\begin{array}{lcl}
\Expsn{k}{ \Delta_k} &\leq & \big(1 - \frac{\mbf{p}}{2}\big) \Delta_{k-1} + \frac{16L^2}{b\mbf{p}}\norms{x^k - x^{k-1} }^2 + \frac{4L^2}{b\mbf{p}}\norms{x^{k-1} - x^{k-2} }^2. 
\end{array}
\end{equation*}
This shows that the third condition of \eqref{eq:Vr_estimator} holds with $\kappa := \frac{\mbf{p}}{2}$, $\Theta := \frac{16L^2}{b\mbf{p}}$, $\hat{\Theta} := \frac{4L^2}{b\mbf{p}}$, and $\delta_k := 0$.

\item[(b)]
Second, if $\overline{G}\tilde{x}^k$ is constructed by the mega-batch estimator $\overline{G}\tilde{x}^k := \frac{1}{n_k}\sum_{\xi \in \mcal{M}_k}\mbf{G}( \tilde{x}^k, \xi)$, we have from  \eqref{eq:A1_lm2_proof1} that $r^k := \overline{G}\tilde{x}^k - G\tilde{x}^k$ and $\Expsn{k}{ \norms{ \overline{G}\tilde{x}^k - G\tilde{x}^k}^2 } \leq \frac{\sigma^2}{n_k}$.
From \eqref{eq:A1_lm2_proof1b}, we have $\Expsn{k}{\widetilde{S}^k} = S^k$, which fulfills the first condition of \eqref{eq:Vr_estimator}.

Next, by Young's inequality, for any $\mu > 0$, we have
\begin{equation}\label{eq:A1_lm2_proof4} 
\arraycolsep=0.2em
\begin{array}{lcl}
\Expsn{k}{ \norms{\widetilde{S}^k - S^k}^2 } & \leq & (1 + \mu)\Expsn{k}{ \norms{\widehat{S}^k - S^k}^2 } + \frac{1 + \mu}{\mu}\Expsn{k}{ \norms{\overline{G}\tilde{x}^k - G\tilde{x}^k}^2 } \vspace{1ex}\\
& \leq & (1 + \mu) \Expsn{k}{ \hat{\Delta}_k } + \frac{(1 + \mu) \sigma^2}{\mu n_k} \vspace{1ex}\\
& = & \Expsn{k}{ (1 + \mu) \hat{\Delta}_k + \frac{\ncmt{(1+\mu)}\sigma^2}{\mu n_k} }.
\end{array}
\end{equation}
Therefore, if we define $\Delta_k := (1+\mu)\hat{\Delta}_k + \frac{(1+\mu)\sigma^2}{\mu n_k}$, then \eqref{eq:A1_lm2_proof4} shows that $\Expsn{k}{ \norms{\widetilde{S}^k - S^k}^2 } \leq \Expsn{k}{\Delta_k}$, which shows that the second condition of \eqref{eq:Vr_estimator} holds.

Finally, combining \eqref{eq:A1_lm2_proof3} and \eqref{eq:A1_lm2_proof4} and using the definition of $\Delta_k$ and $n_k \geq n_{k-1}$, we have
\begin{equation*} 
\ncmt{\arraycolsep=0.2em
\begin{array}{lcl}
	\Expsn{k}{ \Delta_k} &\leq & \big(1 - \frac{\mbf{p}}{2}\big) \Delta_{k-1} + \frac{16(1+\mu)L^2}{b\mbf{p}}\norms{x^k - x^{k-1} }^2 + \frac{4(1+\mu)L^2}{b\mbf{p}}\norms{x^{k-1} - x^{k-2} }^2 + \frac{(1+\mu)\mbf{p}\sigma^2}{2\mu n_k}. 
\end{array}}
\end{equation*}
This inequality shows that the third condition of \eqref{eq:Vr_estimator} holds with $\kappa := \frac{\mbf{p}}{2}$, $\Theta :=  \frac{16(1+\mu)L^2}{b\mbf{p}}$, $\hat{\Theta} :=  \frac{4(1+\mu)L^2}{b\mbf{p}}$, and $\delta_k := \frac{(1+\mu)\mbf{p}\sigma^2}{2\mu n_k}$.
\end{compactitem}
Since we have verified Definition~\ref{de:Vr_estimator} for the two cases (a) and (b), we have completed the proof.
\Eproof

\beforesubsec
\subsection{\textbf{The proof of Lemma~\ref{le:key_estimate2} --- The descent property of the Lyapunov function}}\label{apdx:le:key_estimate2}
\aftersubsec
\revise{Since $w^k$ and $y^k$ are $\mathcal{F}_k$-measurable and $\mbb{E}_k[ e^k ] = 0$, we have 
\begin{equation}\label{eq:lm34_proof1}
\mathbb{E}_k[ \iprods{e^k, w^k} ]  = \iprods{ \mathbb{E}_k[ e^k ], w^k} = 0
\quad \text{and} \quad
\mathbb{E}_k[ \iprods{ e^k, y^k - x^{\star} } ] = \iprods{ \mathbb{E}_k[ e^k ], y^k - x^{\star} } = 0.
\end{equation}
}
Taking the conditional expectation $\Exps{k}{\cdot}$ on both sides of \eqref{eq:key_estimate1}, and using the relations in \eqref{eq:lm34_proof1}, and $\norms{Gx^k - Gx^{k-1}}^2 \leq L^2\norms{x^k - x^{k-1} }^2$ due to the $L$-Lipschitz continuity of $G$, we obtain
\begin{equation*}
\arraycolsep=0.2em
\begin{array}{lcl}
\Exps{k}{\norms{y^{k+1} - x^{\star}}^2 }  &\leq & \norms{y^k - x^{\star}}^2  - 2\eta\iprods{w^k, x^k - x^{\star}} - \norms{y^k - x^k}^2 + \frac{\eta L^2}{\gamma}\norms{x^k -  x^{k-1} }^2 \vspace{1ex}\\
&& - {~} \ (1 - \gamma\eta) \Exps{k}{ \norms{y^{k+1} - x^k}^2 } + 2\eta^2\Exps{k}{\norms{e^k}^2}. 
\end{array} 
\end{equation*}
By \eqref{eq:weak_Minty_cond} of Assumption~\ref{as:A1}, we have $\iprods{w^k, x^k - x^{\star}} \geq -\rho\norms{w^k}^2$.
Moreover, since $w^k = \eta^{-1}(y^k - x^k) + w^k - \hat{w}^k = \eta^{-1}(y^k - x^k) +  Gx^k - Gx^{k-1}$ due to \eqref{eq:w_notation} and \eqref{eq:FRBS_reform}, by Young's inequality, we get
\begin{equation*}
\arraycolsep=0.2em
\begin{array}{lcl}
\iprods{w^k, x^k - x^{\star}}  & \geq  & -\rho\norms{w^k}^2 \vspace{1ex}\\
& = & -\rho\norms{\eta^{-1}(y^k - x^k) + Gx^k - Gx^{k-1} }^2 \vspace{1ex}\\
& \geq & -\frac{2\rho}{\eta^2}\norms{y^k - x^k}^2 - 2\rho\norms{Gx^k - Gx^{k-1}}^2 \vspace{1ex}\\
& \geq & -\frac{2\rho}{\eta^2}\norms{y^k - x^k}^2 - 2\rho L^2\norms{x^k - x^{k-1}}^2.
\end{array}
\end{equation*}
Combining the last two inequalities, for some $c > 0$, we can derive that
\begin{equation}\label{eq:lm2_proof3}
\arraycolsep=0.2em
\begin{array}{lcl}
\Exps{k}{\norms{y^{k+1} - x^{\star}}^2 }  &\leq & \norms{y^k - x^{\star}}^2   - \big(1 - \frac{4\rho}{\eta}\big) \norms{y^k - x^k}^2 + L^2\eta\big(4\rho + \frac{1}{\gamma} \big)\norms{x^k -  x^{k-1} }^2 \vspace{1ex}\\
&& - {~} (1 - \gamma\eta) \Exps{k}{ \norms{y^{k+1} - x^k}^2 } + 2\eta^2\Exps{k}{\norms{e^k}^2} \vspace{1ex}\\
& \leq & \norms{y^k - x^{\star}}^2   - \Big[1 - \frac{4\rho}{\eta} - 2L^2\eta\big(4\rho + \frac{1}{\gamma} + c\big) \Big] \norms{y^k - x^k}^2  \vspace{1ex}\\
&& - {~} (1 - \gamma\eta) \Exps{k}{ \norms{y^{k+1} - x^k}^2 } + 2L^2\eta\big(4\rho + \frac{1}{\gamma} + c\big)\norms{y^k - x^{k-1}}^2 \vspace{1ex}\\
&& - {~}  cL^2\eta \norms{x^k -  x^{k-1} }^2 + 2\eta^2\Exps{k}{\norms{e^k}^2},
\end{array} 
\end{equation}
where we have used Young's inequality $\norms{x^k - x^{k-1}}^2 \leq 2\norms{x^k - y^k}^2 + 2\norms{y^k - x^{k-1}}^2$ in the last line.

Next, from \eqref{eq:Vr_estimator} of Definition~\ref{de:Vr_estimator}, we have 
\begin{equation*}
\arraycolsep=0.2em
\begin{array}{lcl}
\Exps{k}{\norms{e^k}^2} & \leq & \Exps{k}{\Delta_k} \vspace{1ex}\\
&\leq & \frac{1-\kappa}{\kappa}\big[\Delta_{k-1} - \Exps{k}{\Delta_k}\big] + \frac{\hat{\Theta}}{\kappa}\big[\norms{x^{k-1} - x^{k-2}}^2 - \norms{x^k - x^{k-1}}^2 \big] + \frac{\delta_k}{\kappa} \vspace{1ex}\\
&& + {~} \frac{\Theta + \hat{\Theta} }{\kappa}\norms{x^k - x^{k-1}}^2.
\end{array} 
\end{equation*}
Substituting this inequality into \eqref{eq:lm2_proof3}, we get
\begin{equation*}
\arraycolsep=0.2em
\begin{array}{lcl}
\Exps{k}{\norms{y^{k+1} - x^{\star}}^2 } 
& \leq & \norms{y^k - x^{\star}}^2   - \Big[ 1 - \frac{4\rho}{\eta} - 2L^2\eta\big(4\rho + \frac{1}{\gamma} + c\big) \Big] \norms{y^k - x^k}^2  \vspace{1ex}\\
&& - {~} (1 - \gamma\eta) \Exps{k}{ \norms{y^{k+1} - x^k}^2 } + 2L^2\eta\big(4\rho + \frac{1}{\gamma} + c\big)\norms{y^k - x^{k-1}}^2 \vspace{1ex}\\
&& +  \frac{2\eta^2(1-\kappa)}{\kappa}\big[\Delta_{k-1} - \Exps{k}{\Delta_k}\big]  +  \frac{2\eta^2 \hat{\Theta}}{\kappa}\big[\norms{x^{k-1} - x^{k-2}}^2 - \norms{x^k - x^{k-1}}^2 \big] \vspace{1ex}\\
&& - {~} \big[ cL^2\eta - \frac{2\eta^2(\Theta + \hat{\Theta}) }{\kappa} \big] \norms{x^k - x^{k-1}}^2 + \frac{2\eta^2 \delta_k}{\kappa}.
\end{array} 
\end{equation*}
Using $\Pc_k$ from \eqref{eq:Pk_func}, the last estimate leads to
\begin{equation*}
\arraycolsep=0.2em
\begin{array}{lcl}
\Exps{k}{\Pc_{k+1} } & \leq & \Pc_k  - \Big[ 1 - \frac{4\rho}{\eta} - 2L^2\eta\big(4\rho + \frac{1}{\gamma} + c\big) \Big] \norms{y^k - x^k}^2  \vspace{1ex}\\
&& - {~} \Big[1 - \gamma\eta -  2L^2\eta\big(4\rho + \frac{1}{\gamma} + c\big) \Big] \norms{y^k - x^{k-1}}^2 \vspace{1ex}\\
&& - {~} \frac{\eta [ c\kappa L^2 - 2\eta(\Theta + \hat{\Theta})] }{\kappa}  \norms{x^k - x^{k-1}}^2 + \frac{2\eta^2 \delta_k}{\kappa},
\end{array} 
\end{equation*}
which proves \eqref{eq:key_estimate2}.
\Eproof

\beforesec
\section{Appendix: The Proof of Technical Results in Section~\ref{sec:VrFRBS_method2}}\label{apdx:sec:VrFRBS_method2}
\aftersec
This appendix presents the missing proof of the technical results in Section~\ref{sec:VrFRBS_method2}.
\beforesubsec
\subsection{\textbf{The proof of Lemma~\ref{le:sarah_estimator} --- The L-SARAH estimator}}\label{apdx:le:sarah_estimator}
\aftersubsec
For $\widetilde{S}^k$ constructed by \eqref{eq:sarah_estimator}, we denote $X^k(\xi) := 2\mbf{G}(x^k, \xi) - 3\mbf{G}(x^{k-1}, \xi) + \mbf{G}(x^{k-2}, \xi)$ and $X^k := \frac{1}{b}\sum_{\xi \in \mcal{B}_k}X^k(\xi)$.
Then, we have $\Expsn{\xi}{X^k(\xi)} = \bar{X}^k :=  2Gx^k - 3Gx^{k-1} + Gx^{k-2} = S^k - S^{k-1}$ and  $\Expsn{\mcal{B}_k}{X^k} = 2Gx^k - 3Gx^{k-1} + Gx^{k-2} = S^k - S^{k-1} = \bar{X}^k$.
Applying these facts, \ncmt{the i.i.d. property of the mini-batch $\Bc_k$,} and Young's inequality, we can derive that
\begin{equation*} 
\hspace{-2ex}
\arraycolsep=0.2em
\begin{array}{lcl}
\Expsn{\mcal{B}_k}{\norms{ X^k - S^k + S^{k-1}}^2 } & = & \Expsn{\mcal{B}_k}{ \norms{ \frac{1}{b}\sum_{\xi\in\mcal{B}_k}(X^k(\xi) - \bar{X}^k)}^2 } \vspace{1ex}\\
& = & \frac{1}{b^2}\sum_{\xi\in\mcal{B}_k}\Expsn{\xi}{ \norms{X^k(\xi) - \bar{X}^k}^2 } \vspace{1ex}\\
& = & \frac{1}{b^2}\sum_{\xi\in\mcal{B}_k}\big[ \Expsn{\xi}{ \norms{X^k(\xi)}^2 } - \norms{ \Expsn{\xi}{X^k(\xi) } }^2 \big] \vspace{1ex}\\
& \leq & \frac{1}{b}\Expsn{\xi}{ \norms{X^k(\xi) }^2 } \vspace{1ex}\\
& = & \frac{1}{b}\Expsn{\xi}{ \norms{  2\mbf{G}(x^k, \xi) - 3\mbf{G}(x^{k-1}, \xi) + \mbf{G}(x^{k-2}, \xi)  }^2 } \vspace{1ex}\\
& \leq & \frac{8}{b}\Expsn{\xi}{ \norms{  \mbf{G}(x^k, \xi) - \mbf{G}(x^{k-1}, \xi)  }^2 } \vspace{1ex}\\
&& + {~}  \frac{2}{b}\Expsn{\xi}{ \norms{  \mbf{G}(x^{k-1}, \xi) - \mbf{G}(x^{k-2}, \xi)  }^2 }.
\end{array}
\hspace{-8ex}
\end{equation*}
Utilizing  \eqref{eq:Lipschitz_cond} in  Assumption~\ref{as:A1}, we obtain from the last expression that
\begin{equation}\label{eq:sarah_lm1_proof1}
\arraycolsep=0.2em
\begin{array}{lcl}
\Expsn{\mcal{B}_k}{\norms{ X^k - S^k + S^{k-1}}^2 } & \leq & \frac{8L^2}{b} \norms{ x^k - x^{k-1} }^2  +  \frac{2L^2}{b}\norms{x^{k-1} - x^{k-2}}^2.
\end{array}
\end{equation}
Next, let $s_k$ be the random variable representing the switching rule in \eqref{eq:sarah_estimator}.
Then, we have 
\begin{equation*}
\arraycolsep=0.1em
\begin{array}{lcl}
\Expsn{s_k,\mcal{B}_k}{\norms{ \widetilde{S}^k - S^k }^2 } & = & (1 - \mbf{p} ) \Expsn{\mcal{B}_k}{\norms{ \widetilde{S}^{k-1} + X^k - S^k }^2} + \mbf{p} \Expsn{\mcal{B}_k}{ \norms{ \overline{S}^k - S^k }^2 } \vspace{1ex}\\
& = & (1 - \mbf{p} ) \Expsn{\mcal{B}_k}{\norms{ \widetilde{S}^{k-1} - S^{k-1} + (X^k - S^k + S^{k-1})}^2 } + \mbf{p}  \Expsn{\mcal{B}_k}{\norms{ \overline{S}^k - S^k }^2} \vspace{1ex}\\
& = & (1 - \mbf{p} ) \norms{ \widetilde{S}^{k-1} - S^{k-1} }^2 + 2(1 - \mbf{p} ) \Expsn{\mcal{B}_k}{\iprods{ \widetilde{S}^{k-1} - S^{k-1}, X^k - S^k + S^{k-1}} } \vspace{1ex}\\
&& + {~} (1 - \mbf{p} ) \Expsn{\mcal{B}_k}{\norms{ X^k - S^k + S^{k-1}}^2 } + \mbf{p}  \Expsn{\mcal{B}_k}{\norms{ \overline{S}^k - S^k }^2} \vspace{1ex}\\
&\stackrel{\ncmt{\tiny\textcircled{1}}}{=} &  (1 - \mbf{p} ) \norms{ \widetilde{S}^{k-1} - S^{k-1} }^2 + (1 - \mbf{p} ) \Expsn{\mcal{B}_k}{\norms{ X^k - S^k + S^{k-1}}^2 }  \vspace{1ex}\\
&&  + {~} \mbf{p}  \Expsn{\mcal{B}_k}{\norms{ \overline{S}^k - S^k }^2} \vspace{1ex}\\
& \overset{\tiny\eqref{eq:sarah_lm1_proof1}}{\leq}  &  (1 - \mbf{p} ) \norms{ \widetilde{S}^{k-1} - S^{k-1} }^2 +  \frac{8(1 - \mbf{p} )L^2}{b} \norms{ x^k - x^{k-1} }^2    \vspace{1ex}\\
&& + {~} \frac{2(1 - \mbf{p} )L^2}{b}\norms{x^{k-1} - x^{k-2}}^2 + \mbf{p}  \Expsn{\mcal{B}_k}{\norms{ \overline{S}^k - S^k }^2} \vspace{1ex}\\
\end{array}
\end{equation*}
\ncmt{Here, we have used in {\tiny\textcircled{1}} that $\Expsn{\Bc_k}{X^k - S^k + S^{k-1}} = 0$ and $X^k - S^k + S^{k-1}$ is conditionally independent of $\widetilde{S}^{k-1} - S^{k-1}$.}
Taking the conditional expectation $\Expsn{k}{\cdot}$ on both sides of this inequality and using the fact that $\Expsn{k}{ \norms{ \overline{S}^k - S^k }^2 } \leq \sigma_k^2$, we get
\begin{equation*}
\arraycolsep=0.1em
\begin{array}{lcl}
\Expsn{k}{\norms{ \widetilde{S}^k - S^k }^2 } &\leq &   (1 - \mbf{p} ) \norms{ \widetilde{S}^{k-1} - S^{k-1} }^2 +  \frac{8(1 - \mbf{p} )L^2}{b} \norms{ x^k - x^{k-1} }^2  \vspace{1ex}\\
&& + {~}  \frac{2(1 - \mbf{p} )L^2}{b}\norms{x^{k-1} - x^{k-2}}^2 +  \mbf{p}\sigma^2_k.
\end{array}
\end{equation*}
Clearly, if we denote $\Delta_k := \norms{ \widetilde{S}^k - S^k }^2$, then the last inequality shows that $\widetilde{S}^k$ constructed by \eqref{eq:sarah_estimator} satisfies the last two conditions in \eqref{eq:Vr_estimator2} of Definition~\ref{de:Vr_estimator2} with $\kappa := \mbf{p}$, $\Theta := \frac{8(1 - \mbf{p} )L^2}{b}$, $\hat{\Theta} := \frac{2(1 - \mbf{p} )L^2}{b}$, and $\delta_k := \mbf{p}\sigma_k^2$.

Finally, using again $\Expsn{\mcal{B}_k}{\ncmt{X^k}} = S^k - S^{k-1}$ and $\Expsn{\ncmt{\Mc_k}}{\overline{S}^k}  = S^k$, we have
\begin{equation*}
\arraycolsep=0.1em
\begin{array}{lcl}
\Expsn{k}{\widetilde{S}^k} &= & (1-\mbf{p})\Expsn{\mcal{B}_k}{ \widetilde{S}^{k-1} + \ncmt{X^k}} + \mbf{p}\Expsn{\ncmt{\Mc_k}}{\overline{S}^k} \vspace{1ex}\\
& = & (1 - \mbf{p})(\widetilde{S}^{k-1} + S^k - S^{k-1}) + \mbf{p}S^k  \vspace{1ex}\\
& = & (1- \mbf{p})(\widetilde{S}^{k-1} - S^{k-1}) + S^k.
\end{array}
\end{equation*}
This expression shows that $\widetilde{S}^k$ satisfies the first condition in \eqref{eq:Vr_estimator2} of Definition~\ref{de:Vr_estimator2} with $\tau := \mbf{p}$.
\Eproof

\beforesubsec
\subsection{\textbf{The proof of Lemma~\ref{le:hsgd_estimator} --- The HSGD estimator}}\label{apdx:le:hsgd_estimator}
\aftersubsec
First, for $\widetilde{S}^k$ defined by \eqref{eq:hsgd_estimator}, let $e^k := \widetilde{S}^k - S^k$ be the approximation error between $\widetilde{S}^k$ and $S^k$.
Let $X^k(\xi) := 2\mbf{G}(x^k, \xi) - 3\mbf{G}(x^{k-1}, \xi) + \mbf{G}(x^{k-2}, \xi) - (2Gx^k - 3Gx^{k-1} + Gx^{k-2})$ and $Y^k := \overline{S}^k_{\hat{\mcal{B}}_k} - S^k$.
Then, it follows immediately that $\Expsn{\xi}{X^k(\xi)} = 0$ and $\Expsn{\hat{\mcal{B}}_k}{Y^k} = 0$ by our assumption.

Next, if we denote $X^k := \frac{1}{b}\sum_{\xi \in \mcal{B}_k}X^k(\xi)$, then similar to the proof of \eqref{eq:sarah_lm1_proof1} in Lemma~\ref{le:sarah_estimator}, we have
\begin{equation}\label{eq:hsgd_lm1_proof1}
\arraycolsep=0.2em
\begin{array}{lcl}
\Expsn{\mcal{B}_k}{\norms{ X^k}^2 } & \leq & \frac{8L^2}{b} \norms{ x^k - x^{k-1} }^2  +  \frac{2L^2}{b}\norms{x^{k-1} - x^{k-2}}^2.
\end{array}
\end{equation}
Now, utilizing again the update \eqref{eq:hsgd_estimator} and the definitions of $X^k$ and $Y^k$, we can write
\begin{equation}\label{eq:hsgd_lm1_proof1b}
\arraycolsep=0.1em
\begin{array}{lcl}
e^k & = &   \widetilde{S}^k - S^k \vspace{1ex}\\
& \overset{\tiny\eqref{eq:hsgd_estimator}}{=} &  (1- \omega) \widetilde{S}^{k-1} + (1-\omega)\big[2G_{\mcal{B}_k}x^k - 3G_{\mcal{B}_k}x^{k-1} + G_{\mcal{B}_k}x^{k-2} \big]  + \omega\overline{S}_{\hat{\mcal{B}}_k}^k - S^k \vspace{1ex}\\
& = &  (1-\omega)( \widetilde{S}^{k-1} - S^{k-1} ) + \omega( \overline{S}_{\hat{\mcal{B}}_k}^k - S^k ) \vspace{1ex}\\
&& + {~}  (1-\omega)\big[ 2G_{\mcal{B}_k}x^k - 3G_{\mcal{B}_k}x^{k-1} + G_{\mcal{B}_k}x^{k-2}  - (S^k - S^{k-1}) \big] \vspace{1ex}\\
& = &  (1- \omega)e^{k-1}   +  (1 - \omega)X^k + \omega Y^k.
\end{array}
\end{equation}
This expression leads to 
\begin{equation*}
\arraycolsep=0.2em
\begin{array}{lcl}
\norms{e^k}^2 & = &  (1-\omega)^2\norms{e^{k-1}}^2 +  (1-\omega)^2\norms{X^k}^2 + \omega^2\norms{Y^k}^2 \vspace{1ex}\\
&& + {~} 2(1-\omega)^2\iprods{e^{k-1}, X^k} + 2\omega(1-\omega) \iprods{X^k, Y^k} + 2\omega(1-\omega) \iprods{e^{k-1}, Y^k}.
\end{array}
\end{equation*}
Taking the conditional expectation $\Expsn{(\mcal{B}_k,\hat{\mcal{B}}_k)}{\cdot}$ on both sides of this expression and using $\Expsn{(\mcal{B}_k,\hat{\mcal{B}}_k)}{X^k} = \Expsn{\hat{\mcal{B}}_k}{\Expsn{\mcal{B}_k}{X^k  \mid \hat{\mcal{B}}_k } } = 0$ and $\Expsn{(\mcal{B}_k,\hat{\mcal{B}}_k)}{Y^k} = \Expsn{\mcal{B}_k}{\Expsn{\hat{\mcal{B}}_k}{Y^k  \mid \mcal{B}_k } } = 0$, we can show that
\begin{equation}\label{eq:hsgd_lm1_proof2}
\arraycolsep=0.2em
\begin{array}{lcl}
\Expsn{(\mcal{B}_k,\hat{\mcal{B}}_k)}{\norms{e^k}^2} & = &  (1-\omega)^2\norms{e^{k-1}}^2 +  (1-\omega)^2\Expsn{\mcal{B}_k}{ \norms{X^k}^2 } + \omega^2 \Expsn{\hat{\mcal{B}}_k}{ \norms{Y^k}^2 } \vspace{1ex}\\
&& + {~}  2(1-\omega)\omega \Expsn{(\mcal{B}_k,\hat{\mcal{B}}_k)}{ \iprods{X^k, Y^k} }.
\end{array}
\end{equation}
Here, we have used the facts that $X^k$ only depends on $\mcal{B}_k$ and $Y^k$ only depends on $\hat{\mcal{B}}_k$.

\noindent
Now, we consider two cases as follows.
\begin{compactitem}
\item[(i)]~If $\mcal{B}_k$ and $\hat{\mcal{B}}_k$ are independent, then $\Expsn{(\mcal{B}_k,\hat{\mcal{B}}_k)}{ \iprods{X^k, Y^k} \mid \Fc_k} = 0$.
Using this fact, \eqref{eq:hsgd_lm1_proof1}, and $\sigma_k^2 :=  \Expsn{\hat{\mcal{B}}_k}{ \norms{Y^k}^2 }$ into \eqref{eq:hsgd_lm1_proof2}, and then taking the conditional expectation $\Expsn{k}{\cdot}$ on both sides of the result, we obtain
\begin{equation*}
\arraycolsep=0.2em
\begin{array}{lcl}
\Expsn{k}{\norms{e^k}^2} \leq (1-\omega)^2\norms{e^{k-1}}^2 + \frac{8(1-\omega)^2L^2}{b} \norms{ x^k - x^{k-1} }^2  +  \frac{2(1-\omega)^2L^2}{b}\norms{x^{k-1} - x^{k-2}}^2 + \omega^2\sigma_k^2.
\end{array}
\end{equation*}
Clearly, if we define $\Delta_k := \norms{e^k}^2$, then the last inequality shows that $\widetilde{S}^k$ constructed by \eqref{eq:hsgd_estimator} satisfies the last two conditions of \eqref{eq:Vr_estimator2} of Definition~\ref{de:Vr_estimator2} with $\kappa := \omega(2-\omega) \in [0, 1]$, $\Theta := \frac{8(1-\omega)^2L^2}{b}$, $\hat{\Theta} := \frac{2(1-\omega)^2L^2}{b}$, and $\delta_k := \omega^2\sigma^2_k$.

\item[(ii)]~If $\mcal{B}_k$ and $\hat{\mcal{B}}_k$ are not independent, then by Young's inequality, we have
\begin{equation*}
\arraycolsep=0.2em
\begin{array}{lcl}
2(1-\omega)\omega \Expsn{(\mcal{B}_k,\hat{\mcal{B}}_k)}{ \iprods{X^k, Y^k} } & \leq &  (1-\omega)^2 \Expsn{\mcal{B}_k}{ \norms{X^k}^2 } + \omega^2 \Expsn{\hat{\mcal{B}}_k}{ \norms{Y^k}^2 }.
\end{array}
\end{equation*}
Substituting this inequality,  \eqref{eq:hsgd_lm1_proof1}, and $\sigma_k^2 :=  \Expsn{\hat{\mcal{B}}_k}{ \norms{Y^k}^2 }$ into \eqref{eq:hsgd_lm1_proof2}, and taking the conditional expectation $\Expsn{k}{\cdot}$ on both sides of the result, we obtain  
\begin{equation*}
\arraycolsep=0.2em
\begin{array}{lcl}
\Expsn{k}{\norms{e^k}^2} \leq (1-\omega)^2\norms{e^{k-1}}^2 + \frac{16(1-\omega)^2L^2}{b} \norms{ x^k - x^{k-1} }^2  +  \frac{4(1-\omega)^2L^2}{b}\norms{x^{k-1} - x^{k-2}}^2 + 2\omega^2\sigma_k^2.
\end{array}
\end{equation*}
Again, if we define $\Delta_k := \norms{e^k}^2$, then the last inequality shows that $\widetilde{S}^k$ constructed by \eqref{eq:hsgd_estimator} satisfies the last two conditions of \eqref{eq:Vr_estimator2} of Definition~\ref{de:Vr_estimator2} with $\kappa := \omega(2-\omega)\in (0, 1]$, $\Theta := \frac{16(1-\omega)^2L^2}{b}$, $\hat{\Theta} := \frac{4(1-\omega)^2L^2}{b}$, and $\delta_k := 2\omega^2\sigma^2_k$.
\end{compactitem}
Finally, since $\Expsn{\mcal{B}_k}{X^k} = 0$ and $\Expsn{\hat{\mcal{B}}_k}{Y^k} = 0$, it is easy to obtain from \eqref{eq:hsgd_lm1_proof1b} that $\Expsn{(\mcal{B}_k, \hat{\mcal{B}}_k)}{e^k} = (1-\omega)e^{k-1}$.
Taking the conditional expectation $\Expsn{k}{\cdot}$ on both sides of this expression, we verify the first condition of \eqref{eq:Vr_estimator2} in Definition~\ref{de:Vr_estimator2} with $\tau := \omega$.
\Eproof

\beforesubsec
\subsection{\textbf{The proof of Lemma~\ref{le:hsvrg_estimator} --- The HSVRG estimator}}\label{apdx:le:hsvrg_estimator}
\aftersubsec
First, for $\widetilde{S}^k$ constructed by \eqref{eq:hsvrg_estimator}, let $e^k := \widetilde{S}^k - S^k$.
We also define $X^k(\xi) := 2\mbf{G}(x^k, \xi) - 3\mbf{G}(x^{k-1}, \xi) + \mbf{G}(x^{k-2}, \xi) - (2Gx^k - 3Gx^{k-1} + Gx^{k-2})$ and $Y^k := \overline{S}^k_{\hat{\mcal{B}}_k} - S^k$.
Then, we have $\Expsn{\xi}{X^k(\xi)} = 0$ and $\Expsn{\hat{\mcal{B}}_k}{Y^k} = 0$ by our assumption.
In addition, if we denote $X^k := \frac{1}{b}\sum_{\xi \in \mcal{B}_k}X^k(\xi)$, then similar to the proof of \eqref{eq:sarah_lm1_proof1} in Lemma~\ref{le:sarah_estimator}, we also have
\begin{equation}\label{eq:hsvrg_lm1_proof1}
\arraycolsep=0.2em
\begin{array}{lcl}
\Expsn{\mcal{B}_k}{\norms{X^k}^2} & \leq & \frac{8L^2}{b} \norms{ x^k - x^{k-1} }^2  +  \frac{2L^2}{b}\norms{x^{k-1} - x^{k-2}}^2.
\end{array}
\end{equation}
Moreover, \ncmt{to avoid ambiguity, let $D_k$ and $d_k$ denote the quantities $\Delta_k$ and $\delta_k$ associated with the SVRG estimator $\overline{S}^k_{\hat{\Bc}_k}$ as defined in Lemma~\ref{le:svrg_estimator}, thereby distinguishing them from the quantities $\Delta_k$ and $\delta_k$ constructed below for the \eqref{eq:hsvrg_estimator} estimator.
Then, we know from Lemma~\ref{le:svrg_estimator} that} 
\begin{equation}\label{eq:hsvrg_lm1_proof2}
\arraycolsep=0.2em
\left\{\begin{array}{lcl}
\Expsn{k}{\norms{Y^k}^2} & \leq & \Expsn{k}{D_k}, \vspace{1ex}\\
\Expsn{k}{D_k} &\leq& (1 - \frac{\mbf{p}}{2}) D_{k-1} + \frac{16\ncmt{(1+\mu)}L^2}{\hat{b}\mbf{p}}\norms{x^k - x^{k-1}}^2 + \frac{4\ncmt{(1+\mu)}L^2}{\hat{b}\mbf{p}} \norms{x^{k-1} - x^{k-2}}^2 + \ncmt{d_k},
\end{array}\right.
\end{equation}
\ncmt{where $\mu := 0$ and $d_k := 0$ if the snapshot operator $\overline{G}w^k = Gw^k$ in $\overline{S}^k_{\hat{\Bc}_k}$ is evaluated exactly, and $\mu > 0$ and $d_k := \frac{(1+\mu)\mbf{p}\sigma^2}{2\mu n_k}$ if $\overline{G}w^k = \frac{1}{n_k}\sum_{\xi \in \Mc_k} \mbf{G}(w^k, \xi)$ is a mega-batch estimator of size $n_k$.}

Next, similar to \eqref{eq:hsgd_lm1_proof2} in the proof of Lemma~\ref{le:hsgd_estimator} in Appendix~\ref{apdx:le:hsgd_estimator}, we can prove that
\begin{equation}\label{eq:hsvrg_lm1_proof4}
\arraycolsep=0.2em
\begin{array}{lcl}
\Expsn{(\mcal{B}_k,\hat{\mcal{B}}_k)}{\norms{e^k}^2} & = &  (1-\omega)^2\norms{e^{k-1}}^2 +  (1-\omega)^2\Expsn{\mcal{B}_k}{ \norms{X^k}^2 } + \omega^2 \Expsn{\hat{\mcal{B}}_k}{ \norms{Y^k}^2 } \vspace{1ex}\\
&& + {~}  2(1-\omega)\omega \Expsn{(\mcal{B}_k,\hat{\mcal{B}}_k)}{ \iprods{X^k, Y^k} }.
\end{array}
\end{equation}

\noindent
Now, we consider two cases as follows.
\begin{compactitem}
\item[(a)]~If $\mcal{B}_k$ and $\hat{\mcal{B}}_k$ are independent, then $\Expsn{(\mcal{B}_k,\hat{\mcal{B}}_k)}{ \iprods{X^k, Y^k} \mid \Fc_k} = 0$.
Substituting this fact and \eqref{eq:hsvrg_lm1_proof1} into \eqref{eq:hsvrg_lm1_proof4}, and then taking the conditional expectation $\Expsn{k}{\cdot}$ on both sides of the result and using the first line of \eqref{eq:hsvrg_lm1_proof2}, we obtain
\begin{equation*}
\arraycolsep=0.2em
\begin{array}{lcl}
\Expsn{k}{\norms{e^k}^2} \leq (1-\omega)^2\norms{e^{k-1}}^2 + \frac{8(1-\omega)^2L^2}{b} \norms{ x^k - x^{k-1} }^2  +  \frac{2(1-\omega)^2L^2}{b}\norms{x^{k-1} - x^{k-2}}^2 + \omega^2\Expsn{k}{\ncmt{D_k}}.
\end{array}
\end{equation*}
Multiplying both sides of the second line of \eqref{eq:hsvrg_lm1_proof2} by $c > 0$, then adding the result to the last inequality \ncmt{and rearranging the result}, we get
\begin{equation}\label{eq:hsvrg_lm1_proof4b}
\hspace{-1ex}
\arraycolsep=0.2em
\begin{array}{lcl}
\Expsn{k}{\norms{e^k}^2 + (c - \omega^2)\ncmt{D_k}} &\leq& (1-\omega)^2\norms{e^{k-1}}^2 + c(1-\frac{\mbf{p}}{2})\ncmt{D_{k-1}} + 8L^2\left[ \frac{(1-\omega)^2}{b} + \frac{2c\ncmt{(1+\mu)}}{\hat{b}\mbf{p}} \right] \norms{ x^k - x^{k-1} }^2  \vspace{1ex}\\
&& + {~}  2L^2\left[ \frac{(1-\omega)^2}{b} + \frac{2c\ncmt{(1+\mu)}}{\hat{b}\mbf{p}} \right]\norms{x^{k-1} - x^{k-2}}^2 \ncmt{ + c d_k}.
\end{array}
\hspace{-2ex}
\end{equation}
We need to choose $c > \omega^2$ such that $(1-\omega)^2\norms{e^{k-1}}^2 + c(1-\frac{\mbf{p}}{2})\ncmt{D_{k-1}} \leq (1-\kappa)\left[\norms{e^{k-1}}^2 + (c - \omega^2)\ncmt{D_{k-1}}\right]$ for some $\kappa \in (0, 1]$, which holds if
\begin{equation}\label{eq:hsvrg_lm1_proof5}
\arraycolsep=0.2em
\begin{array}{lcl}
(1 - \omega)^2 \leq 1 - \kappa  \qquad \text{and} \qquad
c(1 - \frac{\mbf{p}}{2}) \leq (1-\kappa)(c - \omega^2).
\end{array}
\end{equation}
Since $c - \omega^2 > 0$, these conditions hold if we have
\begin{equation*} 
\arraycolsep=0.2em
\begin{array}{lcl}
\kappa \leq \min\sets{\omega(2-\omega), \frac{\frac{c\mbf{p}}{2} - \omega^2}{c - \omega^2}}.
\end{array}
\end{equation*}
Clearly, if we choose $c := \frac{4\omega^2}{\mbf{p}} > \omega^2$, then \ncmt{the choice $\kappa := \min\sets{\omega(2-\omega), \frac{\mbf{p}}{4}} \in (0,1]$ satisfies this condition at equality.}

Now, under this choice of \ncmt{$c$ and} $\kappa$, if we define $\Delta_k := \norms{e^k}^2 + (c-\omega^2)\ncmt{D_k} = \ncmt{\norms{e^k}^2 + (\frac{4}{\mbf{p}} - 1)\omega^2D_k}$, then \eqref{eq:hsvrg_lm1_proof4b} becomes
\begin{equation*}
\arraycolsep=0.2em
\begin{array}{lcl}
\Expsn{k}{\Delta_k} &\leq& (1-\kappa)\Delta_{k-1} + 8L^2\left[ \frac{(1-\omega)^2}{b} + \frac{2c\ncmt{(1+\mu)}}{\hat{b}\mbf{p}} \right] \norms{ x^k - x^{k-1} }^2  \vspace{1ex}\\
&& + {~}  2L^2\left[ \frac{(1-\omega)^2}{b} + \frac{2c\ncmt{(1+\mu)}}{\hat{b}\mbf{p}} \right]\norms{x^{k-1} - x^{k-2}}^2 \ncmt{+ \frac{4\omega^2}{\mbf{p}} d_k}.
\end{array}
\end{equation*}
This shows that $\widetilde{S}^k$ constructed by \eqref{eq:hsvrg_estimator} satisfies the last two conditions of \eqref{eq:Vr_estimator2} of Definition~\ref{de:Vr_estimator2} with $\Theta :=  8L^2\left[ \frac{(1-\omega)^2}{b} + \frac{8\ncmt{(1+\mu)}\omega^2}{\hat{b}\mbf{p}^2} \right]$, $\hat{\Theta} := 2L^2\left[ \frac{(1-\omega)^2}{b} + \frac{8\ncmt{(1+\mu)}\omega^2}{\hat{b}\mbf{p}^2} \right]$, $\kappa := \min\sets{\omega(2-\omega), \frac{\mbf{p}}{4}}$, \ncmt{and $\delta_k := \frac{4\omega^2}{\mbf{p}} d_k$}.

\item[(b)]~If $\mcal{B}_k$ and $\hat{\mcal{B}}_k$ are not independent, then by Young's inequality, we have
\begin{equation*}
\arraycolsep=0.2em
\begin{array}{lcl}
2(1-\omega)\omega \Expsn{(\mcal{B}_k,\hat{\mcal{B}}_k)}{ \iprods{X^k, Y^k} } & \leq &  (1-\omega)^2 \Expsn{\mcal{B}_k}{ \norms{X^k}^2 } + \omega^2 \Expsn{\hat{\mcal{B}}_k}{ \norms{Y^k}^2 }.
\end{array}
\end{equation*}
Substituting this inequality into \eqref{eq:hsvrg_lm1_proof4}, then taking the conditional expectation $\Expsn{k}{\cdot}$ on both sides of the result and using the first line of \eqref{eq:hsvrg_lm1_proof2}, we obtain
\begin{equation*}
\arraycolsep=0.2em
\begin{array}{lcl}
\Expsn{k}{\norms{e^k}^2} \leq (1-\omega)^2\norms{e^{k-1}}^2 + \frac{16(1-\omega)^2L^2}{b} \norms{ x^k - x^{k-1} }^2  +  \frac{4(1-\omega)^2L^2}{b}\norms{x^{k-1} - x^{k-2}}^2 + 2\omega^2\Expsn{k}{\ncmt{D_k}}.
\end{array}
\end{equation*}
By a similar proof as in (a), if we define $\Delta_k := \ncmt{\norms{e^k}^2 + (\frac{4}{\mbf{p}} - 1)\omega^2D_k}$, then we can use the last inequality to show that $\widetilde{S}^k$ constructed by \eqref{eq:hsvrg_estimator} satisfies the last two conditions of \eqref{eq:Vr_estimator2} of Definition~\ref{de:Vr_estimator2} with $\Theta := 8L^2\left[ \frac{2(1-\omega)^2}{b} + \frac{8\ncmt{(1+\mu)}\omega^2}{\hat{b}\mbf{p}^2} \right]$, $\hat{\Theta} := 2L^2\left[ \frac{2(1-\omega)^2}{b} + \frac{8\ncmt{(1+\mu)}\omega^2}{\hat{b}\mbf{p}^2} \right]$, $\kappa := \min\sets{\omega(2-\omega), \frac{\mbf{p}}{4}}$, \ncmt{and $\delta_k := \frac{4\omega^2}{\mbf{p}} d_k$}.
\end{compactitem}
Finally, \ncmt{similar to Lemma~\ref{le:hsgd_estimator}, we can prove that \eqref{eq:hsvrg_estimator} satisfies} the first condition of \eqref{eq:Vr_estimator2} in Definition~\ref{de:Vr_estimator2} with $\tau := \omega$.
\Eproof

\beforesubsec
\subsection{\textbf{The proof of Lemma~\ref{le:A2_key_estimate1} --- The property of $\Qc_k$}}\label{apdx:le:A2_key_estimate1}
\aftersubsec
First, utilizing the first condition $\Exps{k}{e^k} = (1-\tau)e^{k-1}$ from \eqref{eq:Vr_estimator2}, and Young's inequality, \revise{for $c_1 > 0$ in $\Qc_k$ from \eqref{eq:Qk_func}}, we have
\begin{equation*}
\arraycolsep=0.2em
\begin{array}{lcl}
\Exps{k}{\iprods{e^k, y^k - x^{\star}}} &= & \Exps{k}{\iprods{e^k, y^{k-1} - x^{\star}}} + \Exps{k}{ \iprods{e^k, y^k - y^{k-1} } } \vspace{1ex}\\
& \geq & (1-\tau)\iprods{e^{k-1}, y^{k-1} - x^{\star}} - \frac{c_1}{2}\Exps{k}{\norms{e^k}^2 } - \frac{1}{2c_1}\norms{y^k - y^{k-1}}^2.
\end{array}
\end{equation*}
This inequality implies that
\begin{equation}\label{eq:lm4_proof2}
\arraycolsep=0.2em
\begin{array}{lcl}
\Exps{k}{\iprods{e^k, y^k - x^{\star}}}   & \geq & \frac{1-\tau}{\tau}\Big[ \iprods{e^{k-1}, y^{k-1} - x^{\star}} - \Exps{k}{\iprods{e^k, y^k - x^{\star}}} \Big] \vspace{1ex}\\
&& - \frac{c_1}{2\tau }\Exps{k}{\norms{e^k}^2 } - \frac{1}{2c_1 \tau}\norms{y^k - y^{k-1}}^2.
\end{array}
\end{equation}
Next, using again Young's inequality, the relation $w^k = \eta^{-1}(y^k - x^k) + Gx^k - Gx^{k-1}$ from \eqref{eq:FRBS_reform} as before, and then the $L$-Lipschitz continuity of $G$ for any $c_2 > 0$, we have
\begin{equation}\label{eq:lm4_proof3}
\arraycolsep=0.2em
\begin{array}{lcl}
\iprods{e^k, w^k} &\leq & \frac{c_2}{2}\norms{e^k}^2 + \frac{1}{2c_2}\norms{w^k}^2 \vspace{1ex}\\
& \leq &  \frac{c_2}{2}\norms{e^k}^2 + \frac{1}{c_2\eta^2}\norms{y^k - x^k}^2 + \frac{1}{c_2}\norms{\ncmt{Gx^k - Gx^{k-1}}}^2 \vspace{1ex}\\
& \leq &  \frac{c_2}{2}\norms{e^k}^2 + \frac{1}{c_2\eta^2}\norms{y^k - x^k}^2 + \frac{L^2}{c_2}\norms{x^k - x^{k-1}}^2.
\end{array}
\end{equation}
Similar to the proof of \eqref{eq:lm2_proof3}, we have
\begin{equation*}
\arraycolsep=0.2em
\begin{array}{lcl}
\Exps{k}{\norms{y^{k+1} - x^{\star}}^2 }  &\leq & \norms{y^k - x^{\star}}^2   - \big(1 - \frac{4\rho}{\eta}\big) \norms{y^k - x^k}^2 + L^2\eta\big(4\rho + \frac{1}{\gamma} \big)\norms{x^k -  x^{k-1} }^2 \vspace{1ex}\\
&& - {~} (1 - \gamma\eta) \Exps{k}{ \norms{y^{k+1} - x^k}^2 } + 2\eta^2\Exps{k}{\norms{e^k}^2} \vspace{1ex}\\
&& - {~} 2 \eta \Exps{k}{\iprods{e^k, y^k - x^{\star}}} + 2\eta^2\Exps{k}{ \iprods{e^k, w^k} }.
\end{array} 
\end{equation*}
Substituting \eqref{eq:lm4_proof2} and \eqref{eq:lm4_proof3} into the last inequality, we can show that
\begin{equation*}
\arraycolsep=0.2em
\begin{array}{lcl}
\Exps{k}{\norms{y^{k+1} - x^{\star}}^2 }  &\leq & \norms{y^k - x^{\star}}^2   - \big(1 - \frac{4\rho}{\eta} - \frac{2}{c_2}\big) \norms{y^k - x^k}^2 + L^2\eta\big(4\rho + \frac{1}{\gamma} + \frac{2\eta}{c_2}\big)\norms{x^k -  x^{k-1} }^2 \vspace{1ex}\\
&& - {~} (1 - \gamma\eta) \Exps{k}{ \norms{y^{k+1} - x^k}^2 } + \eta\big( 2\eta + \frac{c_1}{\tau} + c_2\eta \big)\Exps{k}{\norms{e^k}^2} \vspace{1ex}\\
&& - {~}  \frac{2\eta( 1-\tau) }{\tau}\Big[ \iprods{e^{k-1}, y^{k-1} - x^{\star}} - \Exps{k}{\iprods{e^k, y^k - x^{\star}}} \Big] + \frac{\eta}{c_1 \tau}\norms{y^k - y^{k-1}}^2.
\end{array} 
\end{equation*}
By Young's inequality again, we get $\norms{x^k - x^{k-1}}^2 \leq 2\norms{x^k - y^k}^2 + 2\norms{y^k - x^{k-1}}^2$ and $\norms{y^k - y^{k-1}}^2 \leq 2\norms{y^k - x^{k-1}}^2 + 2\norms{x^{k-1} - y^{k-1}}^2$.
Substituting these inequalities into the last estimate, \ncmt{for any $c > 0$,} we get
\begin{equation*}
\arraycolsep=0.2em
\begin{array}{lcl}
\Exps{k}{\norms{y^{k+1} - x^{\star}}^2 }  &\leq & \norms{y^k - x^{\star}}^2   - \Big[ 1 - \frac{4\rho}{\eta} - \frac{2}{c_2} - 2L^2\eta\big(c + 4\rho + \frac{1}{\gamma} + \frac{2\eta}{c_2}\big) \Big] \norms{y^k - x^k}^2 \vspace{1ex}\\
&& - {~}  (1 - \gamma\eta) \Exps{k}{ \norms{y^{k+1} - x^k}^2 }   + \frac{2\eta}{c_1\tau} \norms{y^{k-1} - x^{k-1}}^2 - cL^2\eta \norms{x^k -  x^{k-1} }^2 \vspace{1ex}\\
&& + {~} \Big[ \frac{2\eta}{c_1 \tau} + 2L^2\eta\big(c + 4\rho + \frac{1}{\gamma} + \frac{2\eta}{c_2}\big)   \Big] \norms{y^k - x^{k-1}}^2 + \eta\big( 2\eta + \frac{c_1}{\tau} + c_2\eta \big)\Exps{k}{\norms{e^k}^2}  \vspace{1ex}\\
&& - {~}  \frac{2\eta( 1-\tau) }{\tau}\Big[ \iprods{e^{k-1}, y^{k-1} - x^{\star}} - \Exps{k}{\iprods{e^k, y^k - x^{\star}}} \Big].
\end{array} 
\end{equation*}
Using  $\Qc_k$ from \eqref{eq:Qk_func}, we can show that
\begin{equation*}
\arraycolsep=0.2em
\begin{array}{lcl}
\Exps{k}{ \Qc_{k+1} }  &\leq &\Qc_k  - \Big[ 1 - \frac{2\eta}{c_1\tau}  - \frac{4\rho}{\eta} - \frac{2}{c_2} - 2L^2\eta\big(c + 4\rho + \frac{1}{\gamma} + \frac{2\eta}{c_2}\big) \Big] \norms{y^k - x^k}^2 \vspace{1ex}\\
&& - {~} \Big\{ 1 - \gamma \eta - \Big[ \frac{2\eta}{c_1 \tau} + 2L^2\eta\big(c + 4\rho + \frac{1}{\gamma} + \frac{2\eta}{c_2}\big)   \Big] \Big\} \norms{y^k - x^{k-1}}^2 \vspace{1ex}\\
&& - {~} cL^2\eta \norms{x^k -  x^{k-1} }^2 + \eta\big( 2\eta + \frac{c_1}{\tau} + c_2\eta \big)\Exps{k}{\norms{e^k}^2}.
\end{array} 
\end{equation*}
This proves \eqref{eq:A2_key_estimate1}.

Next, by Young's inequality,  we can easily show that 
\begin{equation*}
\arraycolsep=0.2em
\begin{array}{lcl}
\frac{2\eta(1-\tau)}{\tau}\iprods{e^{k-1}, x^{\star} - y^{k-1}} & \geq & -\frac{8(1-\tau)^2\eta^2}{\tau^2}\norms{e^{k-1}}^2  -  \frac{1}{4}\norms{y^k - x^{\star}}^2 \vspace{1ex}\\
&& - {~}  \frac{1}{2}\norms{y^k - x^{k-1}}^2 -  \frac{1}{2}\norms{y^{k-1} - x^{k-1}}^2.
\end{array} 
\end{equation*}
Using this inequality into $\Qc_k$ from \eqref{eq:Qk_func}, we obtain
\begin{equation*} 
\arraycolsep=0.2em
\begin{array}{lcl}
\Qc_k & \geq &  \frac{3}{4} \norms{y^k - x^{\star}}^2 + \frac{1 - 2\gamma \eta}{2} \norms{y^k - x^{k-1}}^2  + \big( \frac{2\eta}{c_1\tau} - \frac{1}{2} \big) \norms{x^{k-1} -  y^{k-1}}^2  - \frac{8(1-\tau)^2\eta^2}{\tau^2}\norms{e^{k-1}}^2,
\end{array}
\end{equation*}
which proves \eqref{eq:A2_Qk_lower_bound}.
\Eproof

\beforesubsec
\subsection{\textbf{The proof of Lemma~\ref{le:A2_key_estimate2} --- The descent property of the Lyapunov function}}\label{apdx:le:A2_key_estimate2}
\aftersubsec
Let us choose $c_1 = \frac{4\eta}{\tau}$, $c_2 = \ncmt{8}$, $\gamma := \frac{1}{4\eta}$, and $c = 4\eta$ in \eqref{eq:A2_key_estimate1}.
Then, from \eqref{eq:A2_key_estimate1}, we have
\begin{equation*}
\arraycolsep=0.2em
\begin{array}{lcl}
\Exps{k}{ \Qc_{k+1} }  &\leq &\Qc_k  - \Big[ \frac{1}{4}  -  \frac{4\rho}{\eta}  - \ncmt{\frac{L^2\eta}{2}\big(33\eta + 16\rho\big)} \Big] \norms{y^k - x^k}^2 \vspace{1ex}\\
&& - {~} \Big[ \frac{1}{4} - \ncmt{\frac{L^2\eta}{2}\big(33\eta + 16\rho\big)} \Big] \ \norms{y^k - x^{k-1}}^2 \vspace{1ex}\\
&& - {~} 4L^2\eta^2 \norms{x^k -  x^{k-1} }^2 + \eta^2\big(\frac{4}{\tau^2} + \ncmt{10} \big)\Exps{k}{\norms{e^k}^2}.
\end{array} 
\end{equation*}
This inequality implies
\begin{equation*}
\arraycolsep=0.2em
\begin{array}{lcl}
\Tc_{[1]} & := &  \Exps{k}{ \Qc_{k+1} }  + \frac{8(1-\tau)^2\eta^2}{\tau^2} \Exps{k}{ \norms{e^k }^2 } \vspace{1ex}\\
&\leq &\Qc_k + \frac{8(1-\tau)^2\eta^2}{\tau^2}\norms{e^{k-1}}^2 - \Big[ \frac{1}{4}  -  \frac{4\rho}{\eta}  - \ncmt{\frac{L^2\eta}{2}\big(33\eta + 16\rho\big)} \Big] \norms{y^k - x^k}^2 \vspace{1ex}\\
&& - {~} \Big[ \frac{1}{4} - \ncmt{\frac{L^2\eta}{2}\big(33\eta + 16\rho\big)} \Big] \ \norms{y^k - x^{k-1}}^2 - \frac{8(1-\tau)^2\eta^2}{\tau^2}\norms{e^{k-1}}^2 \vspace{1ex}\\
&& - {~} 4L^2\eta^2 \norms{x^k -  x^{k-1} }^2 + \eta^2\big(\frac{12}{\tau^2} + \ncmt{10} \big)\Exps{k}{\norms{e^k}^2},
\end{array} 
\end{equation*}
\ncmt{where we have used $(1-\tau)^2 \leq 1$ for all $\tau \in (0,1]$.}
Next, from \eqref{eq:Vr_estimator2} of Definition~\ref{de:Vr_estimator2}, we have 
\begin{equation}\label{eq:lm2_A2_proof3}
\arraycolsep=0.2em
\begin{array}{lcl}
\Exps{k}{\norms{e^k}^2} & \leq & \Exps{k}{\Delta_k} \vspace{1ex}\\
&\leq & \frac{1-\kappa}{\kappa}\big[\Delta_{k-1} - \Exps{k}{\Delta_k}\big] + \frac{\hat{\Theta}}{\kappa}\big[\norms{x^{k-1} - x^{k-2}}^2 - \norms{x^k - x^{k-1}}^2 \big] + \frac{\delta_k}{\kappa} \vspace{1ex}\\
&& + {~} \frac{\Theta + \hat{\Theta} }{\kappa}\norms{x^k - x^{k-1}}^2.
\end{array} 
\end{equation}
Substituting this inequality  \eqref{eq:lm2_A2_proof3} into $\Tc_{[1]}$, and using $\Lc_k$ from \eqref{eq:A2_Lk_func}, we get
\begin{equation*}
\arraycolsep=0.2em
\begin{array}{lcl}
\Exps{k}{ \Lc_{k+1} }  &\leq &\Lc_k  - \Big[ \frac{1}{4}  -  \frac{4\rho}{\eta}  - \ncmt{\frac{L^2\eta}{2}\big(33\eta + 16\rho\big)} \Big] \norms{y^k - x^k}^2 \vspace{1ex}\\
&& - {~} \Big[ \frac{1}{4} - \ncmt{\frac{L^2\eta}{2}\big(33\eta + 16\rho\big)} \Big] \ \norms{y^k - x^{k-1}}^2 - \frac{8(1-\tau)^2\eta^2}{\tau^2}\norms{e^{k-1}}^2 \vspace{1ex}\\
&& - {~} \eta^2 \big[ 4L^2 - \big(\frac{12}{\tau^2} + \ncmt{10} \big)\frac{\Theta + \hat{\Theta}}{\kappa} \big] \norms{x^k -  x^{k-1} }^2 +  \big(\frac{12}{\tau^2} + \ncmt{10} \big)\frac{\eta^2\delta_k}{\kappa},
\end{array} 
\end{equation*}
which proves \eqref{eq:A2_key_estimate2}.

Finally, from \eqref{eq:A2_Lk_func} and \eqref{eq:A2_Qk_lower_bound}, we can show that
\begin{equation*}
\arraycolsep=0.2em
\begin{array}{lcl}
\Lc_k &\geq & \Qc_k +  \frac{8(1-\tau)^2\eta^2}{\tau^2}\norms{e^{k-1}}^2 \geq \frac{3}{4}\norms{y^k - x^{\star}}^2,
\end{array} 
\end{equation*}
which proves \eqref{eq:A2_Lk_lower_bound}.
\Eproof

\beforesec
\section{Appendix: The Proof of Proposition~\ref{prop:PE_cohypomonotone} in Section~\ref{sec:numerical_experiments}}\label{apdx:sec:numerical_experiments}
\aftersec
	First, since $\hat{A}$ is full-row rank, and $\hat{C} \succ 0$, we have $\mu := \lambda_{\min}(\hat{A}^{\top}\hat{C}^{-1}\hat{A}) \geq  \frac{\sigma_{\min}(\hat{A})^2}{\lambda_{\max}(\hat{C})} > 0$ and $c := \lambda_{\min}(\hat{C}) > 0$.
	Next, let $\theta, \theta' \in \R^d$, $\xi \in \partial f(\theta)$, and $\xi' \in \partial f(\theta')$ be arbitrarily.
	By \cite{fan2001variable}, $f$ is $\frac{1}{a-1}$-weakly convex.
	Hence, we get
	\begin{equation*}
		\begin{array}{lcl}
			\iprods{\xi - \xi', \theta - \theta'} \geq -\frac{1}{a-1}\norms{\theta - \theta'}^2.
		\end{array}
	\end{equation*}
	Let $x = [\theta, w]$, $x' = [\theta', w']$, $u \in \Phi x$, and $u' \in \Phi x'$.
	Then, we can compute that
	\begin{equation*}
		u - u' = \begin{bmatrix}
			[u - u']_{\theta} \\ [u - u']_{w}
		\end{bmatrix} 
		= \begin{bmatrix}
			-\hat{A}^\top (w - w') + \nu(\xi - \xi') \\
			\hat{A}(\theta - \theta') + \hat{C}(w - w')
		\end{bmatrix}.
	\end{equation*}
	In this case, by taking the inner product of $u - u'$ and $x - x'$, the last relation gives
	\begin{equation}\label{eq:prop_cohypomonotone_proof1}
		\begin{array}{lcl}
			\iprods{u - u', x - x'} &=& \nu \iprods{\xi - \xi', \theta - \theta'} + \iprods{\hat{C}(w - w'), w - w'}  
			\geq -\frac{\nu}{a-1}\norms{\theta - \theta'}^2 + \norms{\hat{C}^{1/2}(w-w')}^2.
		\end{array}
	\end{equation}
	We consider two cases.
	\begin{compactitem}
		\item[$\bullet$] \textit{Case 1.} 
		If $\iprods{u - u', x - x'} \geq 0$, then the desired co-hypomonotonicity of $\Phi$ holds for any $\rho \geq 0$.

		\item[$\bullet$] \textit{Case 2.} 
		If $\iprods{u - u', x - x'} < 0$, then we get from \eqref{eq:prop_cohypomonotone_proof1} that $\norms{\hat{C}^{1/2}(w-w')}^2 < \frac{\nu}{a-1}\norms{\theta - \theta'}^2$.
		Moreover, with $\mu := \lambda_{\min}(\hat{A}^\top \hat{C}^{-1}\hat{A})$, for any $s \in \R^d$, we have $\norms{\hat{C}^{-1/2}\hat{A}s}^2 = s^{\top} \hat{A}^\top \hat{C}^{-1/2}\hat{C}^{-1/2} \hat{A}s = s^{\top} \hat{A}^\top \hat{C}^{-1} \hat{A}s \geq \mu\norms{s}^2$.
		From the above two inequalities, and $\hat{C}^{-1/2} [u - u']_w = \hat{C}^{-1/2}\hat{A}(\theta - \theta') + \hat{C}^{1/2}(w - w')$, using the triangle inequality, and $\nu < (a-1)\mu$, we can derive that
		\begin{equation*} 
			\begin{array}{lcl}
				\norms{\hat{C}^{-1/2} [u - u']_w} &\geq& \norms{\hat{C}^{-1/2}\hat{A}(\theta - \theta')} - \norms{\hat{C}^{1/2}(w - w')} \vspace{1ex}\\
				&\geq& \big(\sqrt{\mu} - \sqrt{\frac{\nu}{a-1}}\big)\norms{\theta - \theta'} \geq 0.
			\end{array}
		\end{equation*}
		However, for $c = \lambda_{\min}(\hat{C})$, we also have
		\begin{equation*} 
			\begin{array}{lcl}
				\norms{[u - u']_w} = \norms{\hat{C}^{1/2}\big(\hat{C}^{-1/2}[u - u']_w\big)}  \geq \sqrt{c}\norms{\hat{C}^{-1/2}[u - u']_w}.
			\end{array}
		\end{equation*}
		Combining the last two relations, we obtain
		\begin{equation*} 
			\begin{array}{lcl}
				\norms{\theta - \theta'}^2 \leq \frac{1}{c\big(\sqrt{\mu} - \sqrt{\frac{\nu}{a-1}}\big)^2}\norms{[u - u']_w}^2 \leq \frac{1}{c\big(\sqrt{\mu} - \sqrt{\frac{\nu}{a-1}}\big)^2}\norms{u - u'}^2.
			\end{array}
		\end{equation*}
		Finally, dropping the last term on the right-hand side of \eqref{eq:prop_cohypomonotone_proof1} and then substituting the last relation into the resulting inequality, we arrive at
		\begin{equation*}
			\begin{array}{lcl}
				\iprods{u - u', x - x'} &\geq& -\frac{\nu}{(a-1)c\big(\sqrt{\mu} - \sqrt{\frac{\nu}{a-1}}\big)^2}\norms{u - u'}^2.
			\end{array}
		\end{equation*}
		This shows that $\Phi$ is $\rho$-co-hypomonotone with $\rho := \frac{\nu}{(a-1)c\big(\sqrt{\mu} - \sqrt{\frac{\nu}{a-1}}\big)^2} \geq 0$.
	\end{compactitem}
	From both cases, we conclude that $\Phi$ is $\rho$-co-hypomonotone. 
\Eproof

\bibliographystyle{plain}

\begin{thebibliography}{10}

\bibitem{alacaoglu2021stochastic}
A.~Alacaoglu and Y.~Malitsky.
\newblock Stochastic variance reduction for variational inequality methods.
\newblock In {\em Conference on Learning Theory}, pages 778--816. PMLR, 2022.

\bibitem{alacaoglu2021forward}
A.~Alacaoglu, Y.~Malitsky, and V.~Cevher.
\newblock Forward-reflected-backward method with variance reduction.
\newblock {\em Comput. Optim. Appl.}, 80(2):321--346, 2021.

\bibitem{arjovsky2017wasserstein}
M.~Arjovsky, S.~Chintala, and L.~Bottou.
\newblock Wasserstein generative adversarial networks.
\newblock In {\em International Conference on Machine Learning}, pages
  214--223, 2017.

\bibitem{azar2017minimax}
M.~G. Azar, I.~Osband, and R.~Munos.
\newblock Minimax regret bounds for reinforcement learning.
\newblock In {\em International Conference on Machine Learning}, pages
  263--272. PMLR, 2017.

\bibitem{Bauschke2011}
H.~H. Bauschke and P.~Combettes.
\newblock {\em {C}onvex {A}nalysis and {M}onotone {O}perators {T}heory in
  {H}ilbert {S}paces}.
\newblock Springer-Verlag, 2nd edition, 2017.

\bibitem{bauschke2020generalized}
H.~H. Bauschke, W.~M. Moursi, and X.~Wang.
\newblock Generalized monotone operators and their averaged resolvents.
\newblock {\em Math. Program.}, 189(Series B):55--74, 2021.

\bibitem{Ben-Tal2009}
A.~Ben-Tal, L.~El~Ghaoui, and A.~Nemirovski.
\newblock {\em {R}obust optimization}.
\newblock Princeton University Press, 2009.

\bibitem{Bertsimas2011}
D.B. Bertsimas, D.~Brown and C.~Caramanis.
\newblock {T}heory and {A}pplications of {R}obust {O}ptimization.
\newblock {\em SIAM Review}, 53(3):464--501, 2011.

\bibitem{beznosikov2023stochastic}
A.~Beznosikov, E.~Gorbunov, H.~Berard, and N.~Loizou.
\newblock Stochastic gradient descent-ascent: Unified theory and new efficient
  methods.
\newblock In {\em International Conference on Artificial Intelligence and
  Statistics}, pages 172--235. PMLR, 2023.

\bibitem{bohm2022solving}
A.~B{\"o}hm.
\newblock Solving nonconvex-nonconcave min-max problems exhibiting weak {M}inty
  solutions.
\newblock {\em Transactions on Machine Learning Research}, 2022.

\bibitem{bot2019forward}
R.~I. Bot, P.~Mertikopoulos, M.~Staudigl, and P.~T. Vuong.
\newblock Forward-backward-forward methods with variance reduction for
  stochastic variational inequalities.
\newblock {\em arXiv preprint arXiv:1902.03355}, 2019.

\bibitem{boct2021minibatch}
R.~I. Bo{\c{t}}, P.~Mertikopoulos, M.~Staudigl, and P.~T. Vuong.
\newblock Minibatch forward-backward-forward methods for solving stochastic
  variational inequalities.
\newblock {\em Stochastic Systems}, 11(2):112--139, 2021.

\bibitem{cai2023variance}
X.~Cai, A.~Alacaoglu, and J.~Diakonikolas.
\newblock Variance reduced {H}alpern iteration for finite-sum monotone
  inclusions.
\newblock In {\em The 12th International Conference on Learning Representations
  (ICLR)}, pages 1--33, 2024.

\bibitem{cai2022stochastic}
X.~Cai, C.~Song, C.~Guzm\'{a}n, and J.~Diakonikolas.
\newblock A stochastic {H}alpern iteration with variance reduction for
  stochastic monotone inclusion problems.
\newblock In {\em Proceedings of the 12th International Conference on Learning
  Representations (ICLR 2022)}, 2022.

\bibitem{cai2022tight}
Y.~Cai, A.~Oikonomou, and W.~Zheng.
\newblock Tight last-iterate convergence of the extragradient and the
  optimistic gradient descent-ascent algorithm for constrained monotone
  variational inequalities.
\newblock {\em arXiv preprint arXiv:2204.09228}, 2022.

\bibitem{carmon2019variance}
Y.~Carmon, Y.~Jin, A.~Sidford, and K.~Tian.
\newblock Variance reduction for matrix games.
\newblock {\em Advances in Neural Information Processing Systems}, 32, 2019.

\bibitem{chavdarova2019reducing}
T.~Chavdarova, G.~Gidel, F.~Fleuret, and S.~Lacoste-Julien.
\newblock Reducing noise in gan training with variance reduced extragradient.
\newblock {\em Advances in Neural Information Processing Systems}, 32:393--403,
  2019.

\bibitem{combettes2004proximal}
P.~L. Combettes and T.~Pennanen.
\newblock Proximal methods for cohypomonotone operators.
\newblock {\em SIAM J. Control Optim.}, 43(2):731--742, 2004.

\bibitem{combettes2015stochastic}
P.~L. Combettes and J.-C. Pesquet.
\newblock Stochastic quasi-{F}ej{\'e}r block-coordinate fixed point iterations
  with random sweeping.
\newblock {\em SIAM J. Optim.}, 25(2):1221--1248, 2015.

\bibitem{cui2021analysis}
S.~Cui and U.V. Shanbhag.
\newblock On the analysis of variance-reduced and randomized projection
  variants of single projection schemes for monotone stochastic variational
  inequality problems.
\newblock {\em Set-Valued and Variational Analysis}, 29(2):453--499, 2021.

\bibitem{Cutkosky2019}
A.~Cutkosky and F.~Orabona.
\newblock Momentum-based variance reduction in non-convex {SGD}.
\newblock In {\em Advances in Neural Information Processing Systems}, pages
  15210--15219, 2019.

\bibitem{dafermos1980traffic}
Stella Dafermos.
\newblock Traffic equilibrium and variational inequalities.
\newblock {\em Transportation Science}, 14(1):42--54, 1980.

\bibitem{daskalakis2018training}
C.~Daskalakis, A.~Ilyas, V.~Syrgkanis, and H.~Zeng.
\newblock Training {GANs} with {O}ptimism.
\newblock In {\em International Conference on Learning Representations (ICLR
  2018)}, pages 1--9, 2018.

\bibitem{daskalakis2018limit}
C.~Daskalakis and I.~Panageas.
\newblock The limit points of (optimistic) gradient descent in min-max
  optimization.
\newblock {\em Advances in neural information processing systems}, 31, 2018.

\bibitem{davis2022variance}
D.~Davis.
\newblock Variance reduction for root-finding problems.
\newblock {\em Math. Program.}, pages 1--36, 2022.

\bibitem{Defazio2014}
A.~Defazio, F.~Bach, and S.~Lacoste-Julien.
\newblock {SAGA}: {A} fast incremental gradient method with support for
  non-strongly convex composite objectives.
\newblock In {\em Advances in Neural Information Processing Systems (NIPS)},
  pages 1646--1654, 2014.

\bibitem{diakonikolas2021efficient}
J.~Diakonikolas, C.~Daskalakis, and M.~Jordan.
\newblock Efficient methods for structured nonconvex-nonconcave min-max
  optimization.
\newblock In {\em International Conference on Artificial Intelligence and
  Statistics}, pages 2746--2754. PMLR, 2021.

\bibitem{du2017stochastic}
S.~S. Du, J.~Chen, L.~Li, L.~Xiao, and D.~Zhou.
\newblock Stochastic variance reduction methods for policy evaluation.
\newblock In {\em International Conference on Machine Learning}, pages
  1049--1058. PMLR, 2017.

\bibitem{du2021fairness}
W.~Du, D.~Xu, X.~Wu, and H.~Tong.
\newblock Fairness-aware agnostic federated learning.
\newblock In {\em Proceedings of the 2021 SIAM International Conference on Data
  Mining (SDM)}, pages 181--189. SIAM, 2021.

\bibitem{Facchinei2003}
F.~Facchinei and J.-S. Pang.
\newblock {\em Finite-dimensional variational inequalities and complementarity
  problems}, volume 1-2.
\newblock Springer-Verlag, 2003.

\bibitem{fan2001variable}
J.~Fan and R.~Li.
\newblock Variable selection via nonconcave penalized likelihood and its oracle
  properties.
\newblock {\em Journal of the American statistical Association},
  96(456):1348--1360, 2001.

\bibitem{fang2018spider}
C.~Fang, C.~J. Li, Z.~Lin, and T.~Zhang.
\newblock {SPIDER}: {N}ear-optimal non-convex optimization via stochastic path
  integrated differential estimator.
\newblock In {\em Advances in Neural Information Processing Systems}, pages
  689--699, 2018.

\bibitem{goodfellow2014generative}
I.~Goodfellow, J.~Pouget-Abadie, M.~Mirza, B.~Xu, D.~Warde-Farley, S.~Ozair,
  A.~Courville, and Y.~Bengio.
\newblock Generative adversarial nets.
\newblock In {\em Advances in neural information processing systems}, pages
  2672--2680, 2014.

\bibitem{gorbunov2022stochastic}
E.~Gorbunov, H.~Berard, G.~Gidel, and N.~Loizou.
\newblock Stochastic extragradient: General analysis and improved rates.
\newblock In {\em International Conference on Artificial Intelligence and
  Statistics}, pages 7865--7901. PMLR, 2022.

\bibitem{gorbunov2022extragradient}
E.~Gorbunov, N.~Loizou, and G.~Gidel.
\newblock Extragradient method: $\mathcal{O} (1/k)$ last-iterate convergence
  for monotone variational inequalities and connections with cocoercivity.
\newblock In {\em International Conference on Artificial Intelligence and
  Statistics}, pages 366--402. PMLR, 2022.

\bibitem{gorbunov2022last}
E.~Gorbunov, A.~Taylor, and G.~Gidel.
\newblock Last-iterate convergence of optimistic gradient method for monotone
  variational inequalities.
\newblock {\em Advances in neural information processing systems},
  35:21858--21870, 2022.

\bibitem{huang2022accelerated}
K.~Huang, N.~Wang, and S.~Zhang.
\newblock {A}n {A}ccelerated {V}ariance {R}educed {E}xtra-{P}oint {A}pproach to
  {F}inite-{S}um {VI} and {O}ptimization.
\newblock {\em SIAM J. Optim.}, 36(2):1154--1181, 2026.

\bibitem{iusem2017extragradient}
A.~N. Iusem, A.~Jofr{\'e}, R.~I. Oliveira, and P.~Thompson.
\newblock Extragradient method with variance reduction for stochastic
  variational inequalities.
\newblock {\em SIAM J. Optim.}, 27(2):686--724, 2017.

\bibitem{johnson2013accelerating}
R.~Johnson and T.~Zhang.
\newblock Accelerating stochastic gradient descent using predictive variance
  reduction.
\newblock In {\em Advances in Neural Information Processing Systems (NIPS)},
  pages 315--323, 2013.

\bibitem{SVRG}
Rie Johnson and Tong Zhang.
\newblock Accelerating stochastic gradient descent using predictive variance
  reduction.
\newblock In {\em NIPS}, pages 315--323, 2013.

\bibitem{juditsky2011solving}
A.~Juditsky, A.~Nemirovski, and C.~Tauvel.
\newblock Solving variational inequalities with stochastic mirror-prox
  algorithm.
\newblock {\em Stochastic Systems}, 1(1):17--58, 2011.

\bibitem{kannan2019optimal}
A.~Kannan and U.~V. Shanbhag.
\newblock Optimal stochastic extragradient schemes for pseudomonotone
  stochastic variational inequality problems and their variants.
\newblock {\em Comput. Optim. Appl.}, 74(3):779--820, 2019.

\bibitem{Konnov2001}
I.V. Konnov.
\newblock {\em Combined relaxation methods for variational inequalities.}
\newblock Springer-Verlag, 2001.

\bibitem{korpelevich1976extragradient}
G.M. Korpelevich.
\newblock The extragradient method for finding saddle points and other
  problems.
\newblock {\em Matecon}, 12:747--756, 1976.

\bibitem{kovalev2019don}
D.~Kovalev, S.~Horvath, and P.~Richtarik.
\newblock {D}on't jump through hoops and remove those loops: {SVRG} and
  {K}atyusha are better without the outer loop.
\newblock In {\em Algorithmic Learning Theory}, pages 451--467. PMLR, 2020.

\bibitem{lan2020first}
G.~Lan.
\newblock {\em First-order and Stochastic Optimization Methods for Machine
  Learning}.
\newblock Springer, 2020.

\bibitem{levy2020large}
D.~Levy, Y.~Carmon, J.~C. Duchi, and A.~Sidford.
\newblock Large-scale methods for distributionally robust optimization.
\newblock {\em Advances in Neural Information Processing Systems},
  33:8847--8860, 2020.

\bibitem{li2022convergence}
C.~J. Li, Y.~Yu, N.~Loizou, G.~Gidel, Y.~Ma, N.~Le Roux, and M.~Jordan.
\newblock On the convergence of stochastic extragradient for bilinear games
  using restarted iteration averaging.
\newblock In {\em International Conference on Artificial Intelligence and
  Statistics}, pages 9793--9826. PMLR, 2022.

\bibitem{li2020page}
Z.~Li, H.~Bao, X.~Zhang, and P.~Richt{\'a}rik.
\newblock {PAGE}: {A} simple and optimal probabilistic gradient estimator for
  nonconvex optimization.
\newblock {\em International conference on machine learning}, pages 6286--6295,
  2021.

\bibitem{luo2022last}
Y.~Luo and Q.~Tran-Dinh.
\newblock Extragradient-type methods for co-monotone root-finding problems.
\newblock {\em (UNC-STOR Technical Report)}, 2022.

\bibitem{madry2018towards}
A.~Madry, A.~Makelov, L.~Schmidt, D.~Tsipras, and A.~Vladu.
\newblock Towards deep learning models resistant to adversarial attacks.
\newblock In {\em International Conference on Learning Representations}, 2018.

\bibitem{malitsky2015projected}
Y.~Malitsky.
\newblock Projected reflected gradient methods for monotone variational
  inequalities.
\newblock {\em SIAM J. Optim.}, 25(1):502--520, 2015.

\bibitem{malitsky2019golden}
Y.~Malitsky.
\newblock Golden ratio algorithms for variational inequalities.
\newblock {\em Math. Program.}, 184(1--2):383--410, 2020.

\bibitem{malitsky2020forward}
Y.~Malitsky and M.~K. Tam.
\newblock A forward-backward splitting method for monotone inclusions without
  cocoercivity.
\newblock {\em SIAM J. Optim.}, 30(2):1451--1472, 2020.

\bibitem{mertikopoulos2019optimistic}
P.~Mertikopoulos, B.~Lecouat, H.~Zenati, C.-S. Foo, V.~Chandrasekhar, and
  G.~Piliouras.
\newblock Optimistic mirror descent in saddle-point problems: {G}oing the extra
  (gradient) mile.
\newblock In {\em International Conference on Learning Representations (ICLR)},
  pages 1--23, 2019.

\bibitem{mishchenko2020random}
K.~Mishchenko, A.~Khaled, and P.~Richt{\'a}rik.
\newblock Random reshuffling: Simple analysis with vast improvements.
\newblock {\em Advances in Neural Information Processing Systems},
  33:17309--17320, 2020.

\bibitem{mishchenko2020revisiting}
K.~Mishchenko, D.~Kovalev, E.~Shulgin, P.~Richt{\'a}rik, and Y.~Malitsky.
\newblock Revisiting stochastic extragradient.
\newblock In {\em International Conference on Artificial Intelligence and
  Statistics}, pages 4573--4582. PMLR, 2020.

\bibitem{Nemirovski2009}
A.~Nemirovski, A.~Juditsky, G.~Lan, and A.~Shapiro.
\newblock Robust stochastic approximation approach to stochastic programming.
\newblock {\em SIAM J. on Optimization}, 19(4):1574--1609, 2009.

\bibitem{nguyen2017sarah}
L.~M. Nguyen, J.~Liu, K.~Scheinberg, and M.~Tak{\'a}{\v{c}}.
\newblock {SARAH}: A novel method for machine learning problems using
  stochastic recursive gradient.
\newblock In {\em Proceedings of the 34th International Conference on Machine
  Learning}, pages 2613--2621, 2017.

\bibitem{noor2003extragradient}
M.~A. Noor.
\newblock Extragradient methods for pseudomonotone variational inequalities.
\newblock {\em J. Optim. Theory Appl.}, 117(3):475--488, 2003.

\bibitem{noor1999wiener}
M.~A. Noor and E.A. Al-Said.
\newblock {W}iener--{H}opf equations technique for quasimonotone variational
  inequalities.
\newblock {\em J. Optim. Theory Appl.}, 103:705--714, 1999.

\bibitem{palaniappan2016stochastic}
B.~Palaniappan and F.~Bach.
\newblock Stochastic variance reduction methods for saddle-point problems.
\newblock In {\em Advances in Neural Information Processing Systems}, pages
  1416--1424, 2016.

\bibitem{pethick2023solving}
T.~Pethick, O.~Fercoq, P.~Latafat, P.~Patrinos, and V.~Cevher.
\newblock Solving stochastic weak {M}inty variational inequalities without
  increasing batch size.
\newblock In {\em Proceedings of International Conference on Learning
  Representations (ICLR)}, pages 1--34, 2023.

\bibitem{phelps2009convex}
R.~R. Phelps.
\newblock {\em Convex functions, monotone operators and differentiability},
  volume 1364.
\newblock Springer, 2009.

\bibitem{popov1980modification}
L.~D. Popov.
\newblock A modification of the {A}rrow-{H}urwicz method for search of saddle
  points.
\newblock {\em Math. notes of the Academy of Sciences of the USSR},
  28(5):845--848, 1980.

\bibitem{Reddi2016a}
Sashank~J. Reddi, Ahmed Hefny, Suvrit Sra, Barnab{\'{a}}s P{\'{o}}czos, and
  Alexander~J. Smola.
\newblock Stochastic variance reduction for nonconvex optimization.
\newblock In {\em ICML}, pages 314--323, 2016.

\bibitem{robbins1971convergence}
H.~Robbins and D.~Siegmund.
\newblock A convergence theorem for non negative almost supermartingales and
  some applications.
\newblock In {\em Optimizing methods in statistics}, pages 233--257. Elsevier,
  1971.

\bibitem{robbins1951stochastic}
Herbert Robbins and Sutton Monro.
\newblock A stochastic approximation method.
\newblock {\em The Annals of Mathematical Statistics}, 22(3):400--407, 1951.

\bibitem{Rockafellar2004}
R.~Rockafellar and R.~Wets.
\newblock {\em {V}ariational {A}nalysis}, volume 317.
\newblock Springer, 2004.

\bibitem{Rockafellar1976b}
R.T. Rockafellar.
\newblock Monotone operators and the proximal point algorithm.
\newblock {\em SIAM J. Control Optim.}, 14:877--898, 1976.

\bibitem{Rockafellar1997}
R.T. Rockafellar and R.~J-B. Wets.
\newblock {\em {V}ariational {A}nalysis}.
\newblock Springer-Verlag, 1997.

\bibitem{ryu2016primer}
E.~K. Ryu and S.~Boyd.
\newblock Primer on monotone operator methods.
\newblock {\em Appl. Comput. Math}, 15(1):3--43, 2016.

\bibitem{sutton2018reinforcement}
R.~S. Sutton and A.~G. Barto.
\newblock {\em {R}einforcement {L}earning: {A}n {I}ntroduction}.
\newblock Adaptive Computation and Machine Learning series. MIT Press,
  Cambridge, MA, second edition, 2018.

\bibitem{TranDinh2025a}
Q.~Tran-Dinh.
\newblock {VFOSA}: {V}ariance-{R}educed {F}ast {O}perator {S}plitting {M}ethods
  for {S}tochastic {G}eneralized {E}quations.
\newblock {\em J. Mach. Learn. Res. $($JMLR$)$}, 16:1--68, 2025.

\bibitem{tran2022accelerated}
Q.~Tran-Dinh and Y.~Luo.
\newblock {R}andomized {B}lock-{C}oordinate {O}ptimistic {G}radient
  {A}lgorithms for {R}oot-{F}inding {P}roblems.
\newblock {\em Math. Oper. Res.}, 51(1):746--782, 2025.

\bibitem{tran2025vfog}
Q.~Tran-Dinh and N.~Nguyen-Trung.
\newblock {VFOG}: {V}ariance-reduced fast optimistic gradient methods for a
  class of nonmonotone generalized equations.
\newblock {\em arXiv preprint arXiv:2508.16791}, 2025.

\bibitem{tran2024revisiting}
Q.~Tran-Dinh and N.~Nguyen-Trung.
\newblock Revisiting {E}xtragradient-type methods--{P}art 1: {G}eneralizations
  and sublinear convergence rates.
\newblock {\em Comput. Optim. Appl.}, 94(1):141--210, 2026.

\bibitem{Tran-Dinh2019a}
Q.~Tran-Dinh, N.~H. Pham, D.~T. Phan, and L.~M. Nguyen.
\newblock A hybrid stochastic optimization framework for stochastic composite
  nonconvex optimization.
\newblock {\em Math. Program.}, 191:1005--1071, 2022.

\bibitem{TranDinh2024}
Quoc Tran-Dinh.
\newblock {V}ariance-{R}educed {F}orward-{R}eflected-{B}ackward {S}plitting
  {M}ethods for {N}onmonotone {G}eneralized {E}quations.
\newblock {\em Forty-Second International Conference on Machine Learning
  (ICML)}, 2025.

\bibitem{tseng2000modified}
P.~Tseng.
\newblock A modified forward-backward splitting method for maximal monotone
  mappings.
\newblock {\em SIAM J. Control and Optim.}, 38(2):431--446, 2000.

\bibitem{vuong2018weak}
V.~Phan Tu.
\newblock On the weak convergence of the extragradient method for solving
  pseudo-monotone variational inequalities.
\newblock {\em J. Optim. Theory Appl.}, 176(2):399--409, 2018.

\bibitem{yang2020global}
J.~Yang, N.~Kiyavash, and N.~He.
\newblock Global convergence and variance-reduced optimization for a class of
  nonconvex-nonconcave minimax problems.
\newblock {\em arXiv preprint arXiv:2002.09621}, 2020.

\bibitem{ying2016stochastic}
Y.~Ying, L.~Wen, and S.~Lyu.
\newblock {S}tochastic {O}nline {AUC} {M}aximization.
\newblock In D.~Lee, M.~Sugiyama, U.~Luxburg, I.~Guyon, and R.~Garnett,
  editors, {\em Advances in Neural Information Processing Systems}, volume~29.
  Curran Associates, Inc., 2016.

\bibitem{yousefian2018stochastic}
F.~Yousefian, A.~Nedi{\'c}, and U.~V. Shanbhag.
\newblock On stochastic mirror-prox algorithms for stochastic cartesian
  variational inequalities: {R}andomized block coordinate and optimal averaging
  schemes.
\newblock {\em Set-Valued and Variational Analysis}, 26:789--819, 2018.

\bibitem{yu2022fast}
Y.~Yu, T.~Lin, E.~V. Mazumdar, and M.~Jordan.
\newblock Fast distributionally robust learning with variance-reduced min-max
  optimization.
\newblock In {\em International Conference on Artificial Intelligence and
  Statistics}, pages 1219--1250. PMLR, 2022.

\bibitem{zhang2021multi}
K.~Zhang, Z.~Yang, and T.~Ba{\c{s}}ar.
\newblock Multi-agent reinforcement learning: {A} selective overview of
  theories and algorithms.
\newblock {\em Handbook of reinforcement learning and control}, pages 321--384,
  2021.

\end{thebibliography}

\end{document}